\pdfoutput=1 %
\documentclass{article}
\ifdefined\pdfsuppressptexinfo\pdfsuppressptexinfo=-1\fi

\usepackage{arxiv}
\usepackage{times}
\usepackage{natbib}
\setcitestyle{authoryear,round,citesep={;},aysep={,},yysep={;}}
\makeatletter
\renewenvironment{table}{\@float{table}}{\end@float}
\makeatother
\makeatletter
\newcommand{\staveneedspace}[1]{\par\penalty-100\begingroup\setlength{\dimen@}{#1}\dimen@ii\pagegoal\advance\dimen@ii-\pagetotal
  \ifdim\dimen@>\dimen@ii\ifdim\dimen@ii>\z@\vfil\fi\break\fi\endgroup}
\makeatother
\usepackage[T1]{fontenc}
\usepackage[utf8]{inputenc}
\usepackage{amsmath,amssymb,amsthm,mathtools}
\usepackage{graphicx}
\usepackage{wrapfig}
\usepackage{pifont}
\usepackage[table]{xcolor}
\usepackage{booktabs,array,tabularx,multirow}
\IfFileExists{pdflscape.sty}{\usepackage{pdflscape}}{\usepackage{lscape}}
\usepackage{microtype}
\usepackage[most]{tcolorbox}
\usepackage{hyperref}
\usepackage{url}

\graphicspath{{./}{figures/}}
\newcommand{\grp}[1]{\textsc{#1}}
\definecolor{stave}{HTML}{EA580C}
\definecolor{staveink}{HTML}{9A3B08}
\definecolor{stavehighlight}{HTML}{FFF7ED}
\definecolor{paramsmall}{HTML}{E5F2EA}
\definecolor{parammedium}{HTML}{F1E3B5}
\definecolor{paramlarge}{HTML}{F5DDDA}
\definecolor{tableheader}{HTML}{E8E8E8}
\definecolor{promptback}{HTML}{F2F6FA}
\definecolor{promptframe}{HTML}{7893AD}
\definecolor{deltaup}{HTML}{15803D}
\definecolor{deltadown}{HTML}{B91C1C}
\newcommand{\dup}[1]{\textcolor{deltaup}{$+#1$}}
\newcommand{\ddown}[1]{\textcolor{deltadown}{$-#1$}}
\newcommand{\stavelabel}{\textbf{STAVE}}
\newcommand{\tablekey}[1]{\par\vspace{3pt}{\footnotesize\raggedright #1\par}\vspace{4pt}}
\newcolumntype{N}{>{\raggedleft\arraybackslash}X}
\newcommand{\tok}[1]{{\fontfamily{zi4}\selectfont #1}}

\newtheoremstyle{staveplain}{6pt}{6pt}{\normalfont}{}{\bfseries}{.}{0.5em}{\thmname{#1}\thmnumber{ #2}\thmnote{ (#3)}}
\theoremstyle{staveplain}
\newtheorem{theorem}{Theorem}[section]
\newtheorem{proposition}[theorem]{Proposition}

\newtheorem{corollary}[theorem]{Corollary}
\newtheorem{assumption}[theorem]{Assumption}

\newenvironment{stavebox}{%
  \begin{tcolorbox}[enhanced,breakable,colback=stavehighlight,
    colframe=stave!55,boxrule=0.45pt,arc=1pt,
    left=7pt,right=7pt,top=1pt,bottom=1pt,before skip=8pt,after skip=8pt]
}{\end{tcolorbox}}
\newenvironment{datasetexample}[1]{%
  \begin{tcolorbox}[enhanced,breakable,colback=promptback,
    colframe=promptframe,boxrule=0.55pt,arc=1pt,
    title={#1},fonttitle=\bfseries,coltitle=black,
    left=7pt,right=7pt,top=5pt,bottom=5pt,before skip=8pt,after skip=8pt]
}{\end{tcolorbox}}

\DeclareMathOperator{\diag}{diag}

\DeclareMathOperator{\softmax}{softmax}
\DeclareMathOperator{\col}{col}
\DeclareMathOperator*{\argmin}{arg\,min}
\newcommand{\R}{\mathbb{R}}
\newcommand{\E}{\mathbb{E}}
\newcommand{\ind}{\mathbf{1}}

\title{Two Vectors Replace In-Context Demos: Structured Task Adaptation via Embeddings}
\newcommand{\affwisc}{\textsuperscript{$\spadesuit$}}
\newcommand{\affnu}{\textsuperscript{$\heartsuit$}}
\author{Xi Ding\affwisc\quad Naichen Shi\affnu\quad Jiawei Zhang\affwisc\\[5pt]
\normalfont \affwisc University of Wisconsin--Madison\qquad \affnu Northwestern University}
\date{}
\renewcommand{\shorttitle}{Structured Task Adaptation via Embeddings}
\renewcommand{\headeright}{}
\makeatletter
\renewcommand{\@maketitle}{%
  \vbox{%
    \hsize\textwidth
    \linewidth\hsize
    \vskip 0.1in
    \@toptitlebar
    \centering
    {\LARGE\sc \@title\par}
    \@bottomtitlebar
    \begin{tabular}[t]{c}\bf\rule{\z@}{24\p@}\@author\end{tabular}%
    \vskip 0.22in
  }%
}
\makeatother
\hypersetup{pdftitle={Two Vectors Replace In-Context Demos: Structured Task Adaptation via Embeddings},pdfauthor={Xi Ding, Naichen Shi, Jiawei Zhang}}

\begin{document}
\addtocontents{toc}{\protect\setcounter{tocdepth}{-1}}
\maketitle

\begin{abstract}
In-context learning (ICL) adapts frozen large multimodal models (LMMs) to new tasks from a few demonstrations (demos), but re-encodes them at every query, where each demo image adds up to hundreds of visual tokens. Demo-free methods remove this cost with a compact task state. However, they add it at locations searched per task or at every decoder layer, where task parameters grow with depth. Moreover, inserted tokens or keys cannot change how the original prompt divides its attention within a layer. To address these issues, we propose Structured Task Adaptation via Embeddings (\textbf{STAVE}), which replaces demos with two task-specific vectors added to existing input embeddings. Specifically, a readout vector updates the answer-producing tokens and a context vector updates the other structural token groups. Both are trained with answer labels on prompts with and without demos. We justify these design choices theoretically using a first-order analysis of the loss and a margin bound. Extensive experiments on six LMMs and five large language models show that STAVE matches or outperforms state-of-the-art methods on multimodal tasks with far fewer task parameters and surpasses 15-shot ICL and prior task vectors on 18 text tasks, all at zero-shot inference cost.
\end{abstract}

\section{Introduction}
\label{sec:intro}

In-context learning (ICL) lets a frozen large language model (LLM) perform a new task from a few input--output demonstrations (demos) placed in its prompt \citep{brown2020language,dong2024survey}. Large multimodal models (LMMs) inherit this ability \citep{alayrac2022flamingo,laurencon2024matters,li2024llava}, so one frozen model can handle visual question answering (VQA) or image captioning without training. This flexibility is costly, since every query re-encodes its demos \citep{liu2022few} and each demo image expands into up to hundreds of visual tokens. Its accuracy also depends on which demos are chosen and in what order they appear \citep{zhao2021calibrate}. We study how to adapt a frozen model to a task at the prompt length and cost of zero-shot inference.

Demo-free methods store the task in a compact state added to the frozen model at inference. Extracted task vectors take this state from demo activations. The task vector (TV) \citep{hendel-etal-2023-context} and the function vector (FV) \citep{todd2024function} add one vector at an intermediate layer, while MTV \citep{huang2024multimodal} and STV \citep{ma2026sensitivity} place it at attention heads selected for each task. Trained interventions instead optimize the state with supervision. LIVE \citep{peng2024live} learns a shift vector for every decoder layer, and MimIC \citep{jiang2025mimic} and HiFICL \citep{li2026hificl} learn per-head shifts or virtual keys and values in every layer to approximate the effect of demos on the query. Prompt tuning (PT) \citep{lester-etal-2021-power} inserts learned virtual tokens into the input prompt.

These methods remove the demos at inference, yet two limitations remain. \textbf{\textit{First}, the task state enters either at a location searched for each task or at every decoder layer.} A vector added at an intermediate layer cannot act through the blocks before it, and its layer or heads must be chosen for each task. Trained interventions avoid this choice by adding a shift to every layer, so their task parameters and per-query computation grow with model depth. \textbf{\textit{Second}, methods that insert tokens or keys, such as PT and HiFICL, cannot change how the original prompt divides its attention within a layer.} For fixed queries and keys of the original tokens, added keys scale all their attention weights by one factor, so the relative attention between the image and the question stays unchanged \citep{petrov2024prompting}. PT must also choose where to insert its tokens, before the prompt, after it, or before its final token, and each inserted token lengthens the prompt and its computation.

To address these limitations, we ask two questions: where a demo-free update should enter a frozen model, and how it can be shared across tokens without losing answer-relevant directions. For the first, only an update at the input embeddings acts through every decoder block, and to first order an unrestricted input update can reproduce the first-answer effect of shifts in every layer. Added to existing tokens, the update keeps the prompt length and changes their attention queries and keys, so it can reweight the attention between the image and the question. For the second, structural token groups, such as the question text and the query image, recur in every prompt of a task, so one vector per group serves every query. Sharing has a cost: the final prompt token predicts the first answer token, whereas the other tokens reach it only through attention and still hold query-specific evidence, and a vector shared with many weak-gradient tokens, or with tokens whose gradients point in other directions, has a diluted first-order loss decrease.

\begin{wrapfigure}{r}{0.42\textwidth}
\vspace{-\intextsep}
\centering
\includegraphics[width=\linewidth,trim=0 14 0 15,clip]{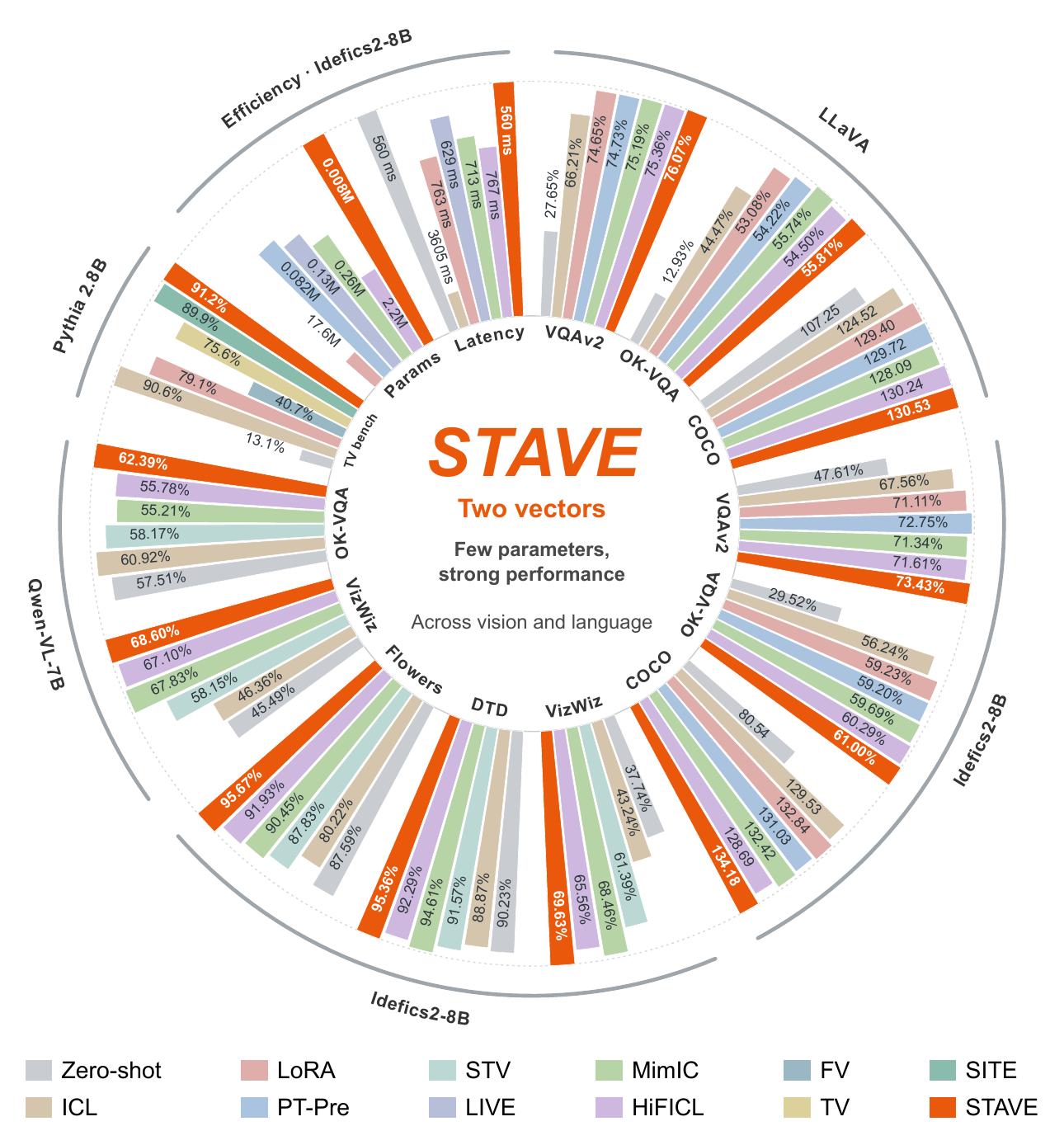}\par
\vspace{-13pt}
\caption{STAVE performance and efficiency. Bars are relative to STAVE, and params use a log scale (Appendix~\ref{app:overview-figure}).}
\label{fig:performance-overview}
\vspace{-\intextsep}
\end{wrapfigure}
Based on this analysis, we propose \textbf{S}tructured \textbf{T}ask \textbf{A}daptation \textbf{v}ia \textbf{E}mbeddings (STAVE), which adds two task-specific vectors to the input embeddings of a frozen LLM or LMM. A readout vector updates only the final prompt token and the answer cue, and a context vector covers the remaining groups. Both are trained with answer labels instead of an ICL teacher, on queries with and without demos, and added once to the demo-free prompt.

\textbf{Contributions.} \ding{182}~\textbf{\underline{Framework}.} We propose STAVE, which replaces in-context demos with a readout vector and a context vector added to the input embeddings of a frozen model. Its task parameters do not grow with model depth, and it runs at zero-shot inference cost. \ding{183}~\mbox{\textbf{\underline{Theoretical Explanation}.}} We show that only an input update reaches every decoder block, but one vector shared there is limited in norm and direction. Weak-gradient tokens such as image tokens dilute its norm, and its single direction cannot follow readout and context gradients that differ. At equal total norm, two separately normalized vectors never give a smaller first-order loss decrease, and their advantage grows as these gradients diverge. We also bound the margin cost of tying groups and show when training with and without demos lowers the demo-free loss. \ding{184}~\mbox{\textbf{\underline{Extensive Evaluations}.}} On multimodal and text tasks, STAVE matches or outperforms per-layer interventions with far fewer task parameters and surpasses ICL and task vectors without demos at inference (Figure~\ref{fig:performance-overview}). \ding{185}~\textbf{\underline{Insight}.} The gain comes from the readout/context separation, not a modality split or a second vector. The readout vector sets a shared answer format, the context vector adds query-specific evidence, and together they mainly elicit answers the model already holds.

\section{Motivating Analysis}
\label{sec:motivating-analysis}

We answer the two questions of Section~\ref{sec:intro} for an additive prompt update, which changes a frozen decoder's output without demos. Its \emph{reach} decides where it should enter the model, and the \emph{first-order cost of sharing} decides how to share it across tokens without losing answer-relevant directions.

\begin{stavebox}
\begin{proposition}[Reach of a prompt update]
\label{prop:depth-reach}
\label{prop:final-depth}
Fix a prompt and teacher-forced answer prefixes in a differentiable causal decoder with $N_{\mathrm{layer}}$ blocks and a positionwise final head. Depth zero denotes input embeddings and depth $\ell$ the output of block $\ell$.
\textbf{(i)} A prompt-only update at depth $\ell$ can act through later blocks and the head. After the last block, with stored keys and values unchanged, \emph{only the final prompt token can affect answer logits, and only the first prediction}.
\textbf{(ii)} If all intervening own-token Jacobians are invertible, any infinitesimal change of all depth-$\ell$ prompt states has a unique unrestricted per-token realization at an earlier depth.
\end{proposition}
\end{stavebox}

\textbf{Reach.}
Thus \emph{input injection can affect every block}. By part (ii), an unrestricted per-prompt input update matches the first-answer effect of simultaneous per-layer shifts to first order (Appendix~\ref{app:depth-reach-proof}).

Adding to existing tokens also keeps the sequence length. At a fixed attention layer, added keys scale all weights on the original tokens by one common factor when their queries, keys, and positional scores stay fixed, which preserves the ratios between them \citep{petrov2024prompting}. An embedding update can change those queries and keys and hence reorder the weights (Appendix~\ref{app:prompt-comparison}). However, a per-prompt update does not transfer to new queries, so reuse requires shared vectors.

\begin{wrapfigure}{R}{0.5\textwidth}
\vspace{-\intextsep}
\centering
\includegraphics[width=0.93\linewidth,trim=0 5 0 3,clip]{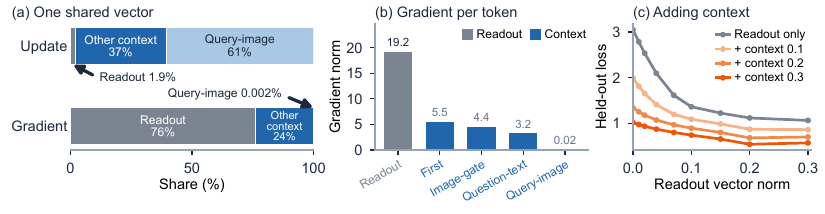}\par
\vspace{-9pt}
\caption{One vector falls short and two do better.}
\label{fig:motivation-chain}
\vspace{-\intextsep}
\end{wrapfigure}
\textbf{Sharing.}
Fix a differentiable loss and let $g_t\in\R^d$ be its gradient at input embedding $t$, where $d$ is the hidden dimension. For a nonempty token set $\mathcal S$ and budget $\rho>0$, one vector shared as $w/\sqrt{|\mathcal S|}$ has total update norm $\|w\|_2$. Its largest first-order loss decrease is
\begin{equation}
 \Gamma(\mathcal S)=\frac{\rho}{\sqrt{|\mathcal S|}}\left\|\sum_{t\in\mathcal S}g_t\right\|_2
 \leq\rho\left(\sum_{t\in\mathcal S}\|g_t\|_2^2\right)^{1/2}.
\label{eq:sharing-bound}
\end{equation}
Free per-token updates of equal total norm attain the right side. Equality holds only when all $g_t$ agree. For disjoint nonempty $\mathcal A,\mathcal B$, two vectors at equal total budget attain $\sqrt{\Gamma(\mathcal A)^2+\Gamma(\mathcal B)^2}\geq\Gamma(\mathcal A\cup\mathcal B)$. The gap splits into a norm term, zero when the mean gradients of $\mathcal A$ and $\mathcal B$ have equal norm, and a direction term that grows as the cosine between their summed gradients falls (Proposition~\ref{prop:sharing-cost}(iii)). If the gradients on $\mathcal B$ vanish, one shared vector attains only $\sqrt{|\mathcal A|/|\mathcal A\cup\mathcal B|}\,\Gamma(\mathcal A)$ (Proposition~\ref{prop:sharing-cost}). Figure~\ref{fig:motivation-chain} measures this in Idefics2 on Flowers before training (Appendix~\ref{app:motivation-chain}). One vector on all tokens splits its update by token count, so image tokens, with almost no gradient, receive most of it, and the final token and answer cue, with 76\% of the gradient, receive under 2\% (Figure~\ref{fig:motivation-chain}(a)).

\textbf{Design.}
STAVE therefore separates a \emph{readout set} $\mathcal A$, the final prompt token and answer cue, from the remaining \emph{context tokens} $\mathcal B$. The final token supplies the first-token readout (Proposition~\ref{prop:depth-reach}(i)), and $\mathcal A$ has the largest mean gradient per token (Figure~\ref{fig:motivation-chain}(b)), so a vector on $\mathcal A$ alone is not diluted by $\mathcal B$. Context tokens reach the readout only through attention and have a descent direction nearly orthogonal to that of $\mathcal A$ (Appendix~\ref{app:motivation-chain}), so one shared vector cannot point along both. Along these directions, the held-out loss levels off with the readout vector alone, whereas each added context vector norm lowers the whole curve (Figure~\ref{fig:motivation-chain}(c)). Section~\ref{sec:method} builds both vectors from reusable structural token groups and bounds the margin cost of tying groups within each vector (Theorem~\ref{thm:tying}).

\section{Method}
\label{sec:method}

STAVE adds two task-specific vectors to existing embeddings of a frozen autoregressive model. A \emph{readout vector} updates the final prompt token and answer cue, and a \emph{context vector} updates the other eligible tokens. Both are trained with and without demos and used without them (Figure~\ref{fig:method}).

\subsection{Structural embedding updates}
\label{sec:method-groups}
\label{sec:method-updates}

Let $f_\theta$ have frozen parameters $\theta$ and hidden dimension $d$. Prompt $p$ has input $E(p)\in\R^{T(p)\times d}$, where $T(p)$ counts tokens after modality expansion. A fixed, label-free map $s_t(p)\in\{0,\ldots,R\}$ assigns each token to a \emph{structural token group} or to zero for exclusion. The $R=7$ groups are \grp{Last} (final prompt token), \grp{Answer-cue} (template tokens introducing the answer before it), \grp{First} (designated first token), \grp{Image-gate} (query-image delimiters), \grp{Query-image} (expanded query-image embeddings), \grp{Question-text} (remaining query text), and \grp{Prefix-text} (instruction prefix).

Beginning-of-sequence (BOS) tokens, padding, answer tokens, and unmatched tokens are excluded. Demo tokens are also excluded, except \grp{First} when the template places it in a demo. \grp{First} is the first attended non-BOS token, or the first query-text token in captioning. Appendix~\ref{app:partition} gives the precedence and template conventions. Remaining tokens are eligible, and groups may be empty.

For $1\leq q\leq R$ learned vectors, let $a(k)\in\{1,\ldots,q\}$ assign group $k$ to a vector and set $a(0)=0$. Let $n_k(p)$ count group $k$ and $N_r(p)=\sum_{k:a(k)=r}n_k(p)$ count all tokens assigned to vector $r$. Define
\begin{equation}
 \widetilde E(p,W)=E(p)+H(p)W,\qquad
 H_{tr}(p)=
 \begin{cases}
 \ind\{a(s_t(p))=r\}/\sqrt{N_r(p)},&N_r(p)>0,\\
 0,&N_r(p)=0.
 \end{cases}
\label{eq:stave-injection}
\end{equation}
Here $W\in\R^{q\times d}$ holds one task's vectors as rows, $H(p)\in\R^{T(p)\times q}$ distributes them, and $\ind$ is the indicator. STAVE sets $q=2$: \grp{Last} and \grp{Answer-cue} share the readout vector, and the other five groups share the context vector. Each vector is normalized once by its merged token count $N_r(p)$.

\begin{figure}[!t]
\centering
\includegraphics[
    width=\textwidth,
    trim=1.8cm 1.82cm 1.8cm 1.45cm,
    clip
]{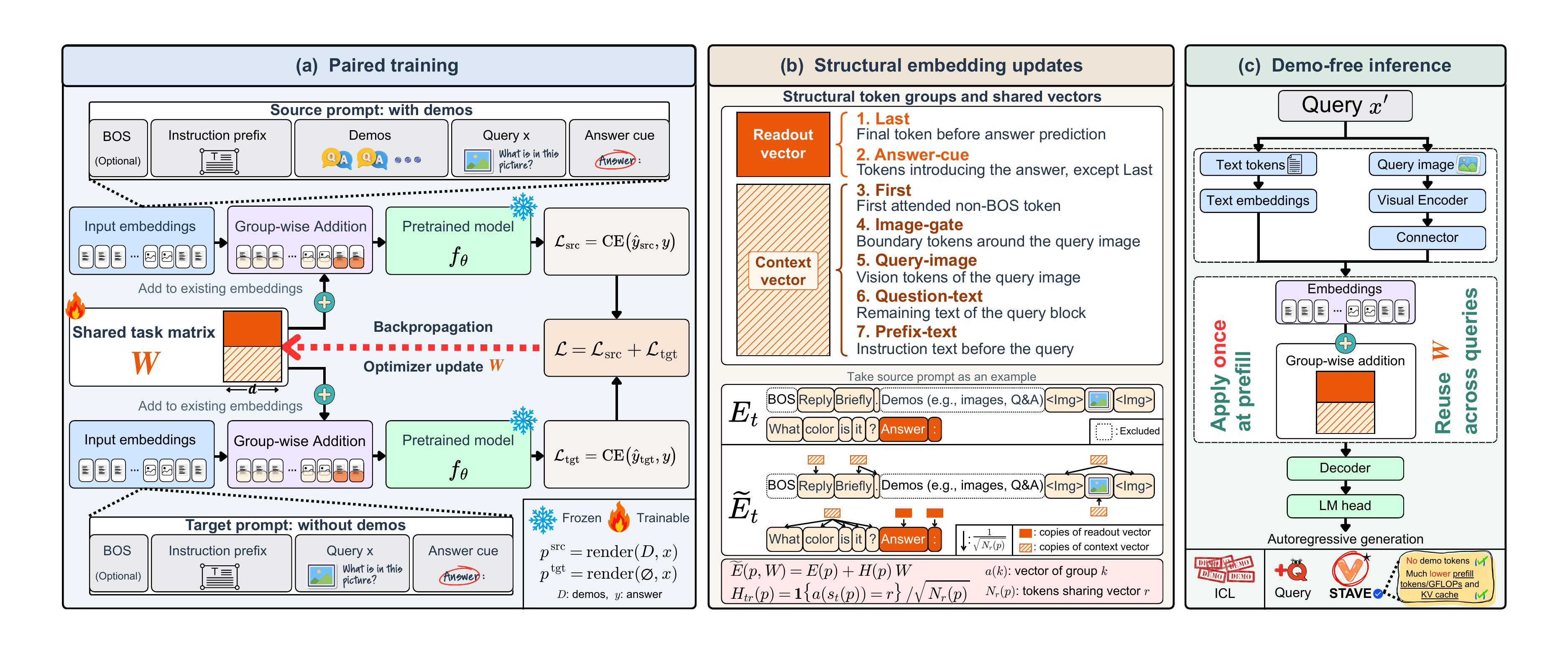}\par
\vspace{-3.2pt}
\caption{Illustration of the proposed framework STAVE in the multimodal setting.}
\label{fig:method}
\end{figure}

Nonempty columns of $H(p)$ have disjoint supports and unit norm. Hence $\|H(p)W\|_F\leq\|W\|_F$, with equality when $N_r(p)>0$ for every $r$. Replication thus cannot grow the injected energy. At $W=0$, \emph{the original computation is recovered exactly}. For a differentiable loss $F$, the gradient of vector $r$ is $N_r(p)^{-1/2}\sum_{t:a(s_t(p))=r}\nabla_{\widetilde E_t}F$, where $\widetilde E_t$ is the updated embedding of token $t$. At $W=0$, $\rho$ times its norm equals $\Gamma$ of that vector's token set in Equation~\ref{eq:sharing-bound} (Appendix~\ref{app:geometry-proof}).

\subsection{Margin cost of tying groups}
\label{sec:method-tying}

To compare tying with per-group updates, fix finitely many prompts and correct answer prefixes with equal group counts $n_k$. Scale each group's vector by $1/\sqrt{n_k}$ and stack them into $v\in\R^{Rd}$, where empty groups have no effect. Equation~\ref{eq:stave-injection} restricts $v$ to a count-weighted subspace $\mathcal U_q$ with orthogonal projector $P_q$. With zero empty-row parameters, the learned-row norm is $\|v\|_2$ (Appendix~\ref{app:tying}).

Let $h_j(v)$ be correct-token minus competing-token logit margin $j$, for a fixed collection of $N\geq1$ comparisons across these contexts. For radius $\rho>0$, define
\begin{equation}
 m_q(\rho)=\max_{v\in\mathcal U_q,\,\|v\|_2\leq\rho}\min_{1\leq j\leq N}h_j(v),
 \qquad
 \delta_q=\max_j\|(I-P_q)\nabla h_j(0)\|_2.
\label{eq:tying-diagnostic}
\end{equation}
Here $m_q$ is the \emph{best worst-case margin} in the tied class, $I$ is the identity, and $\delta_q$ measures sensitivity to discarded group contrasts. Define $m_R$ using the unrestricted group-coordinate ball.

\begin{stavebox}
\begin{theorem}[Cost of tying structural token groups]
\label{thm:tying}
Under the common-count setup above, suppose each $h_j$ is differentiable on a neighborhood of $\{v:\|v\|_2\leq\rho\}$ and its gradient is $\kappa$-Lipschitz on that ball, with $\kappa\geq0$. Then
\begin{equation}
 0\leq m_R(\rho)-m_q(\rho)\leq\rho\delta_q+\kappa\rho^2.
\label{eq:tying-guarantee}
\end{equation}
For affine margins, $\delta_q=0$ suffices for equal attainable margins.
\end{theorem}
\end{stavebox}

The proof projects an unrestricted update onto $\mathcal U_q$ and bounds the omitted linear response and curvature. For readout/context sharing, $\delta_q$ measures count-weighted gradient variation within the two sets, and their mutual contrast is kept. \emph{Small $\delta_q$ and $\kappa$ therefore give a small margin cost.}

Equation~\ref{eq:stave-injection} and the objective are defined for any group counts. When group proportions vary across prompts, the theorem applies within each common-count collection.

\subsection{Paired supervision}
\label{sec:method-training}

For each training query $x_i$ with answer $y_i$, render $p_i^{\mathrm{src}}=\operatorname{render}(D_i,x_i)$ with demos $D_i$ and $p_i^{\mathrm{tgt}}=\operatorname{render}(\varnothing,x_i)$ without them. Both retain the query modality. For answer $y=(a_1,\ldots,a_L)$ of length $L\geq1$, use
\begin{align}
 \ell(y\mid p,W)&=-\frac{1}{L}\sum_{l=1}^{L}\log p_\theta(a_l\mid\widetilde E(p,W),a_{<l}),
\label{eq:stave-ce}\displaybreak[1]\\
 \mathcal L(W)&=\mathcal L_{\mathrm{src}}(W)+\mathcal L_{\mathrm{tgt}}(W),\qquad
 \mathcal L_b(W)=\frac{1}{n}\sum_{i=1}^{n}\ell(y_i\mid p_i^b,W).
\label{eq:stave-objective}
\end{align}
Here $p_\theta$ is the frozen next-token distribution, $a_{<l}$ is the correct prefix, $n$ counts training query--answer pairs, and $b\in\{\mathrm{src},\mathrm{tgt}\}$ selects a branch. The two sequence cross-entropy (CE) losses are equally weighted. Both branches of each paired minibatch contribute before an optimizer step. The vectors start at zero, and backbone parameters and answer-token input embeddings get no direct update.

The \emph{target branch} supervises deployment inputs. The \emph{source branch} fits the same vectors with demos present, using answer labels rather than soft labels or hidden-state targets. This adds a \emph{compatibility requirement}, and Proposition~\ref{prop:paired-descent} bounds its effect on the target loss after one gradient step.

\begin{stavebox}
\begin{proposition}[Target descent under paired supervision]
\label{prop:paired-descent}
Stack the learned vectors into $w\in\R^{qd}$ and write $g_b=\nabla_w\mathcal L_b(w)$ and $g=g_{\mathrm{src}}+g_{\mathrm{tgt}}$. If $\nabla\mathcal L_{\mathrm{tgt}}$ is $\beta$-Lipschitz on the segment from $w$ to $w-\eta g$, where $\beta\geq0$ and $\eta>0$, then
\begin{equation}
 \mathcal L_{\mathrm{tgt}}(w-\eta g)-\mathcal L_{\mathrm{tgt}}(w)
 \leq-\eta\langle g_{\mathrm{tgt}},g\rangle+\frac{\beta\eta^2}{2}\|g\|_2^2.
\label{eq:target-descent}
\end{equation}
\end{proposition}
\end{stavebox}

Thus $\|g_{\mathrm{tgt}}\|_2^2+\langle g_{\mathrm{tgt}},g_{\mathrm{src}}\rangle>0$ gives target descent for sufficiently small gradient steps within the smoothness region. \emph{Positive source--target alignment} also improves the first-order decrease relative to a target-only step. Appendix~\ref{app:target-descent} proves both statements and gives finite-step conditions.

\looseness=-1
For the whole objective, restrict $W$ to a Frobenius ball and define the \emph{source--target compatibility gap} $\mathcal C_{\mathrm{src}\mid\mathrm{tgt}}$ as the smallest excess source loss among target minimizers. Any $\widehat W$ within $\varepsilon_{\mathrm{opt}}\geq0$ of $\min\mathcal L$ then has excess target loss at most $\mathcal C_{\mathrm{src}\mid\mathrm{tgt}}+\varepsilon_{\mathrm{opt}}$, and the gap is zero when the two branches share a minimizer (Appendix~\ref{app:paired}). Under independent target sampling and bounded-loss and Lipschitz assumptions, Appendix~\ref{app:generalization} adds a uniform estimation term with covering dimension $qd=2d$. This separates representation size from compatibility, optimization, and statistical error.

\subsection{Demo-free inference and cost}
\label{sec:method-inference}

In LLMs, STAVE updates frozen input embeddings. In LMMs, it updates merged text and image embeddings after the frozen visual encoder and connector. At inference, add both vectors once during demo-free prefill. Generated embeddings receive no update. Storage is $2d$ task-specific scalars, independent of query count and demo length. Group counts are built from template metadata in $O(T(p))$ time, and adding the scaled vectors costs $O(T(p)d)$ operations. Prompt length and decoder prefill dimensions \emph{equal those of zero-shot inference}, as does key--value (KV) cache shape at a fixed generated length. Training is a one-time cost: both branches run the frozen model and backpropagate to the two vectors, with no backbone optimizer states or weight updates (Appendix~\ref{app:inference}).

\section{Experiments}
\label{sec:experiments}

\subsection{Setup}
\label{sec:exp-setup}
\newsavebox{\stavetabone}\newlength{\stavetabonew}\newlength{\stavequalw}
\newcount\stavecA\newcount\stavecB\newcount\stavecC

\textbf{Datasets, backbones, and protocols.}
We evaluate six frozen LMMs, LLaVA-Interleave-7B~\citep{li2024llava}, Idefics2-8B~\citep{laurencon2024matters}, Qwen-VL-7B~\citep{bai2023qwenvlversatilevisionlanguagemodel}, InternVL3.5-8B~\citep{wang2025internvl3}, Qwen2.5-VL-7B~\citep{bai2025qwen25vltechnicalreport}, and Idefics3-8B~\citep{laurenccon2024building}. The tasks are VQA with VQAv2~\citep{goyal2017vqav2}, OK-VQA~\citep{marino2019okvqa}, and VizWiz~\citep{gurari2018vizwiz}, image captioning with the COCO Karpathy split~\citep{lin2014microsoft,karpathy2015deep}, and fine-grained classification with DTD~\citep{cimpoi2014describing}, Flowers~\citep{nilsback2008automated}, and CUB~\citep{wah2011caltech}. We also evaluate five frozen LLMs, Pythia 2.8B, 6.9B, and 12B~\citep{biderman2023pythia}, GPT-J 6B~\citep{wang2021gptj}, and LLaMA 7B~\citep{touvron2023llama}, on the TV benchmark, the 18 algorithmic, translation, linguistic, and knowledge tasks of \citet{hendel-etal-2023-context}.
\begin{table*}[!t]
\centering
\caption{Comparison on VQAv2, OK-VQA, and COCO (left), with qualitative examples (right). \#~Params: \protect\colorbox{paramsmall}{low}, \protect\colorbox{parammedium}{medium}, \protect\colorbox{paramlarge}{high}. Best: \textbf{bold}, second: \underline{underlined}.}
\label{tab:vqav2-coco}
\sbox{\stavetabone}{\footnotesize\setlength{\tabcolsep}{6.5pt}%
\begin{tabular}{l|c|ccc}
\toprule
\rowcolor{tableheader}
\textbf{LLaVA} & \# Params (M) & VQAv2 & OK-VQA & COCO \\
\midrule
Zero-shot & -- & 27.65$_{\pm 0.00}$ & 12.93$_{\pm 0.00}$ & 107.25$_{\pm 0.00}$ \\
8-shot ICL & -- & 68.61$_{\pm 0.20}$ & 48.43$_{\pm 0.42}$ & 118.88$_{\pm 0.77}$ \\
16-shot ICL & -- & 66.21$_{\pm 0.48}$ & 44.47$_{\pm 0.46}$ & 124.52$_{\pm 1.05}$ \\
32-shot ICL & -- & 62.60$_{\pm 1.01}$ & 39.39$_{\pm 0.52}$ & 124.56$_{\pm 0.73}$ \\
\midrule
LoRA & \cellcolor{paramlarge}{19.7 ($\times$2399.00)} & 74.65$_{\pm 0.47}$ & 53.08$_{\pm 0.41}$ & 129.40$_{\pm 1.01}$ \\
PT-Pre & \cellcolor{parammedium}{0.082 ($\times$10.00)} & 74.73$_{\pm 0.09}$ & 54.22$_{\pm 0.17}$ & 129.72$_{\pm 0.08}$ \\
PT-App & \cellcolor{parammedium}{0.082 ($\times$10.00)} & 74.68$_{\pm 0.24}$ & 53.46$_{\pm 0.33}$ & 127.98$_{\pm 0.80}$ \\
PT-BL & \cellcolor{parammedium}{0.082 ($\times$10.00)} & 74.47$_{\pm 0.21}$ & 53.91$_{\pm 0.50}$ & 128.33$_{\pm 0.47}$ \\
\midrule
LIVE & \cellcolor{parammedium}{0.13 ($\times$16.00)} & 75.15$_{\pm 0.03}$ & 53.69$_{\pm 0.19}$ & 129.75$_{\pm 0.40}$ \\
MimIC & \cellcolor{paramlarge}{17.0 ($\times$2080.13)} & 75.19$_{\pm 0.16}$ & \underline{55.74}$_{\pm 0.25}$ & 128.09$_{\pm 0.25}$ \\
HiFICL & \cellcolor{paramlarge}{2.2 ($\times$272.00)} & \underline{75.36}$_{\pm 0.30}$ & 54.50$_{\pm 0.52}$ & \underline{130.24}$_{\pm 0.32}$ \\
\midrule
\rowcolor{stavehighlight}
\stavelabel & \cellcolor{paramsmall}\textbf{0.008 ($\times$1.00)} & \textbf{76.07}$_{\pm 0.10}$ & \textbf{55.81}$_{\pm 0.10}$ & \textbf{130.53}$_{\pm 1.17}$ \\
\midrule
\rowcolor{tableheader}
\textbf{Idefics2-8B} & \# Params (M) & VQAv2 & OK-VQA & COCO \\
\midrule
Zero-shot & -- & 47.61$_{\pm 0.00}$ & 29.52$_{\pm 0.00}$ & 80.54$_{\pm 0.00}$ \\
8-shot ICL & -- & 67.44$_{\pm 0.30}$ & 57.35$_{\pm 0.20}$ & 119.25$_{\pm 0.42}$ \\
16-shot ICL & -- & 67.56$_{\pm 0.74}$ & 56.24$_{\pm 0.46}$ & 129.53$_{\pm 0.44}$ \\
32-shot ICL & -- & 53.77$_{\pm 0.76}$ & 52.05$_{\pm 1.17}$ & 122.36$_{\pm 0.32}$ \\
\midrule
LoRA & \cellcolor{paramlarge}{17.6 ($\times$2147.00)} & 71.11$_{\pm 0.87}$ & 59.23$_{\pm 0.21}$ & \underline{132.84}$_{\pm 0.57}$ \\
PT-Pre & \cellcolor{parammedium}{0.082 ($\times$10.00)} & \underline{72.75}$_{\pm 0.14}$ & 59.20$_{\pm 0.44}$ & 131.03$_{\pm 1.51}$ \\
PT-App & \cellcolor{parammedium}{0.082 ($\times$10.00)} & 70.75$_{\pm 0.56}$ & 55.83$_{\pm 1.78}$ & 131.20$_{\pm 1.59}$ \\
PT-BL & \cellcolor{parammedium}{0.082 ($\times$10.00)} & 71.04$_{\pm 0.38}$ & 57.82$_{\pm 0.84}$ & 131.13$_{\pm 1.12}$ \\
\midrule
LIVE & \cellcolor{parammedium}{0.13 ($\times$16.00)} & 69.73$_{\pm 1.30}$ & 58.22$_{\pm 0.40}$ & 128.20$_{\pm 0.84}$ \\
MimIC & \cellcolor{parammedium}{0.26 ($\times$32.13)} & 71.34$_{\pm 0.05}$ & 59.69$_{\pm 0.37}$ & 132.42$_{\pm 0.13}$ \\
HiFICL & \cellcolor{paramlarge}{2.2 ($\times$272.00)} & 71.61$_{\pm 0.68}$ & \underline{60.29}$_{\pm 0.53}$ & 128.69$_{\pm 0.14}$ \\
\midrule
\rowcolor{stavehighlight}
\stavelabel & \cellcolor{paramsmall}\textbf{0.008 ($\times$1.00)} & \textbf{73.43}$_{\pm 0.06}$ & \textbf{61.00}$_{\pm 0.41}$ & \textbf{134.18}$_{\pm 0.82}$ \\
\bottomrule
\end{tabular}}%
\stavecA=\dimexpr\textwidth-8pt\relax
\stavecB=\numexpr\wd\stavetabone*1343604/1000000\relax
\stavecC=\numexpr\stavecB+\ht\stavetabone+\dp\stavetabone\relax
\setlength{\stavetabonew}{\numexpr\stavecA*\stavecB/\stavecC\relax sp}%
\setlength{\stavequalw}{\dimexpr\textwidth-8pt-\stavetabonew\relax}%
\begin{minipage}[c]{\stavetabonew}\resizebox{\linewidth}{!}{\usebox{\stavetabone}}\end{minipage}\hfill
\begin{minipage}[c]{\stavequalw}\includegraphics[width=\linewidth]{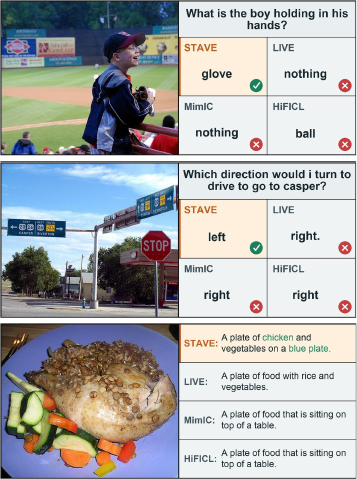}\end{minipage}

\end{table*}%
Each comparison follows the most common protocol for its benchmark. Table~\ref{tab:vqav2-coco} follows MimIC~\citep{jiang2025mimic} and HiFICL~\citep{li2026hificl}, with 1,000 training samples that also form the demo pool, evaluation on the validation split, and results from the best epoch. Tables~\ref{tab:vqa+finegrained} and~\ref{tab:vl-modern} in the appendix follow MTV~\citep{huang2024multimodal} and STV~\citep{ma2026sensitivity}, with a held-out test split. Table~\ref{tab:text} follows \citet{hendel-etal-2023-context}, with at most 256 training, 64 development, and 200 test queries per task. All trained methods answer without demos at test time. COCO is scored by CIDEr~\citep{vedantam2015cider} and all other tasks by accuracy.

\textbf{Compared methods.}
Besides zero-shot and $n$-shot ICL, we compare STAVE with three families of methods. Parameter-efficient fine-tuning includes LoRA~\citep{hu2022lora}, prefix tuning~\citep{li-liang-2021-prefix}, and PT~\citep{lester-etal-2021-power} in three modes, with its tokens before the prompt (PT-Pre), after it (PT-App), or before the final token (PT-BL). Task vectors include TV~\citep{hendel-etal-2023-context}, FV~\citep{todd2024function}, ICV~\citep{liu2024incontext}, I2CL~\citep{li2025implicit}, MTV, STV, and SITE~\citep{park2026soft}, whose values come from demo activations, and the last four learn which heads, values, or scales to use. Trained per-layer interventions include LIVE~\citep{peng2024live}, MimIC, and HiFICL. The three families differ in where they act. Fine-tuning updates weights or adds virtual tokens, task vectors inject demo activations at one layer, in every layer, or at selected heads, and trained interventions learn parameters inside every decoder layer. STAVE instead trains two vectors that are added once to the input embeddings, without changing the prompt length or adding computation inside the decoder layers.

\subsection{Implementation details}
\label{sec:exp-implementation}

STAVE learns only the two rows of $W$, which start at zero in single precision, while the backbone stays frozen in half precision. Each source prompt holds one to eight demos. We optimize the vectors with AdamW~\citep{loshchilov2017decoupled} under a cosine learning-rate schedule. Generally, each training step uses $K=4$ paired examples and a learning rate of $10^{-3}$ on LMM tasks, and $K=8$ and $10^{-1}$ on LLM tasks. We run all compared methods from their released code with the evaluation data, prompts, metrics, and decoding of each protocol. All main experiments are run three times, and we report the mean and standard deviation of the three runs. Additional details, including backbones, protocol settings, hyperparameter sensitivity, and prompt examples, are reported in Appendix~\ref{app:implementation}.
\begin{table*}[!t]
\centering
\caption{Comparison on VQA and classification datasets.}
\label{tab:vqa+finegrained}
\footnotesize
\setlength{\tabcolsep}{3.5pt}
\resizebox{\textwidth}{!}{%
\begin{tabular}{l|c|ccccc|c|ccccc|c}
\toprule
 & & \multicolumn{6}{c|}{\textbf{Idefics2-8B}} & \multicolumn{6}{c}{\textbf{Qwen-VL-7B}} \\
\cmidrule(lr){3-8}\cmidrule(lr){9-14}
\rowcolor{tableheader}
\textbf{Method} & \# Params (M) & VizWiz & OK-VQA & DTD & Flowers & CUB & \textbf{Avg.} & VizWiz & OK-VQA & DTD & Flowers & CUB & \textbf{Avg.} \\
\midrule
Zero-shot & -- & 37.74$_{\pm 0.00}$ & 49.44$_{\pm 0.00}$ & 90.23$_{\pm 0.00}$ & 87.59$_{\pm 0.00}$ & 87.84$_{\pm 0.00}$ & 70.57$_{\pm 0.00}$ & 45.49$_{\pm 0.00}$ & 57.51$_{\pm 0.00}$ & 83.77$_{\pm 0.00}$ & 73.88$_{\pm 0.00}$ & 91.10$_{\pm 0.00}$ & 70.35$_{\pm 0.00}$ \\
4-shot ICL & -- & 43.24$_{\pm 0.11}$ & 50.38$_{\pm 0.19}$ & 88.87$_{\pm 0.12}$ & 80.22$_{\pm 0.40}$ & 85.91$_{\pm 0.12}$ & 69.72$_{\pm 0.05}$ & 46.36$_{\pm 0.32}$ & \underline{60.92}$_{\pm 0.07}$ & 84.44$_{\pm 0.25}$ & 87.17$_{\pm 0.05}$ & 88.73$_{\pm 0.25}$ & 73.53$_{\pm 0.12}$ \\
\midrule
PT-Pre & \cellcolor{parammedium}{0.082 ($\times$10.00)} & 66.19$_{\pm 1.26}$ & 54.49$_{\pm 2.57}$ & \underline{94.64}$_{\pm 0.35}$ & 87.50$_{\pm 0.41}$ & 94.25$_{\pm 0.53}$ & 79.41$_{\pm 0.77}$ & 63.83$_{\pm 0.82}$ & 50.56$_{\pm 0.73}$ & 94.44$_{\pm 0.85}$ & 93.83$_{\pm 0.26}$ & 95.65$_{\pm 0.71}$ & 79.66$_{\pm 0.29}$ \\
PT-App & \cellcolor{parammedium}{0.082 ($\times$10.00)} & 64.56$_{\pm 2.43}$ & 55.10$_{\pm 0.56}$ & 94.14$_{\pm 0.07}$ & 89.60$_{\pm 0.10}$ & \underline{95.10}$_{\pm 0.15}$ & 79.70$_{\pm 0.63}$ & 60.19$_{\pm 3.26}$ & 37.84$_{\pm 4.11}$ & 93.10$_{\pm 0.96}$ & 93.56$_{\pm 0.28}$ & 94.14$_{\pm 0.43}$ & 75.77$_{\pm 1.03}$ \\
PT-BL & \cellcolor{parammedium}{0.082 ($\times$10.00)} & 66.35$_{\pm 0.21}$ & 56.58$_{\pm 0.98}$ & 94.13$_{\pm 0.19}$ & 88.97$_{\pm 0.52}$ & 94.55$_{\pm 1.54}$ & 80.12$_{\pm 0.22}$ & 60.30$_{\pm 4.64}$ & 44.35$_{\pm 3.61}$ & 93.44$_{\pm 0.28}$ & 93.49$_{\pm 0.29}$ & 94.73$_{\pm 0.50}$ & 77.26$_{\pm 0.48}$ \\
\midrule
TV & \cellcolor{paramsmall}\textbf{0.004 ($\times$0.50)} & 46.29$_{\pm 6.90}$ & 50.23$_{\pm 0.90}$ & 90.58$_{\pm 0.49}$ & 87.52$_{\pm 0.02}$ & 87.99$_{\pm 0.26}$ & 72.52$_{\pm 1.41}$ & 49.37$_{\pm 2.41}$ & 59.03$_{\pm 0.68}$ & 91.00$_{\pm 1.80}$ & 93.59$_{\pm 0.06}$ & 92.69$_{\pm 1.36}$ & 77.14$_{\pm 0.83}$ \\
ICV & \cellcolor{parammedium}{0.13 ($\times$16.00)} & 37.75$_{\pm 0.02}$ & 49.49$_{\pm 0.03}$ & 90.30$_{\pm 0.00}$ & 87.80$_{\pm 0.35}$ & 87.86$_{\pm 0.02}$ & 70.64$_{\pm 0.07}$ & 45.57$_{\pm 0.10}$ & 57.81$_{\pm 0.39}$ & 90.07$_{\pm 0.12}$ & 92.22$_{\pm 0.02}$ & 91.51$_{\pm 0.01}$ & 75.44$_{\pm 0.08}$ \\
I2CL & \cellcolor{parammedium}{0.26 ($\times$32.02)} & 42.15$_{\pm 5.87}$ & 47.41$_{\pm 1.99}$ & 92.59$_{\pm 0.86}$ & 89.25$_{\pm 0.32}$ & 93.40$_{\pm 1.73}$ & 72.96$_{\pm 1.01}$ & 53.22$_{\pm 2.04}$ & 51.55$_{\pm 3.02}$ & 90.77$_{\pm 1.62}$ & 93.75$_{\pm 0.25}$ & 93.36$_{\pm 0.97}$ & 76.53$_{\pm 0.26}$ \\
MTV & \cellcolor{parammedium}{0.13 ($\times$16.00)} & 50.85$_{\pm 4.94}$ & 50.30$_{\pm 0.24}$ & 93.73$_{\pm 0.26}$ & 89.46$_{\pm 0.17}$ & 94.15$_{\pm 0.72}$ & 75.70$_{\pm 1.06}$ & 52.53$_{\pm 6.36}$ & 57.68$_{\pm 1.17}$ & 94.27$_{\pm 0.21}$ & 93.73$_{\pm 0.30}$ & 94.66$_{\pm 0.15}$ & 78.57$_{\pm 1.35}$ \\
STV & \cellcolor{paramsmall}\textbf{0.008 ($\times$1.00)} & 61.39$_{\pm 0.20}$ & 49.44$_{\pm 0.00}$ & 91.57$_{\pm 1.53}$ & 87.83$_{\pm 0.41}$ & 93.38$_{\pm 1.01}$ & 76.72$_{\pm 0.35}$ & 58.15$_{\pm 0.90}$ & 58.17$_{\pm 0.75}$ & 76.32$_{\pm 8.90}$ & 88.64$_{\pm 5.47}$ & 91.10$_{\pm 0.00}$ & 74.48$_{\pm 3.13}$ \\
LIVE & \cellcolor{parammedium}{0.13 ($\times$16.00)} & 59.74$_{\pm 0.45}$ & 54.12$_{\pm 0.17}$ & 93.69$_{\pm 0.11}$ & 87.80$_{\pm 0.98}$ & 93.36$_{\pm 0.49}$ & 77.74$_{\pm 0.25}$ & 62.40$_{\pm 0.44}$ & 53.85$_{\pm 0.27}$ & 78.91$_{\pm 0.10}$ & 86.33$_{\pm 0.22}$ & 86.87$_{\pm 0.41}$ & 73.67$_{\pm 0.14}$ \\
MimIC & \cellcolor{parammedium}{0.26 ($\times$32.13)} & \underline{68.46}$_{\pm 1.72}$ & \underline{59.30}$_{\pm 0.70}$ & 94.61$_{\pm 0.60}$ & 90.45$_{\pm 3.54}$ & 93.30$_{\pm 1.22}$ & \underline{81.22}$_{\pm 0.35}$ & \underline{67.83}$_{\pm 0.89}$ & 55.21$_{\pm 0.95}$ & 95.16$_{\pm 0.52}$ & 95.79$_{\pm 1.00}$ & \underline{96.35}$_{\pm 0.47}$ & 82.07$_{\pm 0.54}$ \\
HiFICL & \cellcolor{paramlarge}{2.2 ($\times$272.00)} & 65.56$_{\pm 2.05}$ & 59.24$_{\pm 1.19}$ & 92.29$_{\pm 0.75}$ & \underline{91.93}$_{\pm 0.73}$ & \textbf{95.18}$_{\pm 0.39}$ & 80.84$_{\pm 0.74}$ & 67.10$_{\pm 1.44}$ & 55.78$_{\pm 1.27}$ & \textbf{95.47}$_{\pm 0.27}$ & \textbf{97.29}$_{\pm 0.56}$ & \textbf{97.05}$_{\pm 0.34}$ & \underline{82.54}$_{\pm 0.25}$ \\
\midrule
\rowcolor{stavehighlight}
\stavelabel & \cellcolor{paramsmall}\textbf{0.008 ($\times$1.00)} & \textbf{69.63}$_{\pm 0.37}$ & \textbf{59.59}$_{\pm 0.50}$ & \textbf{95.36}$_{\pm 0.47}$ & \textbf{95.67}$_{\pm 0.42}$ & 93.50$_{\pm 0.72}$ & \textbf{82.75}$_{\pm 0.19}$ & \textbf{68.60}$_{\pm 0.55}$ & \textbf{62.39}$_{\pm 0.12}$ & \underline{95.20}$_{\pm 0.32}$ & \underline{96.30}$_{\pm 0.23}$ & 95.56$_{\pm 0.26}$ & \textbf{83.61}$_{\pm 0.14}$ \\
\bottomrule
\end{tabular}%
}
\end{table*}

\subsection{Main results}
\label{sec:exp-main}

\textbf{Comparison with trained methods on LMMs.}
Tables~\ref{tab:vqav2-coco} and~\ref{tab:vqa+finegrained} test whether an update at the input embeddings alone can match methods that learn a shift in every decoder layer. STAVE has the best score on every column of Table~\ref{tab:vqav2-coco} and the best average on both backbones of Table~\ref{tab:vqa+finegrained}, and it leads on five of six columns with three more recent LMMs (Appendix~\ref{app:recent-lmms}). It stores 14 to 2,080 times fewer parameters than LIVE, MimIC, and HiFICL and keeps zero-shot latency (Figure~\ref{fig:performance-overview}), since an input update already acts through every block (Proposition~\ref{prop:depth-reach}). PT also trains at the input, but it inserts virtual tokens and its best mode varies with the backbone and dataset. STAVE outperforms all three PT modes on average with 10 times fewer parameters and without lengthening the prompt. With over 2,000 times more parameters, LoRA still trails STAVE on every column of Table~\ref{tab:vqav2-coco}.

\begin{table*}[!t]
\centering
\caption{Comparison on the TV benchmark (18 text tasks).}
\label{tab:text}
\footnotesize
\renewcommand{\arraystretch}{1.2}
\resizebox{\textwidth}{!}{%
\begin{tabular}{l|c|ccc|c|c|ccc|ccc|>{\columncolor{stavehighlight}}c}
\toprule
 & & \multicolumn{3}{c|}{Vanilla ICL} & & & \multicolumn{3}{c|}{PT} & & & & \multicolumn{1}{c}{} \\
\cmidrule(lr){3-5}\cmidrule(lr){8-10}
\rowcolor{tableheader}
Model & Zero-shot & $n{=}1$ & $n{=}5$ & $n{=}15$ & LoRA & Prefix & Pre & App & BL & FV & TV & SITE & \multicolumn{1}{c}{\textbf{STAVE}} \\
\midrule
Pythia 2.8B & 13.1$_{\pm 0.0}$ & 60.4$_{\pm 0.6}$ & 85.6$_{\pm 0.4}$ & \underline{90.6}$_{\pm 0.2}$ & 79.1$_{\pm 5.1}$ & 61.8$_{\pm 1.6}$ & 70.4$_{\pm 5.7}$ & 79.4$_{\pm 6.7}$ & 78.4$_{\pm 5.2}$ & 40.7$_{\pm 2.1}$ & 75.6$_{\pm 0.1}$ & 89.9$_{\pm 0.2}$ & \textbf{91.2}$_{\pm 0.3}$ \\
Pythia 6.9B & 10.9$_{\pm 0.0}$ & 65.0$_{\pm 0.9}$ & 86.6$_{\pm 0.5}$ & \underline{90.6}$_{\pm 0.1}$ & 78.1$_{\pm 2.4}$ & 65.7$_{\pm 2.6}$ & 70.5$_{\pm 6.2}$ & 83.5$_{\pm 1.8}$ & 76.1$_{\pm 8.2}$ & 42.4$_{\pm 3.0}$ & 75.5$_{\pm 0.2}$ & 90.2$_{\pm 0.1}$ & \textbf{91.4}$_{\pm 0.2}$ \\
Pythia 12B & 11.4$_{\pm 0.0}$ & 64.5$_{\pm 0.3}$ & 87.0$_{\pm 0.4}$ & 91.0$_{\pm 0.1}$ & 81.5$_{\pm 4.6}$ & 55.2$_{\pm 3.8}$ & 71.8$_{\pm 4.1}$ & 82.9$_{\pm 1.0}$ & 73.9$_{\pm 3.9}$ & 42.2$_{\pm 0.8}$ & 75.0$_{\pm 0.4}$ & \underline{91.1}$_{\pm 0.4}$ & \textbf{91.2}$_{\pm 0.5}$ \\
LLaMA 7B & 17.1$_{\pm 0.0}$ & 67.8$_{\pm 0.5}$ & 88.4$_{\pm 0.2}$ & 90.6$_{\pm 0.3}$ & 76.1$_{\pm 2.7}$ & 70.2$_{\pm 6.8}$ & 59.9$_{\pm 3.8}$ & 87.5$_{\pm 2.5}$ & 82.5$_{\pm 4.7}$ & 59.4$_{\pm 0.5}$ & 80.4$_{\pm 0.9}$ & \textbf{92.0}$_{\pm 0.2}$ & \underline{91.9}$_{\pm 0.4}$ \\
GPT-J 6B & 10.2$_{\pm 0.0}$ & 63.6$_{\pm 0.4}$ & 86.3$_{\pm 0.3}$ & 91.0$_{\pm 0.2}$ & 73.7$_{\pm 4.0}$ & 66.5$_{\pm 6.1}$ & 77.8$_{\pm 3.6}$ & 86.0$_{\pm 0.8}$ & 84.9$_{\pm 3.2}$ & 56.3$_{\pm 1.7}$ & 72.4$_{\pm 0.6}$ & \underline{91.8}$_{\pm 0.2}$ & \textbf{91.9}$_{\pm 0.4}$ \\
\midrule
\textbf{Avg.} & 12.5$_{\pm 0.0}$ & 64.3$_{\pm 0.3}$ & 86.8$_{\pm 0.3}$ & 90.8$_{\pm 0.1}$ & 77.7$_{\pm 0.4}$ & 63.9$_{\pm 0.3}$ & 70.1$_{\pm 3.4}$ & 83.9$_{\pm 1.3}$ & 79.1$_{\pm 2.6}$ & 48.2$_{\pm 0.9}$ & 75.8$_{\pm 0.3}$ & \underline{91.0}$_{\pm 0.1}$ & \textbf{91.5}$_{\pm 0.3}$ \\
\bottomrule
\end{tabular}
}
\end{table*}
\textbf{Comparison with ICL and task vectors.}
STAVE outperforms ICL with up to 32 demos on every column of Table~\ref{tab:vqav2-coco} while using none, and ICL peaks at 8 or 16 demos on VQA. In Table~\ref{tab:vqa+finegrained}, TV and ICV add demo activations without training and stay near zero-shot on Idefics2. I2CL learns their scales, and MTV and STV learn which heads receive them, yet all trail STAVE on average, even STV at the same size, so fitting the values to the answer matters more than choosing where to add them.

\textbf{Comparison on LLMs.}
Table~\ref{tab:text} tests whether the design transfers to LLMs. STAVE has the highest average on the TV benchmark and exceeds 15-shot ICL on every LLM without demos at inference. SITE, which edits head activations in every layer, comes closest, yet STAVE is slightly higher with an input update alone. TV and FV remain far below. LoRA, prefix tuning, and PT train more parameters yet stay below 5-shot ICL on every LLM, and in Tables~\ref{tab:vqa+finegrained} and~\ref{tab:text} PT varies far more across runs than STAVE, whose zero-initialized vectors start at the exact zero-shot computation.

\subsection{Ablation studies}
\label{sec:exp-ablation}

\begin{wrapfigure}{R}{0.4\textwidth}
\vspace{-\intextsep}
\centering
\includegraphics[width=\linewidth,trim=0 7 0 7,clip]{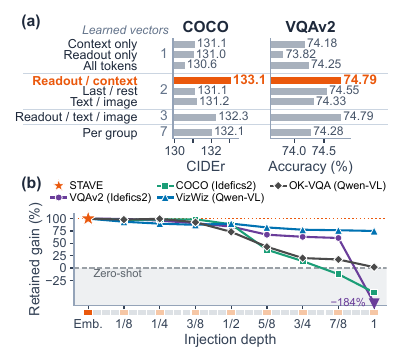}\par
\vspace{-9pt}
\caption{Ablations. (a) Groups split into 1, 2, 3, or 7 vectors, mean of Idefics2 and LLaVA. (b) Injection depth.}
\label{fig:ablation}\label{fig:groups-main}\label{fig:depth-ablation}
\vspace{-\intextsep}
\end{wrapfigure}
\textbf{Structural token groups.}
Figure~\ref{fig:ablation}(a) changes only how groups are assigned to vectors. The readout/context split is best on COCO and tied for best on VQAv2. One vector is never enough, and one for all tokens is weakest on COCO, as the dilution in Equation~\ref{eq:sharing-bound} predicts. Two vectors also fall short if the answer cue joins the context vector (Last / rest) or they split by modality (Text / image). Both put readout tokens in one vector with context tokens of smaller gradient per token (Figure~\ref{fig:motivation-chain}(b)), so the gain comes from separating answer-producing tokens from their context. Three or seven vectors add parameters without a gain, consistent with Theorem~\ref{thm:tying}. Appendix~\ref{app:ablations} adds equal-norm random splits and count normalization variants.

\textbf{Injection depth.}
Figure~\ref{fig:ablation}(b) adds the same vectors after a decoder block instead of at the input (Appendix~\ref{app:depth}). The embedding injection is best or tied for best on all four curves, and the retained gain shrinks with depth, falling below zero-shot after the last block on two of them. This matches Proposition~\ref{prop:depth-reach}: a later update reaches fewer blocks, and after the last block it can change only the first answer token, and only through the final token.
\newsavebox{\stavefigcol}
\begin{table}[!t]
\sbox{\stavefigcol}{\begin{minipage}[t]{0.50\textwidth}
\centering
\expandafter\def\csname @captype\endcsname{figure}
\addtocounter{figure}{2}%
\vspace{7pt}%
\includegraphics[width=\linewidth]{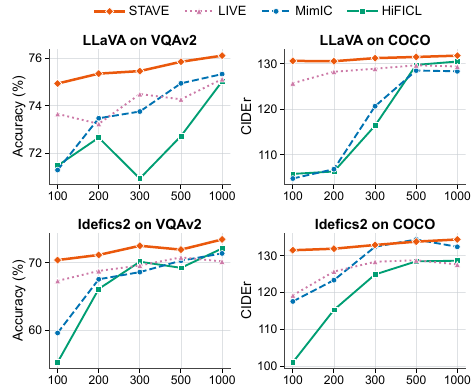}
\par\vspace{-7.2pt}
\caption{Data efficiency across training set sizes.}
\label{fig:sample-size}
\end{minipage}}%
\newsavebox{\stavetabcol}
\sbox{\stavetabcol}{\begin{minipage}[t]{0.405\textwidth}
\centering
\caption{Training-objective ablation.}
\label{tab:training-objective-ablation}
\footnotesize
\setlength{\tabcolsep}{4.2pt}
\renewcommand{\arraystretch}{1.0}
\resizebox{\linewidth}{!}{%
\begin{tabular}{lccc}
\toprule
\rowcolor{tableheader}
Backbone / Dataset & Paired & Target-only & Source-only \\
\midrule
Idefics2 VQAv2 & \textbf{73.46\%} & \underline{73.32\%} & 72.41\% \\
LLaVA VQAv2    & \textbf{76.12\%} & \underline{76.01\%} & 74.98\% \\
Idefics2 COCO  & \textbf{134.38} & \underline{134.01} & 131.34 \\
LLaVA COCO     & \textbf{131.78} & \underline{129.80} & 129.11 \\
\midrule
TV bench       & \textbf{92.07\%} & \underline{91.53\%} & 79.67\% \\
\quad w/ 1 demo & \textbf{92.23\%} & 88.90\% & \underline{92.10\%} \\
\bottomrule
\end{tabular}%
}

\vspace{3pt}
\caption{Hallucination on COCO.}
\label{tab:hallucination}
\vspace{2pt}
\footnotesize
\setlength{\tabcolsep}{3pt}
\renewcommand{\arraystretch}{1.05}
\resizebox{\linewidth}{!}{%
\begin{tabular}{l|cccc}
\toprule
\rowcolor{tableheader}
Idefics2 COCO & CIDEr $\uparrow$ & CHAIRs $\downarrow$ & CHAIRi $\downarrow$ & Recall $\uparrow$ \\
\midrule
Zero-shot & 80.54 & 3.44 & 2.84 & 35.73 \\
8-shot ICL & 119.07 & 3.06 & 2.26 & 41.76 \\
\midrule
LoRA & 132.37 & \underline{2.46} & \underline{1.68} & 43.64 \\
LIVE & 127.62 & 2.56 & 1.83 & 43.15 \\
MimIC & \underline{132.38} & 3.48 & 2.28 & \underline{45.62} \\
HiFICL & 128.62 & 3.30 & 2.10 & \textbf{45.98} \\
\midrule
\rowcolor{stavehighlight}
\stavelabel & \textbf{133.27} & \textbf{2.24} & \textbf{1.53} & 44.20 \\
\bottomrule
\end{tabular}%
}
\par\kern0pt
\end{minipage}}%
\usebox{\stavetabcol}\hfill\usebox{\stavefigcol}
\end{table}
\addtocounter{figure}{-3}%

\textbf{Paired supervision.}
Table~\ref{tab:training-objective-ablation} drops one branch of Equation~\ref{eq:stave-objective}. Paired training scores highest in every row, so demos in training also help the demo-free prompt, as Proposition~\ref{prop:paired-descent} allows for aligned branch gradients. The source branch reuses each labeled query in a second prompt, and since paired vectors exceed target-only vectors on the demo-free prompt in every row, this extra signal costs no target performance, consistent with a small compatibility gap $\mathcal C_{\mathrm{src}\mid\mathrm{tgt}}$ (Section~\ref{sec:method-training}). Each single branch fits only the prompt format it sees. Source-only vectors are lowest for every LMM and dataset. Target-only vectors lose 2.6 points on the TV benchmark with one demo, lose most of their LLaVA COCO score with demos, and hallucinate more on Idefics2 COCO (Appendices~\ref{app:demo-robustness-coco} and~\ref{app:objective-hallucination}). MimIC and HiFICL, which also train only on demo-free prompts, fall below 5 CIDEr on LLaVA COCO once a demo is present. Paired vectors hold in both formats, so one checkpoint serves prompts with and without demos.

\staveneedspace{0.4\textheight}
\subsection{Analysis}
\label{sec:exp-analysis}

\begin{wrapfigure}{R}{0.50\textwidth}
\vspace{-\intextsep}
\centering
\includegraphics[width=\linewidth,trim=11.5 13.1 7.2 9.6,clip]{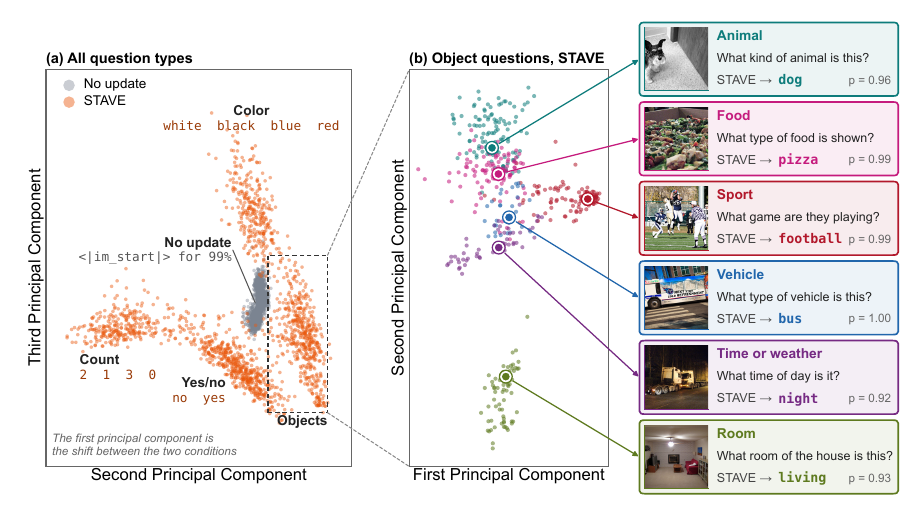}\par
\vspace{-9pt}
\caption{Final-token states of LLaVA on VQAv2 validation questions, without and with STAVE.}
\label{fig:representation-map}
\par\vspace{2pt}
\includegraphics[width=0.92\linewidth,trim=0 5.9 0 0,clip]{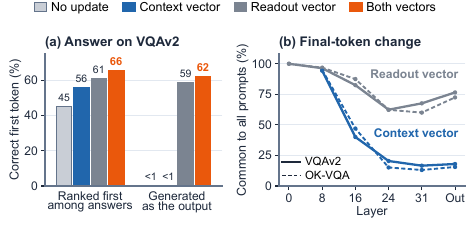}\par
\vspace{-10pt}
\caption{Effect of each learned vector.}
\label{fig:row-roles-main}
\vspace{-\intextsep}
\end{wrapfigure}
\addtocounter{figure}{1}%
\textbf{Query-specific answers from one \mbox{update}.}
The same two vectors are added to every prompt, yet the answer changes with the query (Figure~\ref{fig:representation-map}). Without the update, LLaVA predicts a template token for almost every VQAv2 question. STAVE moves the states of each answer type in a different direction, and within object questions it separates the categories they ask about, because the frozen blocks process the vectors together with each query. Table~\ref{tab:vqav2-coco} shows examples.

\textbf{What each vector contributes.}
Figure~\ref{fig:row-roles-main}(a) adds the Idefics2 VQAv2 vectors one at a time to held-out prompts without retraining and reads the reference first token. Without the update, this token often already ranks first among the task's answers, yet the model almost never generates it. The context vector alone ranks it first more often but rarely changes the output, whereas the readout vector alone makes the model answer. Both vectors score highest. Figure~\ref{fig:row-roles-main}(b) shows why the vectors play different roles. On VQAv2 and OK-VQA, most of the change that the readout vector makes to the final-token state is common to all prompts, whereas the change that the context vector adds becomes specific to each prompt from the middle layers on. The readout vector thus sets a shared answer format, and the context vector supplies query-specific evidence for choosing the answer. Appendix~\ref{app:analysis} adds more tests of each vector.

\textbf{Why a small update is sufficient.}
STAVE's $2d$ parameters draw on what the frozen backbone already holds (Figure~\ref{fig:row-roles-main}(a)) and leave it intact. On Qwen-VL OK-VQA, PT and all per-layer interventions in Table~\ref{tab:vqa+finegrained} fall below zero-shot, while STAVE exceeds it (Appendix~\ref{app:impl-table2}). The small update makes STAVE data efficient (Figure~\ref{fig:sample-size}). With a tenth of the training samples it loses a few points, whereas MimIC and HiFICL lose far more on Idefics2, consistent with the risk bound of Section~\ref{sec:method-training}, whose estimation term depends on the covering dimension $2d$, not on the number of decoder layers.

\begin{figure}[t]
\centering
\includegraphics[width=0.6\textwidth]{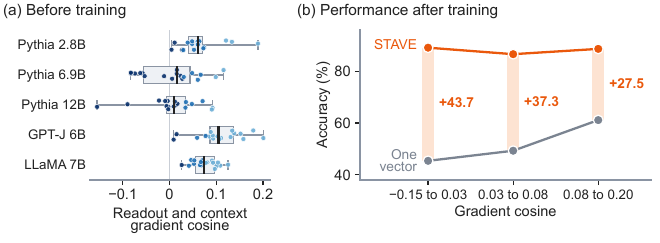}
\caption{Gradient cosine before training (a) and performance after training (b) on the TV benchmark.}
\label{fig:direction-gap}
\end{figure}
\textbf{Separate vectors help most when gradients disagree.}
Proposition~\ref{prop:sharing-cost}(iii) predicts when STAVE's two vectors matter. For the readout set $\mathcal A$ and the context tokens $\mathcal B$, the first-order decrease of one shared vector grows with the cosine between the two sets' summed gradients, whereas that of two vectors does not depend on it, so their gap widens as the cosine falls. Figure~\ref{fig:direction-gap} tests this prediction on the 18 tasks of the TV benchmark with each of the five LLMs. Before training, we measure this cosine for each task and LLM, and after training we compare STAVE with a single vector shared by all eligible tokens. The cosines lie between $-0.15$ and $0.20$ (Figure~\ref{fig:direction-gap}(a)), so the readout and context tokens ask for nearly orthogonal updates in every LLM. As predicted, the shared vector performs better as the cosine grows, whereas STAVE performs similarly across the whole range (Figure~\ref{fig:direction-gap}(b)). STAVE's advantage is thus largest where the two gradients disagree most, 43.7 points in the lowest third of cosines against 27.5 points in the highest, and the cosine is negatively rank-correlated with this advantage for each of the five LLMs. Theorem~\ref{thm:tying} gives the same picture for margins. Its bound grows with the gradient component that a tied class discards, and one vector also discards the readout--context contrast that STAVE keeps. Because even the highest cosines stay far below one, the condition under which one vector suffices does not arise in these models, and the second vector lets STAVE perform well whether or not the two token sets agree.

\textbf{Fine-grained classification.}
STAVE leads on Idefics2 DTD and Flowers, by 3.7 points on Flowers, and stays within 1.7 points of the best method on the other four columns, all led by HiFICL. The test classes of these two-choice tasks never appear as training queries, so correcting a zero-shot error means separating similar classes that no training query contains. HiFICL can add query-specific visual evidence in every layer with 272 times more parameters, whereas the two vectors of STAVE act only through the frozen blocks and draw on the classes the backbone already recognizes.

\begin{wraptable}{R}{0.44\textwidth}
\vspace{-\intextsep}
\centering
\caption{Inference cost on Idefics2 VQAv2.}
\label{tab:cost-main}
\footnotesize
\setlength{\tabcolsep}{3pt}
\resizebox{0.98\linewidth}{!}{%
\begin{tabular}{l|rrr}
\toprule
\rowcolor{tableheader}
vs.\ ICL & Tokens & TTFT (ms) & Mem. (GiB) \\
\midrule
Zero-shot & \textbf{128} & \textbf{98} & \textbf{16.5} \\
8-shot ICL & 1{,}026 & 1{,}430 & 29.8 \\
16-shot ICL & 1{,}922 & 2{,}854 & 43.0 \\
32-shot ICL & 3{,}717 & 5{,}877 & 68.7 \\
\rowcolor{stavehighlight}
\stavelabel & \textbf{128} & \textbf{98} & \textbf{16.5} \\
\midrule
\rowcolor{tableheader}
vs.\ trained & State (KiB) & TTFT (ms) & Decode (ms/tok) \\
\midrule
LoRA & 68{,}704 & 116 & 33.6 \\
LIVE & 512 & 101 & 27.4 \\
MimIC & 1{,}028 & 106 & 31.6 \\
HiFICL & 8{,}704 & 108 & 34.3 \\
\rowcolor{stavehighlight}
\stavelabel & \textbf{32} & \textbf{98} & \textbf{23.9} \\
\bottomrule
\end{tabular}%
}
\vspace{-\intextsep}
\end{wraptable}
\textbf{Object hallucination.}
Table~\ref{tab:hallucination} checks what drives the CIDEr gain on Idefics2 COCO. STAVE has the highest CIDEr and lowest CHAIRs and CHAIRi~\citep{rohrbach2018object}, the shares of captions and of mentions that name absent objects. MimIC and HiFICL recall more objects, but their CHAIRs stays near the zero-shot level, while STAVE lowers it by a third, so its gain comes from more accurate mentions. In the COCO example of Table~\ref{tab:vqav2-coco}, STAVE names the chicken, whereas MimIC and HiFICL describe only a plate of food.

\textbf{Inference cost.}
Table~\ref{tab:cost-main} measures the per-query cost of the Idefics2 VQAv2 checkpoints of Table~\ref{tab:vqav2-coco} on one NVIDIA H200. STAVE adds its vectors before the first decoder block, so its decoder computation is that of zero-shot inference (Appendix~\ref{app:cost-measured}). ICL re-encodes its demos at every query, and with 32 demos its time to first token (TTFT) is 60 times that of STAVE, and its memory grows with the demo count. The trained methods keep the prompt length but add computation inside the decoder layers, which slows every decoded token. STAVE has the lowest cost in every column of both panels.

\section{Related Work}
\label{sec:related-work}

ICL puts demos in every query's prompt \citep{brown2020language}, and in LMMs each demo image adds many visual tokens \citep{laurencon2024matters}.
Extracted task vectors reuse demo activations as one vector at one layer \citep{hendel-etal-2023-context,todd2024function}, at heads selected per task \citep{huang2024multimodal,ma2026sensitivity,park2026soft}, or in every layer with fixed or learned scales \citep{liu2024incontext,li2025implicit}, and STAVE beats all seven on average, even STV at equal size.
Learned task vectors instead train one vector from the answer at a chosen layer and token \citep{yang2026task}.
STAVE also learns from the answer, but on tokens set by prompt structure, so it needs no layer, head, or token search.
Trained interventions add per-layer vectors \citep{peng2024live}, per-head shifts \citep{jiang2025mimic}, attention-logit biases \citep{talemi2026hyper}, or virtual keys and values \citep{li2026hificl} to every decoder layer, and all four train only on demo-free prompts, three against an ICL teacher.
STAVE stores $2d$ input scalars regardless of depth, adds no per-layer computation, and scores higher on average.
Its teacher-free training with and without demos beats demo-free training (Table~\ref{tab:training-objective-ablation}).
PT and prefix tuning insert new tokens or attention memory \citep{lester-etal-2021-power,li-liang-2021-prefix,gupta2026meta}, and within a layer added keys scale a token's original attention weights by one factor, so its relative attention over image and text stays fixed \citep{petrov2024prompting}.
STAVE writes onto existing tokens, so it can change this relative attention, keeps the zero-shot prompt length, and outperforms 20-token PT with 10 times fewer parameters.
Steering vectors and ReFT also edit existing hidden states, but at chosen layers and tokens \citep{turner2023activation,wu2024reft}.
Adapters and LoRA modify decoder modules per task \citep{houlsby2019parameter,hu2022lora,liu2022few}, whereas STAVE swaps two input vectors, shares every module across tasks, and outperforms LoRA with orders of magnitude fewer parameters.
Appendix~\ref{app:related-work} extends this discussion.

\section{Conclusion}
\label{sec:conclusion}

We presented STAVE, which replaces in-context demos with two task-specific vectors added to the input embeddings of a frozen LLM or LMM. This is the only depth that reaches every decoder block, and the vectors are added to existing tokens, one for the answer-producing tokens and one for the other five groups. Trained with and without demos, they match or outperform per-layer interventions on six LMMs with 14 to 2,080 times fewer task parameters and exceed 15-shot ICL on five LLMs, while keeping the prompt length, latency, and memory of zero-shot inference. The ablations show that the gain comes from separating the answer-producing tokens from their context, not from more vectors or a modality split, and that training with demos also helps the demo-free prompt. Our analysis suggests that the vectors mainly make the frozen model output answers it already represents, and their captions name fewer absent objects. Training is a one-time cost but backpropagates through the frozen decoder, so its memory grows with the backbone, and it requires access to the weights. Since each task adds only two vectors and no prompt tokens, one deployed backbone can serve many tasks. Extending structural token groups to multi-image and video prompts is a natural next step.

\bibliography{references}

\begin{thebibliography}{74}
\providecommand{\natexlab}[1]{#1}
\providecommand{\url}[1]{\texttt{#1}}
\expandafter\ifx\csname urlstyle\endcsname\relax
  \providecommand{\doi}[1]{doi: #1}\else
  \providecommand{\doi}{doi: \begingroup \urlstyle{rm}\Url}\fi

\bibitem[Alayrac et~al.(2022)Alayrac, Donahue, Luc, Miech, Barr, Hasson, Lenc, Mensch, Millican, Reynolds, et~al.]{alayrac2022flamingo}
Jean-Baptiste Alayrac, Jeff Donahue, Pauline Luc, Antoine Miech, Iain Barr, Yana Hasson, Karel Lenc, Arthur Mensch, Katherine Millican, Malcolm Reynolds, et~al.
\newblock Flamingo: a visual language model for few-shot learning.
\newblock \emph{Advances in neural information processing systems}, 35:\penalty0 23716--23736, 2022.

\bibitem[Bai et~al.(2023)Bai, Bai, Yang, Wang, Tan, Wang, Lin, Zhou, and Zhou]{bai2023qwenvlversatilevisionlanguagemodel}
Jinze Bai, Shuai Bai, Shusheng Yang, Shijie Wang, Sinan Tan, Peng Wang, Junyang Lin, Chang Zhou, and Jingren Zhou.
\newblock Qwen-vl: A versatile vision-language model for understanding, localization, text reading, and beyond, 2023.
\newblock URL \url{https://arxiv.org/abs/2308.12966}.

\bibitem[Bai et~al.(2025)Bai, Chen, Liu, Wang, Ge, Song, Dang, Wang, Wang, Tang, Zhong, Zhu, Yang, Li, Wan, Wang, Ding, Fu, Xu, Ye, Zhang, Xie, Cheng, Zhang, Yang, Xu, and Lin]{bai2025qwen25vltechnicalreport}
Shuai Bai, Keqin Chen, Xuejing Liu, Jialin Wang, Wenbin Ge, Sibo Song, Kai Dang, Peng Wang, Shijie Wang, Jun Tang, Humen Zhong, Yuanzhi Zhu, Mingkun Yang, Zhaohai Li, Jianqiang Wan, Pengfei Wang, Wei Ding, Zheren Fu, Yiheng Xu, Jiabo Ye, Xi~Zhang, Tianbao Xie, Zesen Cheng, Hang Zhang, Zhibo Yang, Haiyang Xu, and Junyang Lin.
\newblock Qwen2.5-vl technical report, 2025.
\newblock URL \url{https://arxiv.org/abs/2502.13923}.

\bibitem[Biderman et~al.(2023)Biderman, Schoelkopf, Anthony, Bradley, O’Brien, Hallahan, Khan, Purohit, Prashanth, Raff, et~al.]{biderman2023pythia}
Stella Biderman, Hailey Schoelkopf, Quentin~Gregory Anthony, Herbie Bradley, Kyle O’Brien, Eric Hallahan, Mohammad~Aflah Khan, Shivanshu Purohit, USVSN~Sai Prashanth, Edward Raff, et~al.
\newblock Pythia: A suite for analyzing large language models across training and scaling.
\newblock In \emph{International conference on machine learning}, pp.\  2397--2430. PMLR, 2023.

\bibitem[Brown et~al.(2020)Brown, Mann, Ryder, Subbiah, Kaplan, Dhariwal, Neelakantan, Shyam, Sastry, Askell, et~al.]{brown2020language}
Tom Brown, Benjamin Mann, Nick Ryder, Melanie Subbiah, Jared~D Kaplan, Prafulla Dhariwal, Arvind Neelakantan, Pranav Shyam, Girish Sastry, Amanda Askell, et~al.
\newblock Language models are few-shot learners.
\newblock \emph{Advances in neural information processing systems}, 33:\penalty0 1877--1901, 2020.

\bibitem[Cai et~al.(2025)Cai, Huang, Wang, and Wu]{cai2025beyond}
Wang Cai, Hsiu-Yuan Huang, Zhixiang Wang, and Yunfang Wu.
\newblock Beyond demonstrations: Dynamic vector construction from latent representations.
\newblock In \emph{Proceedings of the 2025 Conference on Empirical Methods in Natural Language Processing}, pp.\  5853--5868, 2025.

\bibitem[Cimpoi et~al.(2014)Cimpoi, Maji, Kokkinos, Mohamed, and Vedaldi]{cimpoi2014describing}
Mircea Cimpoi, Subhransu Maji, Iasonas Kokkinos, Sammy Mohamed, and Andrea Vedaldi.
\newblock Describing textures in the wild.
\newblock In \emph{Proceedings of the IEEE conference on computer vision and pattern recognition}, pp.\  3606--3613, 2014.

\bibitem[Ding \& Wang(2025)Ding and Wang]{ding2025language}
Xi~Ding and Lei Wang.
\newblock Do language models understand time?
\newblock In \emph{Companion Proceedings of the ACM on Web Conference 2025}, pp.\  1855--1868, 2025.

\bibitem[Dong et~al.(2024)Dong, Li, Dai, Zheng, Ma, Li, Xia, Xu, Wu, Chang, et~al.]{dong2024survey}
Qingxiu Dong, Lei Li, Damai Dai, Ce~Zheng, Jingyuan Ma, Rui Li, Heming Xia, Jingjing Xu, Zhiyong Wu, Baobao Chang, et~al.
\newblock A survey on in-context learning.
\newblock In \emph{Proceedings of the 2024 conference on empirical methods in natural language processing}, pp.\  1107--1128, 2024.

\bibitem[Dong et~al.(2026)Dong, Jiang, Zhu, and Ning]{dong2026understanding}
Yuxin Dong, Jiachen Jiang, Zhihui Zhu, and Xia Ning.
\newblock Understanding task vectors in in-context learning: Emergence, functionality, and limitations.
\newblock In \emph{International Conference on Learning Representations}, volume 2026, pp.\  19393--19428, 2026.

\bibitem[Goyal et~al.(2017)Goyal, Khot, Summers-Stay, Batra, and Parikh]{goyal2017vqav2}
Yash Goyal, Tejas Khot, Douglas Summers-Stay, Dhruv Batra, and Devi Parikh.
\newblock Making the v in vqa matter: Elevating the role of image understanding in visual question answering.
\newblock In \emph{Proceedings of the IEEE conference on computer vision and pattern recognition}, pp.\  6904--6913, 2017.

\bibitem[Gupta et~al.(2026)Gupta, Storkey, and Lapata]{gupta2026meta}
Akash Gupta, Amos Storkey, and Mirella Lapata.
\newblock Meta-adaptive prompt distillation for few-shot visual question answering.
\newblock In \emph{International Conference on Learning Representations}, volume 2026, pp.\  120800--120837, 2026.

\bibitem[Gurari et~al.(2018)Gurari, Li, Stangl, Guo, Lin, Grauman, Luo, and Bigham]{gurari2018vizwiz}
Danna Gurari, Qing Li, Abigale~J Stangl, Anhong Guo, Chi Lin, Kristen Grauman, Jiebo Luo, and Jeffrey~P Bigham.
\newblock Vizwiz grand challenge: Answering visual questions from blind people.
\newblock In \emph{Proceedings of the IEEE conference on computer vision and pattern recognition}, pp.\  3608--3617, 2018.

\bibitem[Hendel et~al.(2023)Hendel, Geva, and Globerson]{hendel-etal-2023-context}
Roee Hendel, Mor Geva, and Amir Globerson.
\newblock In-context learning creates task vectors.
\newblock In \emph{Findings of the Association for Computational Linguistics: EMNLP 2023}, pp.\  9318--9333, 2023.

\bibitem[Hojel et~al.(2024)Hojel, Bai, Darrell, Globerson, and Bar]{hojel2024finding}
Alberto Hojel, Yutong Bai, Trevor Darrell, Amir Globerson, and Amir Bar.
\newblock Finding visual task vectors.
\newblock In \emph{European Conference on Computer Vision}, pp.\  257--273. Springer, 2024.

\bibitem[Houlsby et~al.(2019)Houlsby, Giurgiu, Jastrzebski, Morrone, De~Laroussilhe, Gesmundo, Attariyan, and Gelly]{houlsby2019parameter}
Neil Houlsby, Andrei Giurgiu, Stanislaw Jastrzebski, Bruna Morrone, Quentin De~Laroussilhe, Andrea Gesmundo, Mona Attariyan, and Sylvain Gelly.
\newblock Parameter-efficient transfer learning for nlp.
\newblock In \emph{International conference on machine learning}, pp.\  2790--2799. PMLR, 2019.

\bibitem[Hu et~al.(2021)Hu, Shen, Wallis, Allen-Zhu, Li, Wang, Wang, and Chen]{hu2022lora}
Edward~J Hu, Yelong Shen, Phillip Wallis, Zeyuan Allen-Zhu, Yuanzhi Li, Shean Wang, Lu~Wang, and Weizhu Chen.
\newblock Lora: Low-rank adaptation of large language models.
\newblock \emph{arXiv preprint arXiv:2106.09685}, 2021.

\bibitem[Huang et~al.(2024)Huang, Mitra, Arbelle, Karlinsky, Darrell, and Herzig]{huang2024multimodal}
Brandon Huang, Chancharik Mitra, Assaf Arbelle, Leonid Karlinsky, Trevor Darrell, and Roei Herzig.
\newblock Multimodal task vectors enable many-shot multimodal in-context learning.
\newblock \emph{Advances in Neural Information Processing Systems}, 37:\penalty0 22124--22153, 2024.

\bibitem[Jiang et~al.(2024)Jiang, He, Zeng, Wei, Ku, Liu, and Chen]{jiang2024mantis}
Dongfu Jiang, Xuan He, Huaye Zeng, Cong Wei, Max Ku, Qian Liu, and Wenhu Chen.
\newblock Mantis: Interleaved multi-image instruction tuning.
\newblock \emph{arXiv preprint arXiv:2405.01483}, 2024.

\bibitem[Jiang et~al.(2025)Jiang, Fu, Hao, Hu, Peng, Geng, and Yang]{jiang2025mimic}
Yuchu Jiang, Jiale Fu, Chenduo Hao, Xinting Hu, Yingzhe Peng, Xin Geng, and Xu~Yang.
\newblock Mimic in-context learning for multimodal tasks.
\newblock In \emph{Proceedings of the Computer Vision and Pattern Recognition Conference}, pp.\  29825--29835, 2025.

\bibitem[Kang et~al.(2025)Kang, Lee, Park, Park, Kim, Kim, Lee, and Song]{kang2025adaptive}
Joonseong Kang, Soojeong Lee, Subeen Park, Sumin Park, Taero Kim, Jihee Kim, Ryunyi Lee, and Kyungwoo Song.
\newblock Adaptive task vectors for large language models.
\newblock \emph{arXiv preprint arXiv:2506.03426}, 2025.

\bibitem[Karpathy \& Fei-Fei(2015)Karpathy and Fei-Fei]{karpathy2015deep}
Andrej Karpathy and Li~Fei-Fei.
\newblock Deep visual-semantic alignments for generating image descriptions.
\newblock In \emph{Proceedings of the IEEE conference on computer vision and pattern recognition}, pp.\  3128--3137, 2015.

\bibitem[Lauren{\c{c}}on et~al.(2024{\natexlab{a}})Lauren{\c{c}}on, Marafioti, Sanh, and Tronchon]{laurenccon2024building}
Hugo Lauren{\c{c}}on, Andr{\'e}s Marafioti, Victor Sanh, and L{\'e}o Tronchon.
\newblock Building and better understanding vision-language models: insights and future directions.
\newblock \emph{arXiv preprint arXiv:2408.12637}, 2024{\natexlab{a}}.

\bibitem[Lauren{\c{c}}on et~al.(2024{\natexlab{b}})Lauren{\c{c}}on, Tronchon, Cord, and Sanh]{laurencon2024matters}
Hugo Lauren{\c{c}}on, L{\'e}o Tronchon, Matthieu Cord, and Victor Sanh.
\newblock What matters when building vision-language models?
\newblock \emph{Advances in Neural Information Processing Systems}, 37:\penalty0 87874--87907, 2024{\natexlab{b}}.

\bibitem[Lester et~al.(2021)Lester, Al-Rfou, and Constant]{lester-etal-2021-power}
Brian Lester, Rami Al-Rfou, and Noah Constant.
\newblock The power of scale for parameter-efficient prompt tuning.
\newblock In \emph{Proceedings of the 2021 conference on empirical methods in natural language processing}, pp.\  3045--3059, 2021.

\bibitem[Li et~al.(2024{\natexlab{a}})Li, Liu, Hu, Sun, Hu, and Zhang]{li2024state}
Dongfang Li, Zhenyu Liu, Xinshuo Hu, Zetian Sun, Baotian Hu, and Min Zhang.
\newblock In-context learning state vector with inner and momentum optimization.
\newblock \emph{Advances in Neural Information Processing Systems}, 37:\penalty0 7797--7820, 2024{\natexlab{a}}.

\bibitem[Li et~al.(2024{\natexlab{b}})Li, Zhang, Zhang, Zhang, Li, Li, Ma, and Li]{li2024llava}
Feng Li, Renrui Zhang, Hao Zhang, Yuanhan Zhang, Bo~Li, Wei Li, Zejun Ma, and Chunyuan Li.
\newblock Llava-next-interleave: Tackling multi-image, video, and 3d in large multimodal models.
\newblock \emph{arXiv preprint arXiv:2407.07895}, 2024{\natexlab{b}}.

\bibitem[Li et~al.(2023)Li, Patel, Vi{\'e}gas, Pfister, and Wattenberg]{li2023inference}
Kenneth Li, Oam Patel, Fernanda Vi{\'e}gas, Hanspeter Pfister, and Martin Wattenberg.
\newblock Inference-time intervention: Eliciting truthful answers from a language model.
\newblock \emph{Advances in Neural Information Processing Systems}, 36:\penalty0 41451--41530, 2023.

\bibitem[Li \& Liang(2021)Li and Liang]{li-liang-2021-prefix}
Xiang~Lisa Li and Percy Liang.
\newblock Prefix-tuning: Optimizing continuous prompts for generation.
\newblock In \emph{Proceedings of the 59th Annual Meeting of the Association for Computational Linguistics and the 11th International Joint Conference on Natural Language Processing (Volume 1: Long Papers)}, pp.\  4582--4597, 2021.

\bibitem[Li et~al.(2026)Li, Liu, Kang, Luo, Lou, Wu, and Xiong]{li2026hificl}
Xiaoyu Li, Yuhang Liu, Xuanshuo Kang, Zheng Luo, Fangqi Lou, Xiaohua Wu, and Zihan Xiong.
\newblock Hificl: High-fidelity in-context learning for multimodal tasks.
\newblock In \emph{Proceedings of the IEEE/CVF Conference on Computer Vision and Pattern Recognition}, pp.\  3069--3078, 2026.

\bibitem[Li et~al.(2025{\natexlab{a}})Li, Cao, He, Cheng, Fu, Xiao, Wang, and Tang]{li2025m2iv}
Yanshu Li, Yi~Cao, Hongyang He, Qisen Cheng, Xiang Fu, Xi~Xiao, Tianyang Wang, and Ruixiang Tang.
\newblock {M$^2$IV}: Towards efficient and fine-grained multimodal in-context learning via representation engineering.
\newblock \emph{arXiv preprint arXiv:2504.04633}, 2025{\natexlab{a}}.

\bibitem[Li et~al.(2025{\natexlab{b}})Li, Xu, Han, Gao, Wen, Liu, Wang, and Metaxas]{li2025implicit}
Zhuowei Li, Zihao Xu, Ligong Han, Yunhe Gao, Song Wen, Di~Liu, Hao Wang, and Dimitris Metaxas.
\newblock Implicit in-context learning.
\newblock In \emph{International Conference on Learning Representations}, volume 2025, pp.\  101644--101666, 2025{\natexlab{b}}.

\bibitem[Lin et~al.(2014)Lin, Maire, Belongie, Hays, Perona, Ramanan, Doll{\'a}r, and Zitnick]{lin2014microsoft}
Tsung-Yi Lin, Michael Maire, Serge Belongie, James Hays, Pietro Perona, Deva Ramanan, Piotr Doll{\'a}r, and C~Lawrence Zitnick.
\newblock Microsoft coco: Common objects in context.
\newblock In \emph{European conference on computer vision}, pp.\  740--755. Springer, 2014.

\bibitem[Liu et~al.(2022{\natexlab{a}})Liu, Tam, Muqeeth, Mohta, Huang, Bansal, and Raffel]{liu2022few}
Haokun Liu, Derek Tam, Mohammed Muqeeth, Jay Mohta, Tenghao Huang, Mohit Bansal, and Colin~A Raffel.
\newblock Few-shot parameter-efficient fine-tuning is better and cheaper than in-context learning.
\newblock \emph{Advances in Neural Information Processing Systems}, 35:\penalty0 1950--1965, 2022{\natexlab{a}}.

\bibitem[Liu et~al.(2023{\natexlab{a}})Liu, Li, Wu, and Lee]{liu2023visual}
Haotian Liu, Chunyuan Li, Qingyang Wu, and Yong~Jae Lee.
\newblock Visual instruction tuning.
\newblock \emph{Advances in neural information processing systems}, 36:\penalty0 34892--34916, 2023{\natexlab{a}}.

\bibitem[Liu et~al.(2024{\natexlab{a}})Liu, Li, Li, and Lee]{liu2024improved}
Haotian Liu, Chunyuan Li, Yuheng Li, and Yong~Jae Lee.
\newblock Improved baselines with visual instruction tuning.
\newblock In \emph{Proceedings of the IEEE/CVF conference on computer vision and pattern recognition}, pp.\  26296--26306, 2024{\natexlab{a}}.

\bibitem[Liu et~al.(2023{\natexlab{b}})Liu, Yuan, Fu, Jiang, Hayashi, and Neubig]{liu2023pre}
Pengfei Liu, Weizhe Yuan, Jinlan Fu, Zhengbao Jiang, Hiroaki Hayashi, and Graham Neubig.
\newblock Pre-train, prompt, and predict: A systematic survey of prompting methods in natural language processing.
\newblock \emph{ACM computing surveys}, 55\penalty0 (9):\penalty0 1--35, 2023{\natexlab{b}}.

\bibitem[Liu et~al.(2024{\natexlab{b}})Liu, Ye, Xing, and Zou]{liu2024incontext}
Sheng Liu, Haotian Ye, Lei Xing, and James~Y. Zou.
\newblock In-context vectors: Making in context learning more effective and controllable through latent space steering.
\newblock In Ruslan Salakhutdinov, Zico Kolter, Katherine Heller, Adrian Weller, Nuria Oliver, Jonathan Scarlett, and Felix Berkenkamp (eds.), \emph{Proceedings of the 41st International Conference on Machine Learning}, volume 235 of \emph{Proceedings of Machine Learning Research}, pp.\  32287--32307. PMLR, 21--27 Jul 2024{\natexlab{b}}.
\newblock URL \url{https://proceedings.mlr.press/v235/liu24bx.html}.

\bibitem[Liu et~al.(2022{\natexlab{b}})Liu, Ji, Fu, Tam, Du, Yang, and Tang]{liu-etal-2022-p}
Xiao Liu, Kaixuan Ji, Yicheng Fu, Weng Tam, Zhengxiao Du, Zhilin Yang, and Jie Tang.
\newblock P-tuning: Prompt tuning can be comparable to fine-tuning across scales and tasks.
\newblock In \emph{Proceedings of the 60th Annual Meeting of the Association for Computational Linguistics (Volume 2: Short Papers)}, pp.\  61--68, 2022{\natexlab{b}}.

\bibitem[Liu et~al.(2024{\natexlab{c}})Liu, Zheng, Du, Ding, Qian, Yang, and Tang]{liu2021gpt}
Xiao Liu, Yanan Zheng, Zhengxiao Du, Ming Ding, Yujie Qian, Zhilin Yang, and Jie Tang.
\newblock Gpt understands, too.
\newblock \emph{AI open}, 5:\penalty0 208--215, 2024{\natexlab{c}}.

\bibitem[Loshchilov \& Hutter(2017)Loshchilov and Hutter]{loshchilov2017decoupled}
Ilya Loshchilov and Frank Hutter.
\newblock Decoupled weight decay regularization.
\newblock \emph{arXiv preprint arXiv:1711.05101}, 2017.

\bibitem[Luo et~al.(2025)Luo, Darrell, and Bar]{luo2025vlm}
Grace Luo, Trevor Darrell, and Amir Bar.
\newblock Vision-language models create cross-modal task representations.
\newblock In Aarti Singh, Maryam Fazel, Daniel Hsu, Simon Lacoste-Julien, Felix Berkenkamp, Tegan Maharaj, Kiri Wagstaff, and Jerry Zhu (eds.), \emph{Proceedings of the 42nd International Conference on Machine Learning}, volume 267 of \emph{Proceedings of Machine Learning Research}, pp.\  41157--41176. PMLR, 13--19 Jul 2025.
\newblock URL \url{https://proceedings.mlr.press/v267/luo25c.html}.

\bibitem[Ma et~al.(2026)Ma, Gou, Hu, Wang, Zhuang, and Cai]{ma2026sensitivity}
Ziyu Ma, Chenhui Gou, Yiming Hu, Yong Wang, Bohan Zhuang, and Jianfei Cai.
\newblock Where and what matters: Sensitivity-aware task vectors for many-shot multimodal in-context learning.
\newblock In \emph{Proceedings of the AAAI Conference on Artificial Intelligence}, volume~40, pp.\  7892--7900, 2026.

\bibitem[Marino et~al.(2019)Marino, Rastegari, Farhadi, and Mottaghi]{marino2019okvqa}
Kenneth Marino, Mohammad Rastegari, Ali Farhadi, and Roozbeh Mottaghi.
\newblock Ok-vqa: A visual question answering benchmark requiring external knowledge.
\newblock In \emph{Proceedings of the IEEE/cvf conference on computer vision and pattern recognition}, pp.\  3195--3204, 2019.

\bibitem[Mu et~al.(2023)Mu, Li, and Goodman]{mu2023gist}
Jesse Mu, Xiang Li, and Noah Goodman.
\newblock Learning to compress prompts with gist tokens.
\newblock \emph{Advances in Neural Information Processing Systems}, 36:\penalty0 19327--19352, 2023.

\bibitem[Nilsback \& Zisserman(2008)Nilsback and Zisserman]{nilsback2008automated}
Maria-Elena Nilsback and Andrew Zisserman.
\newblock Automated flower classification over a large number of classes.
\newblock In \emph{2008 Sixth Indian conference on computer vision, graphics \& image processing}, pp.\  722--729. IEEE, 2008.

\bibitem[Oymak et~al.(2023)Oymak, Rawat, Soltanolkotabi, and Thrampoulidis]{oymak2023role}
Samet Oymak, Ankit~Singh Rawat, Mahdi Soltanolkotabi, and Christos Thrampoulidis.
\newblock On the role of attention in prompt-tuning.
\newblock In \emph{International Conference on Machine Learning}, pp.\  26724--26768. PMLR, 2023.

\bibitem[Park et~al.(2026)Park, Kim, Choi, and Rhee]{park2026soft}
Jungwon Park, Jimyeong Kim, Changin Choi, and Wonjong Rhee.
\newblock Soft head selection for injecting icl-derived task embeddings.
\newblock In \emph{Findings of the Association for Computational Linguistics: ACL 2026}, pp.\  27161--27214, 2026.

\bibitem[Peng et~al.(2024)Peng, Hao, Hu, Peng, Geng, and Yang]{peng2024live}
Yingzhe Peng, Chenduo Hao, Xinting Hu, Jiawei Peng, Xin Geng, and Xu~Yang.
\newblock Live: Learnable in-context vector for visual question answering.
\newblock \emph{Advances in Neural Information Processing Systems}, 37:\penalty0 9773--9800, 2024.

\bibitem[Petrov et~al.(2024)Petrov, Torr, and Bibi]{petrov2024prompting}
Aleksandar Petrov, Philip Torr, and Adel Bibi.
\newblock When do prompting and prefix-tuning work? a theory of capabilities and limitations.
\newblock In \emph{International Conference on Learning Representations}, volume 2024, pp.\  6031--6054, 2024.

\bibitem[Qin \& Eisner(2021)Qin and Eisner]{qin-eisner-2021-learning}
Guanghui Qin and Jason Eisner.
\newblock Learning how to ask: Querying lms with mixtures of soft prompts.
\newblock In \emph{Proceedings of the 2021 Conference of the North American Chapter of the Association for Computational Linguistics: Human Language Technologies}, pp.\  5203--5212, 2021.

\bibitem[Rimsky et~al.(2024)Rimsky, Gabrieli, Schulz, Tong, Hubinger, and Turner]{panickssery-etal-2024-steering}
Nina Rimsky, Nick Gabrieli, Julian Schulz, Meg Tong, Evan Hubinger, and Alexander Turner.
\newblock Steering llama 2 via contrastive activation addition.
\newblock In \emph{Proceedings of the 62nd Annual Meeting of the Association for Computational Linguistics (Volume 1: Long Papers)}, pp.\  15504--15522, 2024.

\bibitem[Rohrbach et~al.(2018)Rohrbach, Hendricks, Burns, Darrell, and Saenko]{rohrbach2018object}
Anna Rohrbach, Lisa~Anne Hendricks, Kaylee Burns, Trevor Darrell, and Kate Saenko.
\newblock Object hallucination in image captioning.
\newblock In \emph{Proceedings of the 2018 conference on empirical methods in natural language processing}, pp.\  4035--4045, 2018.

\bibitem[Saglam et~al.(2025)Saglam, Hu, Yang, Kalogerias, and Karbasi]{saglam2025learning}
Baturay Saglam, Xinyang Hu, Zhuoran Yang, Dionysis Kalogerias, and Amin Karbasi.
\newblock Learning task representations from in-context learning.
\newblock In \emph{Findings of the Association for Computational Linguistics: ACL 2025}, pp.\  6634--6663, 2025.

\bibitem[Sclar et~al.(2024)Sclar, Choi, Tsvetkov, and Suhr]{sclar2024quantifying}
Melanie Sclar, Yejin Choi, Yulia Tsvetkov, and Alane Suhr.
\newblock Quantifying language models' sensitivity to spurious features in prompt design or: How i learned to start worrying about prompt formatting.
\newblock In \emph{International Conference on Learning Representations}, volume 2024, pp.\  25055--25083, 2024.

\bibitem[Subramani et~al.(2022)Subramani, Suresh, and Peters]{subramani-etal-2022-extracting}
Nishant Subramani, Nivedita Suresh, and Matthew~E Peters.
\newblock Extracting latent steering vectors from pretrained language models.
\newblock In \emph{Findings of the Association for Computational Linguistics: ACL 2022}, pp.\  566--581, 2022.

\bibitem[Talemi et~al.(2026)Talemi, Kashiani, and Afghah]{talemi2026hyper}
Niloufar~Alipour Talemi, Hossein Kashiani, and Fatemeh Afghah.
\newblock Hyper-icl: Attention calibration with hyperbolic anchor distillation for multimodal in-context learning.
\newblock \emph{arXiv preprint arXiv:2606.04434}, 2026.

\bibitem[Todd et~al.(2024)Todd, Li, Sen~Sharma, Mueller, Wallace, and Bau]{todd2024function}
Eric Todd, Millicent Li, Arnab Sen~Sharma, Aaron Mueller, Byron Wallace, and David Bau.
\newblock Function vectors in large language models.
\newblock In \emph{International conference on learning representations}, volume 2024, pp.\  17282--17333, 2024.

\bibitem[Touvron et~al.(2023)Touvron, Lavril, Izacard, Martinet, Lachaux, Lacroix, Rozi{\`e}re, Goyal, Hambro, Azhar, et~al.]{touvron2023llama}
Hugo Touvron, Thibaut Lavril, Gautier Izacard, Xavier Martinet, Marie-Anne Lachaux, Timoth{\'e}e Lacroix, Baptiste Rozi{\`e}re, Naman Goyal, Eric Hambro, Faisal Azhar, et~al.
\newblock Llama: Open and efficient foundation language models.
\newblock \emph{arXiv preprint arXiv:2302.13971}, 2023.

\bibitem[Turner et~al.(2023)Turner, Thiergart, Leech, Udell, Vazquez, Mini, and MacDiarmid]{turner2023activation}
Alexander~Matt Turner, Lisa Thiergart, Gavin Leech, David Udell, Juan~J Vazquez, Ulisse Mini, and Monte MacDiarmid.
\newblock Steering language models with activation engineering.
\newblock \emph{arXiv preprint arXiv:2308.10248}, 2023.

\bibitem[Vedantam et~al.(2015)Vedantam, Lawrence~Zitnick, and Parikh]{vedantam2015cider}
Ramakrishna Vedantam, C~Lawrence~Zitnick, and Devi Parikh.
\newblock Cider: Consensus-based image description evaluation.
\newblock In \emph{Proceedings of the IEEE conference on computer vision and pattern recognition}, pp.\  4566--4575, 2015.

\bibitem[Venhoff et~al.(2025)Venhoff, Arcuschin, Torr, Conmy, and Nanda]{venhoff2025understanding}
Constantin Venhoff, Iv{\'a}n Arcuschin, Philip Torr, Arthur Conmy, and Neel Nanda.
\newblock Understanding reasoning in thinking language models via steering vectors.
\newblock \emph{arXiv preprint arXiv:2506.18167}, 2025.

\bibitem[Wah et~al.(2011)Wah, Branson, Welinder, Perona, and Belongie]{wah2011caltech}
Catherine Wah, Steve Branson, Peter Welinder, Pietro Perona, and Serge Belongie.
\newblock The caltech-ucsd birds-200-2011 dataset.
\newblock 2011.

\bibitem[Wang \& Komatsuzaki(2021)Wang and Komatsuzaki]{wang2021gptj}
Ben Wang and Aran Komatsuzaki.
\newblock Gpt-j-6b: A 6 billion parameter autoregressive language model, 2021.

\bibitem[Wang et~al.(2025)Wang, Gao, Gu, Pu, Cui, Wei, Liu, Jing, Ye, Shao, Wang, Chen, Zhang, Yang, Wang, Wei, Yin, Li, Cui, Chen, Ding, Tian, Wu, Xie, Li, Yang, Duan, Wang, Hou, Hao, Zhang, Li, Zhao, Duan, Deng, Fu, He, Wang, He, Shi, He, Xiong, Lv, Wu, Shao, Zhang, Deng, Qi, Ge, Guo, Zhang, Zhang, Cao, Lin, Tang, Gao, Huang, Gu, Lyu, Tang, Wang, Lv, Ouyang, Wang, Dou, Zhu, Lu, Lin, Dai, Su, Zhou, Chen, Qiao, Wang, and Luo]{wang2025internvl3}
Weiyun Wang, Zhangwei Gao, Lixin Gu, Hengjun Pu, Long Cui, Xingguang Wei, Zhaoyang Liu, Linglin Jing, Shenglong Ye, Jie Shao, Zhaokai Wang, Zhe Chen, Hongjie Zhang, Ganlin Yang, Haomin Wang, Qi~Wei, Jinhui Yin, Wenhao Li, Erfei Cui, Guanzhou Chen, Zichen Ding, Changyao Tian, Zhenyu Wu, Jingjing Xie, Zehao Li, Bowen Yang, Yuchen Duan, Xuehui Wang, Zhi Hou, Haoran Hao, Tianyi Zhang, Songze Li, Xiangyu Zhao, Haodong Duan, Nianchen Deng, Bin Fu, Yinan He, Yi~Wang, Conghui He, Botian Shi, Junjun He, Yingtong Xiong, Han Lv, Lijun Wu, Wenqi Shao, Kaipeng Zhang, Huipeng Deng, Biqing Qi, Jiaye Ge, Qipeng Guo, Wenwei Zhang, Songyang Zhang, Maosong Cao, Junyao Lin, Kexian Tang, Jianfei Gao, Haian Huang, Yuzhe Gu, Chengqi Lyu, Huanze Tang, Rui Wang, Haijun Lv, Wanli Ouyang, Limin Wang, Min Dou, Xizhou Zhu, Tong Lu, Dahua Lin, Jifeng Dai, Weijie Su, Bowen Zhou, Kai Chen, Yu~Qiao, Wenhai Wang, and Gen Luo.
\newblock Internvl3.5: Advancing open-source multimodal models in versatility, reasoning, and efficiency, 2025.
\newblock URL \url{https://arxiv.org/abs/2508.18265}.

\bibitem[Wang et~al.(2023)Wang, Chauhan, Wang, and Hsieh]{wang2023universality}
Yihan Wang, Jatin Chauhan, Wei Wang, and Cho-Jui Hsieh.
\newblock Universality and limitations of prompt tuning.
\newblock \emph{Advances in Neural Information Processing Systems}, 36:\penalty0 75623--75643, 2023.

\bibitem[Wu et~al.(2024)Wu, Arora, Wang, Geiger, Jurafsky, Manning, and Potts]{wu2024reft}
Zhengxuan Wu, Aryaman Arora, Zheng Wang, Atticus Geiger, Dan Jurafsky, Christopher~D Manning, and Christopher Potts.
\newblock Reft: Representation finetuning for language models.
\newblock \emph{Advances in Neural Information Processing Systems}, 37:\penalty0 63908--63962, 2024.

\bibitem[Yang et~al.(2026)Yang, Cho, Ding, and Inoue]{yang2026task}
Haolin Yang, Hakaze Cho, Kaize Ding, and Naoya Inoue.
\newblock Task vectors, learned not extracted: Performance gains and mechanistic insights.
\newblock In \emph{International Conference on Learning Representations}, volume 2026, pp.\  74846--74903, 2026.

\bibitem[Zhang et~al.(2023)Zhang, Han, Liu, Gao, Zhou, Hu, Yan, Lu, Li, and Qiao]{zhang2023llama}
Renrui Zhang, Jiaming Han, Chris Liu, Peng Gao, Aojun Zhou, Xiangfei Hu, Shilin Yan, Pan Lu, Hongsheng Li, and Yu~Qiao.
\newblock Llama-adapter: Efficient fine-tuning of language models with zero-init attention.
\newblock \emph{arXiv preprint arXiv:2303.16199}, 2023.

\bibitem[Zhang et~al.(2025)Zhang, Li, Liu, and Yue]{zhang2025learning}
Yiyuan Zhang, Handong Li, Jing Liu, and Xiangyu Yue.
\newblock Learning beyond still frames: Scaling vision-language models with video.
\newblock In \emph{Proceedings of the IEEE/CVF International Conference on Computer Vision}, pp.\  22425--22435, 2025.

\bibitem[Zhao et~al.(2024)Zhao, Cai, Si, Ma, An, Chen, Liu, Wang, Han, and Chang]{zhao2024mmicl}
Haozhe Zhao, Zefan Cai, Shuzheng Si, Xiaojian Ma, Kaikai An, Liang Chen, Zixuan Liu, Sheng Wang, Wenjuan Han, and Baobao Chang.
\newblock Mmicl: Empowering vision-language model with multi-modal in-context learning.
\newblock In \emph{International Conference on Learning Representations}, volume 2024, pp.\  14942--14980, 2024.

\bibitem[Zhao et~al.(2021)Zhao, Wallace, Feng, Klein, and Singh]{zhao2021calibrate}
Zihao Zhao, Eric Wallace, Shi Feng, Dan Klein, and Sameer Singh.
\newblock Calibrate before use: Improving few-shot performance of language models.
\newblock In \emph{International conference on machine learning}, pp.\  12697--12706. Pmlr, 2021.

\bibitem[Zhou et~al.(2024)Zhou, Li, Wang, and Shen]{zhou2024visual}
Yucheng Zhou, Xiang Li, Qianning Wang, and Jianbing Shen.
\newblock Visual in-context learning for large vision-language models.
\newblock In \emph{Findings of the Association for Computational Linguistics: ACL 2024}, pp.\  15890--15902, 2024.

\bibitem[Zou et~al.(2023)Zou, Phan, Chen, Campbell, Guo, Ren, Pan, Yin, Mazeika, Dombrowski, et~al.]{zou2023representation}
Andy Zou, Long Phan, Sarah Chen, James Campbell, Phillip Guo, Richard Ren, Alexander Pan, Xuwang Yin, Mantas Mazeika, Ann-Kathrin Dombrowski, et~al.
\newblock Representation engineering: A top-down approach to ai transparency.
\newblock \emph{arXiv preprint arXiv:2310.01405}, 2023.

\end{thebibliography}
\bibliographystyle{stave}
\clearpage
\appendix
\numberwithin{equation}{section}

\addtocontents{toc}{\protect\setcounter{tocdepth}{2}}
\renewcommand{\contentsname}{Appendix Contents}
\begingroup
\setlength{\parskip}{0pt}
\tableofcontents
\endgroup
\setcounter{tocdepth}{3}
\clearpage

\renewcommand{\topfraction}{0.9}
\renewcommand{\bottomfraction}{0.8}
\renewcommand{\textfraction}{0.07}
\renewcommand{\floatpagefraction}{0.85}
\setcounter{topnumber}{3}
\setcounter{bottomnumber}{2}
\setcounter{totalnumber}{4}

\section{Extended related work}
\label{app:related-work}

\textbf{Multimodal ICL and its cost.}
Visual instruction tuning connects image understanding with language generation \citep{liu2023visual,liu2024improved}, and interleaved image--text training together with multi-image instruction tuning supports richer multimodal contexts \citep{laurencon2024matters,jiang2024mantis,zhao2024mmicl}.
These models inherit demo-based adaptation from language models \citep{brown2020language}, but each demo image adds visual processing and many visual tokens to every query.
Visual In-Context Learning retrieves demos, summarizes their images, and composes textual demos \citep{zhou2024visual}, yet the demos remain in every prompt.
Demo-free methods remove them and store the task in a compact state.
They differ in where this state enters the frozen model: as extracted or trained decoder states, as added input tokens or attention memory, or, in STAVE, as updates to existing input embeddings.
STAVE keeps only the query image and text at inference, so it never encodes a demo image for a new query and runs with the prompt length and cache of zero-shot inference.

\textbf{Extracted visual and multimodal task vectors.}
Visual task vectors in MAE-VQGAN show that averaged attention-head activations can specify visual tasks without input--output examples \citep{hojel2024finding}, and cross-modal activation patching shows that LMMs encode task information shared across text and image inputs \citep{luo2025vlm}.
MTV averages demo-conditioned head activations and uses reinforcement learning to select the heads that receive them \citep{huang2024multimodal}.
STV selects context-sensitive heads from activation differences and then uses reinforcement learning to pick values from clustered activation banks \citep{ma2026sensitivity}.
In both, the values are restricted to demo activations, and every task needs a search over head locations.
STAVE learns its values directly with answer supervision on tokens fixed by prompt structure, which removes the head search and the activation bank and lets the readout and context vectors be optimized jointly for the same answer.
Because the update enters before the first decoder layer, one pair of vectors reaches every layer's attention, whereas a head-level vector acts only where it is inserted.
STAVE outperforms MTV and STV in every comparison, including STV at an equal parameter count.

\textbf{Distillation through decoder interventions.}
LIVE learns one vector and one scale per decoder layer by matching the output distribution of a demo-conditioned model, with an added answer loss \citep{peng2024live}.
M$^2$IV assigns separate vectors to the attention and feed-forward branches of every layer and combines output distillation, answer supervision, and a branch-correlation objective \citep{li2025m2iv}.
MimIC learns per-head shifts after attention with query-dependent scaling and aligns hidden states with the demo-conditioned model layer by layer \citep{jiang2025mimic}.
Hyper-ICL adds low-rank, query-modulated biases to attention logits and aligns intermediate features with a demo-conditioned teacher in hyperbolic space \citep{talemi2026hyper}.
These parameters sit in every decoder layer, so their storage grows with depth, their objectives match teacher outputs or intermediate features, and the student sees only demo-free prompts.
STAVE stores $2d$ scalars regardless of depth, 14 to 2,080 times fewer than LIVE and MimIC, adds no per-layer computation at inference, and scores higher on average.
Its two vectors are trained with answer CE on prompts with and without demos, with no teacher target, and this paired objective improves demo-free accuracy over training on demo-free prompts alone.

\textbf{Teacher-free virtual attention context.}
HiFICL places low-rank virtual KV pairs in every attention head, whose interaction with each query mixes learned context values into the original attention output, and trains them with the task loss alone, as STAVE does \citep{li2026hificl}.
STAVE instead changes the embeddings of existing readout and context tokens before the decoder.
It stores $2d$ scalars, over 200 times fewer than HiFICL, adds no attention-memory entries, and scores higher on average on every LMM we evaluate.
Its readout vector also changes the attention query of the token that predicts the answer, so it can re-weight attention over the query's own image and text tokens, whereas virtual keys and values can only mix a learned context into the original attention output.

\textbf{Extracted and adaptive task representations in language models.}
TV takes the prompted residual state at one layer \citep{hendel-etal-2023-context}, FV sums the outputs of task-relevant attention heads \citep{todd2024function}, and in-context vectors derive steering directions from latent input--output differences \citep{liu2024incontext}.
Later methods optimize injection coefficients or weighted combinations of head outputs \citep{li2025implicit,saglam2025learning}, generate a vector for each query \citep{kang2025adaptive}, or refine it with inner and momentum optimization \citep{li2024state}.
SITE learns soft head-selection weights over ICL-derived embeddings \citep{park2026soft}, and dynamic vector construction segments extracted representations and learns their injection locations \citep{cai2025beyond}.
Learned task vectors replace extraction with answer-supervised training of a vector at a chosen layer and token \citep{yang2026task}, and analyses of extracted task vectors identify limits on representing high-rank task mappings \citep{dong2026understanding}.
STAVE also learns its vectors from the answer, but places them at the input over tokens defined by prompt structure, so its tokens need no per-task search and one structural partition transfers across tasks and backbones.
Its normalization keeps the update energy from growing with the number of tokens a vector updates, which matters in LMMs, where one image expands into hundreds of tokens.
On the TV benchmark, STAVE exceeds TV and FV on every model and reaches the highest average accuracy, above SITE and 15-shot ICL, without any demo at inference.
The vector-sharing tradeoffs of this parameterization are analyzed in Appendix~\ref{app:tying}.

\textbf{Task adaptation at added input tokens.}
PT prepends learned vectors to the input embeddings of a frozen model and approaches full fine-tuning as model scale grows \citep{lester-etal-2021-power}.
P-tuning generates continuous prompts at template-defined locations with a small prompt encoder \citep{liu2021gpt}, deep prompt tuning and prefix tuning supply prompt vectors or KV pairs at every decoder layer \citep{liu-etal-2022-p,li-liang-2021-prefix}, and gist tokens compress an instruction into the activations of a few added tokens \citep{mu2023gist}.
In LMMs, MAPD meta-learns soft prompts together with an attention mapper over visual features and adapts them from a few labeled examples at test time \citep{gupta2026meta}.
Surveys of prompting collect further variants \citep{liu2023pre}.
All of these methods add tokens or attention-memory entries whose placement is a design choice, and language models are sensitive to prompt formatting \citep{sclar2024quantifying}.
Theory shows that added prefix tokens only bias the attention output of the original tokens without changing their relative attention \citep{petrov2024prompting}, characterizes the expressive limits of bounded-length prompts \citep{wang2023universality}, and traces how a tuned prompt acts through a single attention layer \citep{oymak2023role}.
STAVE also keeps the decoder frozen and stores the task at the input, but differs in three respects.
It writes onto existing tokens selected by a structural partition shared across tasks and backbones, so no insertion point is chosen and the prompt keeps its zero-shot length and cache.
Its readout vector changes the attention query of the token that predicts the answer, so it can change the relative attention over the original image and text tokens, which added keys cannot do at a fixed layer.
A zero update recovers the original computation, so training starts from the frozen model's own predictions, whereas an added token with a zero embedding still enters the attention normalizer, and LLaMA-Adapter obtains this property only through a separately learned zero-initialized gate \citep{zhang2023llama}.
A two-token soft prompt stores as many parameters as STAVE, so these differences come from where the parameters enter rather than from capacity.
STAVE outperforms 20-token PT on LMMs with 10 times fewer parameters and outperforms PT and prefix tuning on the TV benchmark.
Appendix~\ref{app:prompt-comparison} analyzes the attention-level consequence.

\textbf{Updates to existing representations.}
Steering vectors add a fixed direction to the residual stream at chosen layers \citep{subramani-etal-2022-extracting,turner2023activation,panickssery-etal-2024-steering,zou2023representation}, inference-time intervention shifts selected attention-head outputs \citep{li2023inference}, and ReFT learns low-rank interventions on the hidden states of a fixed set of prefix and suffix tokens in selected layers \citep{wu2024reft}.
Mixtures of soft prompts initialize each prompt vector from a template word and tune it in place \citep{qin-eisner-2021-learning}.
These methods choose a layer, a direction, or a set of tokens per task, and the steering approaches take their directions from activation contrasts rather than from the task answer.
STAVE also updates existing tokens but acts at the input embeddings, so one update reaches every layer's attention and no layer is chosen.
Moving the same update to the output of a decoder layer does not improve it in our depth ablation.
Its tokens follow prompt structure, including the expanded image embeddings and the final prompt token, and its two shared vectors are scaled by the number of tokens each updates and learned from the answer on paired prompts.

\textbf{Weight and module adaptation.}
Adapters and LoRA learn task-specific modules or low-rank weight updates \citep{houlsby2019parameter,hu2022lora}, and few-shot parameter-efficient fine-tuning has been proposed as an alternative to supplying demos repeatedly \citep{liu2022few}.
These methods also keep the pretrained weights frozen, but their task parameters are spread across decoder layers and grow with depth and rank.
STAVE stores two embedding vectors, shares every decoder module across tasks, switches tasks by replacing two vectors, and outperforms LoRA with orders of magnitude fewer parameters.

\clearpage
\section{Injection depth and inserted tokens}
\label{app:depth-theory}

The derivative statements concern the real-valued decoder map at fixed tokenization, attention masks, positional indices, and correct answer prefixes. Modality features are frozen, and any stochastic computation is fixed or disabled. These statements do not differentiate through token selection or numerical quantization.

\subsection{Proof of Proposition~\ref{prop:depth-reach}}
\label{app:depth-reach-proof}

Fix a nonempty prompt with $T$ attended tokens. Let $H_\ell\in\R^{T\times d}$ contain all its states after block $\ell$, with $H_0$ the input embeddings. A prompt-only intervention at depth $\ell$ changes these states before later blocks are evaluated. It does not alter answer-token states at that depth or retroactively modify keys and values already formed in blocks $1,\ldots,\ell$. After the last block, the output computation consists only of positionwise normalization, if present, and a positionwise vocabulary head.

\begin{proof}
For part (i), the prompt keys and values used by block $k$ are computed from $H_{k-1}$. An addition to $H_\ell$ therefore leaves all prompt keys and values in blocks $k\leq\ell$ unchanged and can affect later blocks only. The first answer token is predicted from the final prompt token. Every subsequent teacher-forced answer token is predicted from an answer-prefix token. At $\ell=N_{\mathrm{layer}}$, no later block can transmit a prompt-state change to these answer tokens. The positionwise head can therefore depend on the intervention only through the final prompt state, and only for the first prediction. Its derivative there may also vanish. Since the context vector excludes the final token, it has zero sequence-loss gradient at this depth.

For part (ii), take depths $0\leq\ell'<\ell\leq N_{\mathrm{layer}}$. Causality makes the Jacobian of each intervening block, restricted to all prompt states, block lower triangular in token order. Its diagonal $d\times d$ blocks are the own-token derivatives. By assumption, these blocks are invertible. Their product, the Jacobian $\mathsf T_{\ell,\ell'}$ from the row-stacked entries of $H_{\ell'}$ to those of $H_\ell$, is thus invertible. For any desired infinitesimal change $\Delta\in\R^{Td}$ at depth $\ell$, the unique prompt-state perturbation producing it to first order is $\mathsf T_{\ell,\ell'}^{-1}\Delta$. The first-answer logits factor through $H_\ell$ and the remaining prompt computation, so the two interventions have the same first-order effect on those logits.
\end{proof}

\textbf{Several depths.}
For simultaneous prompt-state shifts $\Delta_\ell$ at multiple depths, all evaluated about the same unmodified forward pass, let $\mathsf T_\ell=\mathsf T_{\ell,0}$ and $\mathsf T_0=I$. The input perturbation $\sum_\ell\mathsf T_\ell^{-1}\Delta_\ell$ reproduces their combined first-order first-answer effect by the chain rule and linearity, given nonsingularity along every involved path. Part (ii) concerns unrestricted perturbations over all attended prompt tokens and the first answer token. It characterizes what the input interface can reach, while the two-vector parameterization is analyzed in Appendix~\ref{app:geometry}.

\subsection{Attention ratios at a fixed layer}
\label{app:prompt-comparison}

Fix one attention head and an attention query $\xi\in\R^{d_h}$, where $d_h$ is the head dimension. For visible token $i$, let $k_i$ and $u_i$ be its key and value, and let its finite attention score be $s_i=\xi^\top k_i/\sqrt{d_h}+c_i$, with fixed positional term $c_i$. Partition the visible tokens into nonempty original and added sets $Q$ and $D$. Writing $Z_Q=\sum_{i\in Q}e^{s_i}$ and $Z_D=\sum_{i\in D}e^{s_i}$ gives
\begin{equation}
 o=(1-\lambda)o^Q+\lambda o^D,\qquad
 \lambda=\frac{Z_D}{Z_Q+Z_D},\qquad
 o^Q=\frac{\sum_{i\in Q}e^{s_i}u_i}{Z_Q},\quad
 o^D=\frac{\sum_{i\in D}e^{s_i}u_i}{Z_D}.
\label{eq:attention-decomposition}
\end{equation}
This is the single-head decomposition used to motivate shift-based adaptation \citep{jiang2025mimic}. All quantities here belong to the same forward computation. In a multilayer decoder, $o^Q$ need not equal the output on a separately rendered demo-free prompt, because the original states and positional scores can also change.

\textbf{Added keys and changed queries.}
If the original scores remain fixed when keys are added, every original weight $\alpha_i=e^{s_i}/Z_Q$ becomes $\alpha_i'=Z_Q\alpha_i/(Z_Q+Z_D)$. Hence $\alpha_i'/\alpha_j'=\alpha_i/\alpha_j$ for $i,j\in Q$ \citep{petrov2024prompting}. In contrast, holding keys and positional terms fixed while changing the query to $\xi+\epsilon b$ gives
\begin{equation}
 \left.\frac{d}{d\epsilon}\log\frac{\alpha_i(\epsilon)}{\alpha_j(\epsilon)}\right|_{\epsilon=0}
 =\frac{(k_i-k_j)^\top b}{\sqrt{d_h}},
\label{eq:query-ratio-derivative}
\end{equation}
where $b\in\R^{d_h}$ is the query-change direction. This follows by subtracting the two scores, since their common normalizer cancels. An update at an existing token can change queries, keys, and values without inserting tokens, and Equation~\ref{eq:query-ratio-derivative} isolates the query response that added keys lack at a fixed layer.

\section{Structural token groups and vector sharing}
\label{app:geometry}

Write $S\in\{0,1\}^{R\times q}$ for the assignment matrix, with $S_{kr}=\ind\{a(k)=r\}$ and exactly one nonzero entry per row. The transpose of learned row $r$ for task $\tau$ is denoted by $v_{\tau,r}\in\R^d$, and the task subscript is suppressed on $W$. The independent-group configuration learns $R$ vectors, while STAVE ties the groups into the readout and context vectors as in Equation~\ref{eq:stave-injection}.

\subsection{Partition conventions and modality handling}
\label{app:partition}

The partition uses the serialized prompt and template metadata, not the answer label, model predictions, or learned attention scores. Counts refer to decoder-input tokens after image expansion. Demo spans include their inputs, answers, and demo-specific formatting. They are excluded before any span is assigned, with the \grp{First} token as the only exception. BOS tokens, padding, supervised answer tokens, and generated answer tokens are also excluded.

\textbf{Endpoints and precedence.}
\grp{Last} is the final prompt token, the last attended token before the answer. \grp{First} is the first attended non-BOS token of the rendered prompt. When the template has an instruction prefix, as in VizWiz and VQAv2, this token lies in the shared prefix in both renderings and no demo token is updated. When the template has no prefix, it opens the query in the demo-free rendering and the first demo in the source rendering, and that single token receives the context vector while every other demo token is unchanged. On prompts that begin with an image, it is the opening delimiter of that image. The caption experiments instead use the first query-text token in both renderings. Tokens are assigned in the order \grp{Last}, \grp{First}, \grp{Answer-cue}, \grp{Image-gate}, \grp{Query-image}, \grp{Question-text}, and \grp{Prefix-text}, and the first matching label is kept. The same precedence is used for both renderings.

\textbf{Spans.}
\grp{Answer-cue} consists of the template tokens that introduce the answer, such as ``Answer:'', except the final prompt token, which forms \grp{Last}. \grp{Image-gate} holds explicit delimiters around the query image, not its expanded visual embeddings, which form \grp{Query-image}. \grp{Question-text} is the remaining text of the query, such as the question in VQA. \grp{Prefix-text} is the template's instruction span outside demos. Tokens not identified by the fixed metadata remain unchanged. A group with no eligible tokens has count zero and contributes no update.

\textbf{Modality handling.}
Text prompts use the frozen input-embedding layer. In LMMs, the frozen visual encoder and connector supply image embeddings at the hidden dimension of the decoder. STAVE operates on the merged decoder-input sequence, not on image pixels, visual encoder activations, or unexpanded placeholder tokens. Image groups are absent in text-only prompts, and \grp{Image-gate} can also be absent in a multimodal template. Both paired renderings retain the query and its modality. These conventions specify an unambiguous operator.

\subsection{Normalized geometry and gradients}
\label{app:geometry-proof}

For prompt $p$, let $\mathcal I(p)=\{k:n_k(p)>0\}$ and $\mathcal J(p)=\{r:N_r(p)>0\}$ denote the active groups and rows. Define $G_{tk}(p)=\ind\{s_t(p)=k\}/\sqrt{n_k(p)}$ for $k\in\mathcal I(p)$, with zero columns otherwise. Let $\Lambda(p)$ be diagonal, with
\begin{equation}
 \Lambda_{kk}(p)=
 \begin{cases}
 \sqrt{n_k(p)/N_{a(k)}(p)},&n_k(p)>0,\\
 0,&n_k(p)=0.
 \end{cases}
\label{eq:count-weights}
\end{equation}
An active group always has $N_{a(k)}(p)>0$, so every division is well-defined.

\begin{proposition}[Normalized geometry under vector sharing]
\label{prop:geometry}
The injection operator satisfies
\begin{align}
 H(p)&=G(p)\Lambda(p)S,\qquad
 G(p)^\top G(p)=\diag\big(\ind\{k\in\mathcal I(p)\}\big),\nonumber\\
 H(p)^\top H(p)&=\diag\big(\ind\{r\in\mathcal J(p)\}\big),\qquad
 \|H(p)W\|_F^2=\sum_{r\in\mathcal J(p)}\|v_{\tau,r}\|_2^2\leq\|W\|_F^2.
\label{eq:stave-isometry}
\end{align}
Moreover, $H(p)H(p)^\top$ is the orthogonal projector onto corrections constant on each active row's token set and zero elsewhere.
\end{proposition}

\begin{proof}
For an active group $k$ assigned to row $r$, the corresponding product of scales is $n_k^{-1/2}\sqrt{n_k/N_r}=N_r^{-1/2}$. Excluded tokens and inactive groups contribute zero. This proves $H=G\Lambda S$. Active columns of $G$ have disjoint supports and squared norm $n_k/n_k=1$, and the same argument gives squared norm $N_r/N_r=1$ for active columns of $H$. Taking the trace of $W^\top H^\top HW$ proves the norm identity. The orthonormal active columns of $H$ span exactly the stated correction space, so their outer products sum to its orthogonal projector.
\end{proof}

\textbf{Row gradients.}
For any differentiable loss $F$ on $\widetilde E=E+HW$, the chain rule gives
\begin{equation}
 \nabla_WF=H^\top\nabla_{\widetilde E}F,\qquad
 \nabla_{v_{\tau,r}}F=\frac{1}{\sqrt{N_r}}
       \sum_{t:a(s_t)=r}\nabla_{\widetilde E_t}F\quad(N_r>0).
\label{eq:group-gradient}
\end{equation}
The row gradient is zero when $N_r=0$, and $\|\nabla_WF\|_F\leq\|\nabla_{\widetilde E}F\|_F$. Normalization controls replication-induced energy, not gradient invariance: if all token gradients in a row equal $g$, its gradient is $\sqrt{N_r}g$.

\subsection{First-order cost of sharing}
\label{app:gradient-geometry}

Fix one prompt and a differentiable loss $F$, evaluated at zero update. Let $g_t=\nabla_{\widetilde E_t}F\in\R^d$, and fix $\rho>0$. For a nonempty eligible set $\mathcal S$, define $\Gamma(\mathcal S)$ as in Equation~\ref{eq:sharing-bound}.

\begin{proposition}[First-order cost of sharing]
\label{prop:sharing-cost}
\textbf{(i)} The largest first-order decrease from an update $w/\sqrt{|\mathcal S|}$ at every token in $\mathcal S$, with $\|w\|_2\leq\rho$, is $\Gamma(\mathcal S)$. The right side of Equation~\ref{eq:sharing-bound} is the optimum for free per-token updates of total Frobenius norm at most $\rho$. Equality holds exactly when all $g_t$ in $\mathcal S$ are equal.

\textbf{(ii)} For disjoint nonempty $\mathcal A,\mathcal B$, two separately normalized vectors with combined squared norm at most $\rho^2$ attain
\[
 \Gamma_{\mathrm{split}}=\sqrt{\Gamma(\mathcal A)^2+\Gamma(\mathcal B)^2}
 \geq\max\{\Gamma(\mathcal A\cup\mathcal B),\Gamma(\mathcal A),\Gamma(\mathcal B)\}.
\]
If $g_t=0$ on $\mathcal B$, then $\Gamma(\mathcal A\cup\mathcal B)=\sqrt{|\mathcal A|/|\mathcal A\cup\mathcal B|}\,\Gamma(\mathcal A)$.

\textbf{(iii)} Let $\cos\theta$ be the cosine between $\sum_{t\in\mathcal A}g_t$ and $\sum_{t\in\mathcal B}g_t$, set to one if either sum is zero. Then
\[
 \Gamma_{\mathrm{split}}^2-\Gamma(\mathcal A\cup\mathcal B)^2
 =\frac{\big(\sqrt{|\mathcal B|}\,\Gamma(\mathcal A)-\sqrt{|\mathcal A|}\,\Gamma(\mathcal B)\big)^2+2\sqrt{|\mathcal A||\mathcal B|}\,\Gamma(\mathcal A)\Gamma(\mathcal B)(1-\cos\theta)}{|\mathcal A|+|\mathcal B|}.
\]
The first term, a norm gap, vanishes exactly when the mean gradients of $\mathcal A$ and $\mathcal B$ have equal norm. The second term, a direction gap, grows as $\cos\theta$ falls for fixed $\Gamma(\mathcal A)$ and $\Gamma(\mathcal B)$.
\end{proposition}

\begin{proof}
The shared update has directional derivative $\langle w,\sum_{t\in\mathcal S}g_t\rangle/\sqrt{|\mathcal S|}$. Minimizing this linear functional on the radius-$\rho$ ball yields $-\Gamma(\mathcal S)$. Applying the same argument to the stacked free token updates gives the right side of Equation~\ref{eq:sharing-bound}. With $\overline g=|\mathcal S|^{-1}\sum_{t\in\mathcal S}g_t$, the identity
\[
 \sum_{t\in\mathcal S}\|g_t\|_2^2-\frac{1}{|\mathcal S|}\left\|\sum_{t\in\mathcal S}g_t\right\|_2^2
 =\sum_{t\in\mathcal S}\|g_t-\overline g\|_2^2
\]
proves the inequality and its equality condition.

For two vectors, the stacked gradient has norm $\sqrt{\Gamma(\mathcal A)^2+\Gamma(\mathcal B)^2}/\rho$, giving the stated optimum. A one-vector update $w$ on $\mathcal S=\mathcal A\cup\mathcal B$ is realized by $w_{\mathcal A}=\sqrt{|\mathcal A|/|\mathcal S|}\,w$ and $w_{\mathcal B}=\sqrt{|\mathcal B|/|\mathcal S|}\,w$, with the same total norm. Thus splitting cannot reduce the optimum for this fixed prompt. Setting either vector to zero yields the other two lower bounds. The zero-gradient statement follows because $\mathcal B$ changes only the normalization factor, not the gradient sum.

For (iii), write $G_{\mathcal A}$ and $G_{\mathcal B}$ for the two gradient sums, so that $\|G_{\mathcal A}\|_2=\sqrt{|\mathcal A|}\,\Gamma(\mathcal A)/\rho$, $\|G_{\mathcal B}\|_2=\sqrt{|\mathcal B|}\,\Gamma(\mathcal B)/\rho$, and $\langle G_{\mathcal A},G_{\mathcal B}\rangle=\|G_{\mathcal A}\|_2\|G_{\mathcal B}\|_2\cos\theta$. Expanding $\Gamma(\mathcal A\cup\mathcal B)^2=\rho^2\|G_{\mathcal A}+G_{\mathcal B}\|_2^2/(|\mathcal A|+|\mathcal B|)$ gives
\[
 (|\mathcal A|+|\mathcal B|)\,\Gamma(\mathcal A\cup\mathcal B)^2=|\mathcal A|\,\Gamma(\mathcal A)^2+|\mathcal B|\,\Gamma(\mathcal B)^2+2\sqrt{|\mathcal A||\mathcal B|}\,\Gamma(\mathcal A)\Gamma(\mathcal B)\cos\theta.
\]
Subtracting this from $(|\mathcal A|+|\mathcal B|)\,\Gamma_{\mathrm{split}}^2$ and completing the square proves the identity. The first term is zero exactly when $\|G_{\mathcal A}\|_2/|\mathcal A|=\|G_{\mathcal B}\|_2/|\mathcal B|$.
\end{proof}

The dilution statement is exact under its zero-gradient hypothesis for a fixed prompt. Appendix~\ref{app:allocation} measures it across image resolutions.

\subsection{Sharing cost before training}
\label{app:motivation-chain}

Figure~\ref{fig:motivation-chain} measures the quantities of Proposition~\ref{prop:sharing-cost} in the frozen Idefics2 backbone of Table~\ref{tab:vqa+finegrained} on Flowers. No vector is trained and no optimizer is involved. The measurement uses 512 demo-free samples from the training pool of that table, 256 to fit the update directions and 256 others to measure the loss. For each sample, $g_t$ is the zero-update gradient of the answer CE of Equation~\ref{eq:stave-ce} with respect to input embedding $t$. A prompt holds one token in each of \grp{Last}, \grp{Answer-cue}, \grp{First}, and \grp{Image-gate}, 64 \grp{Query-image} tokens, and 37.5 \grp{Question-text} tokens on average, and no \grp{Prefix-text}. The readout set $\mathcal A$ is \grp{Last} and \grp{Answer-cue}, and the context set $\mathcal B$ holds the other eligible tokens.

\textbf{Update and gradient shares.}
One vector shared by all eligible tokens adds $w/\sqrt{|\mathcal S|}$ to each of them, so the share of $\|w\|_2^2$ that a group receives equals its share of the tokens. For the gradient, let $\hat g_k$ average $n_k^{-1/2}\sum_{t:s_t=k}g_t$ over all 512 samples. By Equation~\ref{eq:group-gradient}, $\hat g_k$ is the mean gradient of an independent vector for group $k$, and $\rho\|\hat g_k\|_2$ is its value of $\Gamma$. Panel (a) reports each group's share of $\sum_k\|\hat g_k\|_2^2$. The readout set receives 1.9\% of the shared update and holds 76\% of the gradient. \grp{Query-image} receives 61\% of the update and holds 0.002\% of the gradient, so it comes close to the zero-gradient case of Proposition~\ref{prop:sharing-cost}(ii): it enlarges the normalization count without adding to the gradient sum. On the held-out samples, the first-order loss decrease per unit update norm is 12.5 along the direction of one shared vector and 27.1 along that of the readout vector alone.

\textbf{Gradient per token.}
Panel (b) reports, for each token set $\mathcal S$, the norm of $|\mathcal S|^{-1}\sum_{t\in\mathcal S}g_t$ after this vector is averaged over the 512 samples. The readout set reaches 19.2, 3.5 times \grp{First}, the largest context group at 5.5, and \grp{Query-image} reaches 0.02. The readout value is below that of \grp{Answer-cue} (30.2) and of \grp{Last} (22.4) alone, because the mean gradients of these two tokens are nearly orthogonal, so their average is shorter than either of them.

\textbf{Adding context.}
Panel (c) uses the unit descent directions $u_{\mathcal A}$ and $u_{\mathcal B}$ of the readout and context vectors, each the negative mean zero-update gradient of its vector over the fitting samples, scaled to unit norm. On the held-out samples, the readout vector is set to $a\,u_{\mathcal A}$ and the context vector to $c\,u_{\mathcal B}$, and the panel reports the mean answer CE in nats. The zero update gives 3.05. With $c=0$, the loss falls to 1.12 at $a=0.2$ and 1.07 at $a=0.3$. Each context norm $c\in\{0.1,0.2,0.3\}$ lowers the whole curve, and $c=0.3$ reaches 0.55 at $a=0.2$. The two directions have cosine similarity 0.03, so one vector on $\mathcal A\cup\mathcal B$ cannot point along both. All directions come from zero-update gradients and all norms stay at most 0.3, so the figure describes the neighborhood of the zero update.

\subsection{Proof of Theorem~\ref{thm:tying}}
\label{app:tying}

\textbf{Tied coordinates.}
Fix the common group counts used in the theorem, and let $w\in\R^{qd}$ stack the transposed rows of $W$. The normalized independent-group coordinates of its tied update are
\begin{equation}
 v=Q_qw,\qquad Q_q=(\Lambda S)\otimes I_d,\qquad
 \mathcal U_q=\col(Q_q),\qquad P_q=Q_qQ_q^\top.
\label{eq:tying-projector}
\end{equation}
Here $\otimes$ is the Kronecker product and $I_d$ is the identity on $\R^d$. The active columns of $\Lambda S$ are orthonormal, by summing $n_k/N_r$ over groups assigned to row $r$, and inactive columns are zero. Hence $P_q$ is an orthogonal projector, even with empty rows. For each $v\in\mathcal U_q$, $w=Q_q^\top v$ realizes $v$ with inactive coordinates zero and $\|w\|_2=\|v\|_2$. Inactive full-group coordinates have no effect on the margins and can also be set to zero.

\textbf{Interpreting $\delta_q$.}
To interpret $\delta_q$, let $g_{jk}=\nabla_{v_k}h_j(0)/\sqrt{n_k}$ for active group $k$, where $v_k\in\R^d$ is its independent vector. This is the mean gradient with respect to that group's input embeddings. For active row $r$, put $\overline g_{jr}=\sum_{k:a(k)=r,\,n_k>0}(n_k/N_r)g_{jk}$. Projection onto $\mathcal U_q$ replaces each gradient block by $\sqrt{n_k}\,\overline g_{jr}$, so
\begin{equation}
 \delta_q^2=\max_j\sum_{r:N_r>0}\sum_{k:a(k)=r,\,n_k>0}
          n_k\|g_{jk}-\overline g_{jr}\|_2^2.
\label{eq:contrast-variance}
\end{equation}
Thus readout/context sharing discards only within-row group contrasts. This identity does not assume that their gradients agree.

\begin{proof}[Proof of Theorem~\ref{thm:tying}]
The tied feasible set is contained in the unrestricted group-coordinate ball, so $m_q(\rho)\leq m_R(\rho)$. Continuity and compactness give an unrestricted maximizer $v^\star$. Its projection $\overline v=P_qv^\star$ is tied-feasible and has norm at most $\rho$.

Gradient Lipschitzness along every segment from zero to a point in the ball gives
\begin{equation}
 h_j(v)=h_j(0)+\nabla h_j(0)^\top v+r_j(v),\qquad
 |r_j(v)|\leq\frac{\kappa}{2}\|v\|_2^2.
\label{eq:margin-taylor}
\end{equation}
Indeed, the remainder is the integral of $\langle\nabla h_j(tv)-\nabla h_j(0),v\rangle$ over $t\in[0,1]$. Since $v^\star-\overline v$ is orthogonal to $\mathcal U_q$,
\begin{align*}
 h_j(\overline v)
 &\geq h_j(v^\star)
 -\|(I-P_q)\nabla h_j(0)\|_2\|v^\star-\overline v\|_2
 -\frac{\kappa}{2}\big(\|v^\star\|_2^2+\|\overline v\|_2^2\big)\\
 &\geq h_j(v^\star)-\rho\delta_q-\kappa\rho^2.
\end{align*}
Taking the minimum over the finite margin collection and using feasibility of $\overline v$ proves the upper bound. Affine margins admit $\kappa=0$, which gives equality of attainable margins when $\delta_q=0$.
\end{proof}

\textbf{From margins to cross-entropy.}
For any desired margin $\gamma\in\R$, an unrestricted update attaining every margin at least $\gamma+\rho\delta_q+\kappa\rho^2$ implies the existence of a tied update attaining every margin at least $\gamma$ on these contexts. If the collection includes every competing token for each supervised prefix in a vocabulary of size $C\geq2$, the tied logits $z\in\R^C$ at each prefix satisfy $z_a-z_c\geq\gamma$ for every $c\ne a$, where $a$ is the correct token. Thus
\begin{equation}
 -\log\softmax(z)_a
 =\log\left(1+\sum_{c\ne a}e^{z_c-z_a}\right)
 \leq c_\gamma:=\log\big(1+(C-1)e^{-\gamma}\big).
\label{eq:margin-ce}
\end{equation}
Averaging preserves this CE bound.

\textbf{Varying counts.}
The theorem requires a single $Q_q$ for all its contexts. Equal group counts suffice, and so do common normalized within-row proportions. With varying proportions, the shared row $w$ induces different full-group vectors $Q_q(p)w$. For example, with two groups, one scalar row, and count profiles $(1,1)$ and $(1,4)$, these vectors are $(w/\sqrt{2},w/\sqrt{2})$ and $(w/\sqrt{5},2w/\sqrt{5})$. No single full-group vector realizes both corrections for nonzero $w$. Conversely, the full-group vector $(1,0)$ is not tied on the first profile. Thus the two shared embedding-correction families need not be nested across varying-count prompts, and the theorem applies to each count profile separately.

\section{Paired supervision and demo-free risk}
\label{app:paired}

Equation~\ref{eq:stave-objective} is a sum of answer-supervised source and target CE losses. It contains no hidden-state matching or soft-label objective. Throughout the whole-sample results below, $n$ denotes the number of training query--answer pairs, not the optimizer's minibatch size. The source renderings are fixed for these comparisons. The same deterministic arguments apply to a source loss averaged over a prescribed demo-sampling rule whenever that averaged loss satisfies the stated continuity or differentiability assumptions.

\subsection{Proof of Proposition~\ref{prop:paired-descent}}
\label{app:target-descent}

\begin{proof}
Let $w^+=w-\eta g$. Integrating the target gradient along the segment gives
\begin{align*}
 \mathcal L_{\mathrm{tgt}}(w^+)-\mathcal L_{\mathrm{tgt}}(w)
 &=-\eta\int_0^1\langle\nabla\mathcal L_{\mathrm{tgt}}(w-t\eta g),g\rangle\,dt\\
 &\leq-\eta\langle g_{\mathrm{tgt}},g\rangle
              +\frac{\beta\eta^2}{2}\|g\|_2^2,
\end{align*}
which proves Equation~\ref{eq:target-descent}.
\end{proof}

\textbf{Finite-step conditions.}
Write $a=\langle g_{\mathrm{tgt}},g\rangle=\|g_{\mathrm{tgt}}\|_2^2+\langle g_{\mathrm{tgt}},g_{\mathrm{src}}\rangle$. If $a>0$, then $g\ne0$. For $\beta>0$, a step satisfying $0<\eta<2a/(\beta\|g\|_2^2)$ and the segment smoothness assumption strictly lowers target loss. For $\beta=0$, every positive step satisfying that assumption does so.

\textbf{Comparison with a target-only step.}
For comparison with a target-only step, additionally assume the same gradient Lipschitz bound on the segment from $w$ to $w-\eta g_{\mathrm{tgt}}$. Applying the integral remainder bound in the other direction yields
\[
 \mathcal L_{\mathrm{tgt}}(w-\eta g_{\mathrm{tgt}})-\mathcal L_{\mathrm{tgt}}(w)
 \geq-\eta\|g_{\mathrm{tgt}}\|_2^2-\frac{\beta\eta^2}{2}\|g_{\mathrm{tgt}}\|_2^2.
\]
Subtracting this from the paired-step upper bound gives
\begin{equation}
 \begin{split}
 \mathcal L_{\mathrm{tgt}}(w-\eta g)-\mathcal L_{\mathrm{tgt}}(w-\eta g_{\mathrm{tgt}})
 \leq{}&-\eta\langle g_{\mathrm{tgt}},g_{\mathrm{src}}\rangle\\
 &+\frac{\beta\eta^2}{2}\big(\|g\|_2^2+\|g_{\mathrm{tgt}}\|_2^2\big).
 \end{split}
\label{eq:paired-versus-target-step}
\end{equation}
If $c=\langle g_{\mathrm{tgt}},g_{\mathrm{src}}\rangle>0$ and $\beta>0$, the paired step is strictly better at the same step size whenever $0<\eta<2c/[\beta(\|g\|_2^2+\|g_{\mathrm{tgt}}\|_2^2)]$ and both segment assumptions hold. For $\beta=0$, positivity of $c$ suffices on those segments. These comparisons concern plain gradient steps from the same parameter value.

\subsection{Compatibility of the empirical objectives}

Fix $\rho>0$ and the comparison class $\mathcal W_{q,\rho}=\{W\in\R^{q\times d}:\|W\|_F\leq\rho\}$. Assume both whole-sample branch losses are continuous on it. Define
\begin{equation}
 \mathcal L_b^\star=\min_{W\in\mathcal W_{q,\rho}}\mathcal L_b(W),\qquad
 \mathcal C_{\mathrm{src}\mid\mathrm{tgt}}
 =\min_{W\in\argmin_{U\in\mathcal W_{q,\rho}}\mathcal L_{\mathrm{tgt}}(U)}
          \mathcal L_{\mathrm{src}}(W)-\mathcal L_{\mathrm{src}}^\star.
\label{eq:compatibility-gap}
\end{equation}
The nonnegative gap $\mathcal C_{\mathrm{src}\mid\mathrm{tgt}}$ is the source-loss sacrifice required to remain target-optimal. The ball and the target minimizer set are nonempty compact sets, so all displayed minima exist.

\begin{proposition}[Target cost of incompatible supervision]
\label{prop:compatibility-bound}
If $\widehat W\in\mathcal W_{q,\rho}$ satisfies
$\mathcal L(\widehat W)\leq\min_{W\in\mathcal W_{q,\rho}}\mathcal L(W)+\varepsilon_{\mathrm{opt}}$
for $\varepsilon_{\mathrm{opt}}\geq0$, then
\begin{equation}
 0\leq\mathcal L_{\mathrm{tgt}}(\widehat W)-\mathcal L_{\mathrm{tgt}}^\star
 \leq\mathcal C_{\mathrm{src}\mid\mathrm{tgt}}+\varepsilon_{\mathrm{opt}}.
\label{eq:target-compatibility-bound}
\end{equation}
\end{proposition}

\begin{proof}
Choose a target minimizer $W^\dagger$ attaining the source minimum in Equation~\ref{eq:compatibility-gap}. Approximate paired optimality implies
\[
 \mathcal L_{\mathrm{tgt}}(\widehat W)+\mathcal L_{\mathrm{src}}(\widehat W)
 \leq\mathcal L_{\mathrm{tgt}}^\star+\mathcal L_{\mathrm{src}}(W^\dagger)+\varepsilon_{\mathrm{opt}}.
\]
Use $\mathcal L_{\mathrm{src}}(\widehat W)\geq\mathcal L_{\mathrm{src}}^\star$ and subtract $\mathcal L_{\mathrm{tgt}}^\star$ to obtain the upper bound. The lower bound follows from the definition of the target minimum.
\end{proof}

\textbf{Shared minimizers.}
The symmetric source bound follows by exchanging branches. If the two losses share a minimizer, both compatibility gaps vanish and an approximate paired minimizer is within $\varepsilon_{\mathrm{opt}}$ of each branch optimum on the stated ball.

\textbf{A counterexample.}
Shared answer labels do not ensure compatible responses. For a scalar binary-logit example, margins $h_{\mathrm{src}}(w)=w$ and $h_{\mathrm{tgt}}(w)=-w$ give losses $\log(1+e^{-w})$ and $\log(1+e^w)$. Their gradients cancel at zero, and no $w$ gives both margins a positive value. Replacing the source margin by $2w$ makes a paired gradient step at zero increase target loss. The compatibility gap $\mathcal C_{\mathrm{src}\mid\mathrm{tgt}}$ is the quantity that separates this case from compatible supervision.

\subsection{Demo-free risk guarantees}
\label{app:generalization}

Let $n\geq1$ and let $(x_i,y_i)_{i=1}^n$ be independent, identically distributed query--answer pairs from the deployment distribution of one task. The frozen model, target rendering, structural partition, row assignment, and radius $\rho>0$ are fixed independently of this sample. Assume the losses are measurable in the example, and define
\begin{equation}
 \mathcal R_{\mathrm{tgt}}(W)=\E_{(x,y)}\ell(y\mid\operatorname{render}(\varnothing,x),W).
\label{eq:target-risk}
\end{equation}
The empirical mean is $\mathcal L_{\mathrm{tgt}}$ in Equation~\ref{eq:stave-objective}. Source demos may be reused or chosen using the training sample: the concentration argument uses only the $n$ independent target examples. Paired renderings, answer tokens, and repeated minibatches do not increase this sample count.

\begin{assumption}[Bounded and Lipschitz target losses]
\label{ass:bounded-class}
There are constants $0<M_\ell<\infty$ and $0\leq G_\ell<\infty$ such that, for every target example in the distribution's support and all $W,W'\in\mathcal W_{q,\rho}$, its loss lies in $[0,M_\ell]$ and
\begin{equation}
 |\ell(y\mid p,W)-\ell(y\mid p,W')|
 \leq G_\ell\|H(p)(W-W')\|_F,\qquad p=\operatorname{render}(\varnothing,x).
\label{eq:embedding-lipschitz}
\end{equation}
\end{assumption}

These conditions are uniform over examples and the comparison class. Uniform bounds on logits and input derivatives over the ball supply such constants. The radius is an analysis device, not an extra penalty, clipping rule, or projection in STAVE.

\begin{theorem}[Uniform target-risk deviation]
\label{thm:uniform-target}
Under Assumption~\ref{ass:bounded-class}, for fixed $\epsilon\in(0,\rho]$ and $\delta\in(0,1)$, with probability at least $1-\delta$,
\begin{equation}
 \sup_{W\in\mathcal W_{q,\rho}}|\mathcal R_{\mathrm{tgt}}(W)-\mathcal L_{\mathrm{tgt}}(W)|
 \leq\mathfrak B_n(\epsilon,\delta)
 :=2G_\ell\epsilon+M_\ell\sqrt{\frac{qd\log(1+2\rho/\epsilon)+\log(2/\delta)}{2n}}.
\label{eq:uniform-target-bound}
\end{equation}
Here $\epsilon$ is the covering radius, $\delta$ the failure probability, and $\mathfrak B_n$ the deviation bound.
\end{theorem}

\begin{proof}
Proposition~\ref{prop:geometry} makes every target loss $G_\ell$-Lipschitz in $\|W-W'\|_F$. The comparison class is a Euclidean ball of dimension $qd$. A maximal set separated by more than $\epsilon$ is an $\epsilon$-net. Its disjoint open balls of radius $\epsilon/2$ lie inside a ball of radius $\rho+\epsilon/2$, so a volume comparison bounds the net size by $(1+2\rho/\epsilon)^{qd}$.

For completeness, if $X\in[0,M_\ell]$, the log moment-generating function $\psi(s)=\log\E e^{s(X-\E X)}$ has $\psi(0)=\psi'(0)=0$ and $\psi''(s)\leq M_\ell^2/4$. The second derivative is a variance under an exponentially tilted distribution, still supported on $[0,M_\ell]$, and any such variance is at most its second moment about $M_\ell/2$, hence at most $M_\ell^2/4$. Integrating twice gives $\psi(s)\leq s^2M_\ell^2/8$. Independence and the exponential Markov inequality, optimized over $s$ for each tail, therefore give at every fixed net point
\[
 \Pr\big(|\mathcal L_{\mathrm{tgt}}-\mathcal R_{\mathrm{tgt}}|>t\big)
 \leq2\exp(-2nt^2/M_\ell^2).
\]
A union bound over the net yields the square-root term in Equation~\ref{eq:uniform-target-bound}. Every point of the ball is within $\epsilon$ of a net point. The empirical and population losses each change by at most $G_\ell\epsilon$, proving the result. Continuity in $W$ permits the supremum to be taken over a countable dense subset, so the event is measurable.
\end{proof}

\begin{corollary}[Conditional target risk of paired training]
\label{thm:deployment-risk}
On the event of Theorem~\ref{thm:uniform-target}, suppose the empirical source loss is continuous on $\mathcal W_{q,\rho}$ and $\widehat W$ satisfies the radius and paired optimization condition of Proposition~\ref{prop:compatibility-bound}. Then
\begin{equation}
 \mathcal R_{\mathrm{tgt}}(\widehat W)-\inf_{W\in\mathcal W_{q,\rho}}\mathcal R_{\mathrm{tgt}}(W)
 \leq\mathcal C_{\mathrm{src}\mid\mathrm{tgt}}+\varepsilon_{\mathrm{opt}}+2\mathfrak B_n(\epsilon,\delta).
\label{eq:compatibility-risk}
\end{equation}
\end{corollary}

\begin{proof}
Uniform deviation and Proposition~\ref{prop:compatibility-bound} give
\[
 \mathcal R_{\mathrm{tgt}}(\widehat W)
 \leq\mathcal L_{\mathrm{tgt}}^\star+\mathcal C_{\mathrm{src}\mid\mathrm{tgt}}+\varepsilon_{\mathrm{opt}}+\mathfrak B_n.
\]
The same event implies $\mathcal L_{\mathrm{tgt}}^\star\leq\inf_W\mathcal R_{\mathrm{tgt}}(W)+\mathfrak B_n$. Combining the two inequalities proves the claim.
\end{proof}

\textbf{Capacity term.}
The uniform event covers any sample-dependent choice of $\widehat W$ in the fixed ball, including one fitted using source prompts. The capacity term uses $2d$ coordinates for STAVE and $7d$ for independent group vectors, a comparison made at common radius and loss constants. No common-count assumption is needed for this statistical result. When the class is chosen from $M_{\mathrm{cfg}}$ predetermined configurations, a union bound replaces $\log(2/\delta)$ by $\log(2M_{\mathrm{cfg}}/\delta)$.

\section{Inference cost}
\label{app:inference}

\textbf{Prompt tokens and storage.}
Let $T_{\mathrm{src}}$ and $T_{\mathrm{tgt}}$ be decoder-input lengths with and without demos for the same query, including expanded visual tokens. STAVE inference uses exactly $T_{\mathrm{tgt}}$ prompt tokens. Counts and assignments can be constructed in $O(T_{\mathrm{tgt}})$ time from fixed structural metadata. Adding the scaled vectors costs $O(T_{\mathrm{tgt}}d)$ operations and can be performed in place. Storage is $2d$ task-specific scalars, independent of query count and demo length. No independent full-group matrix needs to be stored for inference.

\textbf{KV cache.}
For a decoder with $N_{\mathrm{layer}}$ layers, $h_{\mathrm{kv}}$ KV heads per layer, and head dimension $d_h$, a conventional uncompressed prompt cache stores
\begin{equation}
 \operatorname{KVElements}(T)=2N_{\mathrm{layer}}T h_{\mathrm{kv}}d_h.
\label{eq:kv-accounting}
\end{equation}
Here $T$ is the number of cached tokens and the factor two counts keys and values. For equal precision, batch size, and output length, STAVE and zero-shot inference have the same cache shape. The values in that cache differ. Relative to keeping demos, the saved prompt-cache elements are
$2N_{\mathrm{layer}}(T_{\mathrm{src}}-T_{\mathrm{tgt}})h_{\mathrm{kv}}d_h$
when the source prompt is longer. Both caches grow identically per subsequently generated token, although the number of generated tokens can differ.

\textbf{Prefill computation.}
Under conventional dense attention, the prefill attention term is $O(N_{\mathrm{layer}}T^2d)$ and the positionwise decoder terms are linear in $T$ for a fixed architecture. STAVE preserves the target prompt's decoder dimensions and adds only the embedding-update pass. A soft prompt of length $P$ instead changes the attention-size term from $T^2$ to $(T+P)^2$. A KV prefix adds attention-memory entries even when it computes no new attention queries.

\textbf{Visual preprocessing and latency.}
The visual encoder and connector process the query image as in zero-shot inference. Demo removal avoids the demo prompt tokens and, when demos would otherwise be encoded at inference, their visual preprocessing. Measured latency also depends on kernels, precision, batching, and stopping behavior, so Appendix~\ref{app:cost-measured} runs every method under one setting.

\textbf{Training cost.}
Both paired branches run the frozen model, and gradients propagate through the decoder to the input vectors. Backbone optimizer states and weight updates are absent, while activation storage and backward computation remain, and the source branch includes demos. This one-time training cost is separate from the per-query inference cost above.

\subsection{Measured inference cost}
\label{app:cost-measured}

Table~\ref{tab:cost-h200} reports measured per-query cost for the Idefics2-8B VQAv2 entries of Table~\ref{tab:vqav2-coco}. Every method loads the checkpoint that produces the score in that table and runs on the same H200 inside one job, with batch size 1, fp16 weights, three beams, and the same 100 validation samples after warmup. TTFT is one prefill call with a single new token and is the like-for-like measure of prompt cost. Because the methods emit answers of different lengths, the Fixed-20 column forces exactly 20 generated tokens for every method. TTFT and the decode column time the generation call alone, with decode taken from the 20-token generation time, while Fixed-20 is end to end and also includes the processor pass over the prompt text and images. State is the stored task state. STAVE only adds its vectors to the input embeddings before the first decoder block, so its decoder computation during prefill and decoding is that of zero-shot inference. The zero-shot and STAVE rows therefore share one value in the TTFT, decode, and Fixed-20 columns, the mean of their two measurements, which differ by less than 1\%.

STAVE keeps the prompt length, prefill FLOPs, and KV cache of zero-shot inference, at 128 tokens. MimIC, HiFICL, and LIVE run the same prompt through an attention wrapper, which multiplies TTFT by 1.03 to 1.10 relative to STAVE. LoRA keeps the prompt but adds low-rank branches in the decoder, the visual encoder, and the connector, which multiplies TTFT by 1.18 and the per-token decode cost by 1.41. ICL with 32 demos needs 60 times the TTFT, 31 times the prefill FLOPs, and 52 GiB more peak memory than STAVE.

STAVE has the lowest stored state of any trained method here and matches the TTFT, per-token decode cost, and fixed-length latency of zero-shot inference, below every other row, and Table~\ref{tab:vqav2-coco} gives it the highest accuracy on this model and dataset, so nothing in this comparison trades accuracy for speed against it.

Figures~\ref{fig:cost-shots-a} and~\ref{fig:cost-shots-b} report six cost dimensions against the demo count on Idefics2-8B VizWiz and OK-VQA and Qwen-VL-7B OK-VQA, measured on one A100 40GB with batch size 1, fp16 weights, greedy decoding, at most 10 new tokens, and a shared stop rule, as medians over 500 test queries. They use the full group set. Prompt tokens, prefill FLOPs, KV cache, TTFT, and peak activation memory depend only on the prompt and are identical for the two-vector configuration. Answer latency also depends on the generated answer length and includes the CPU pass that builds the group assignment (4 to 14 ms). Qwen-VL with 32 demos exceeds the 40 GB memory and is omitted. On all three model and dataset pairs, STAVE keeps the prompt tokens, prefill FLOPs, KV cache, and TTFT of zero-shot inference, while ICL cost grows with every added demo. At 16 demos, ICL multiplies the TTFT of STAVE by 8 to 13 and its peak activation memory by 14 to 17.

\begin{figure}[!htbp]
\centering
\includegraphics[width=0.85\linewidth]{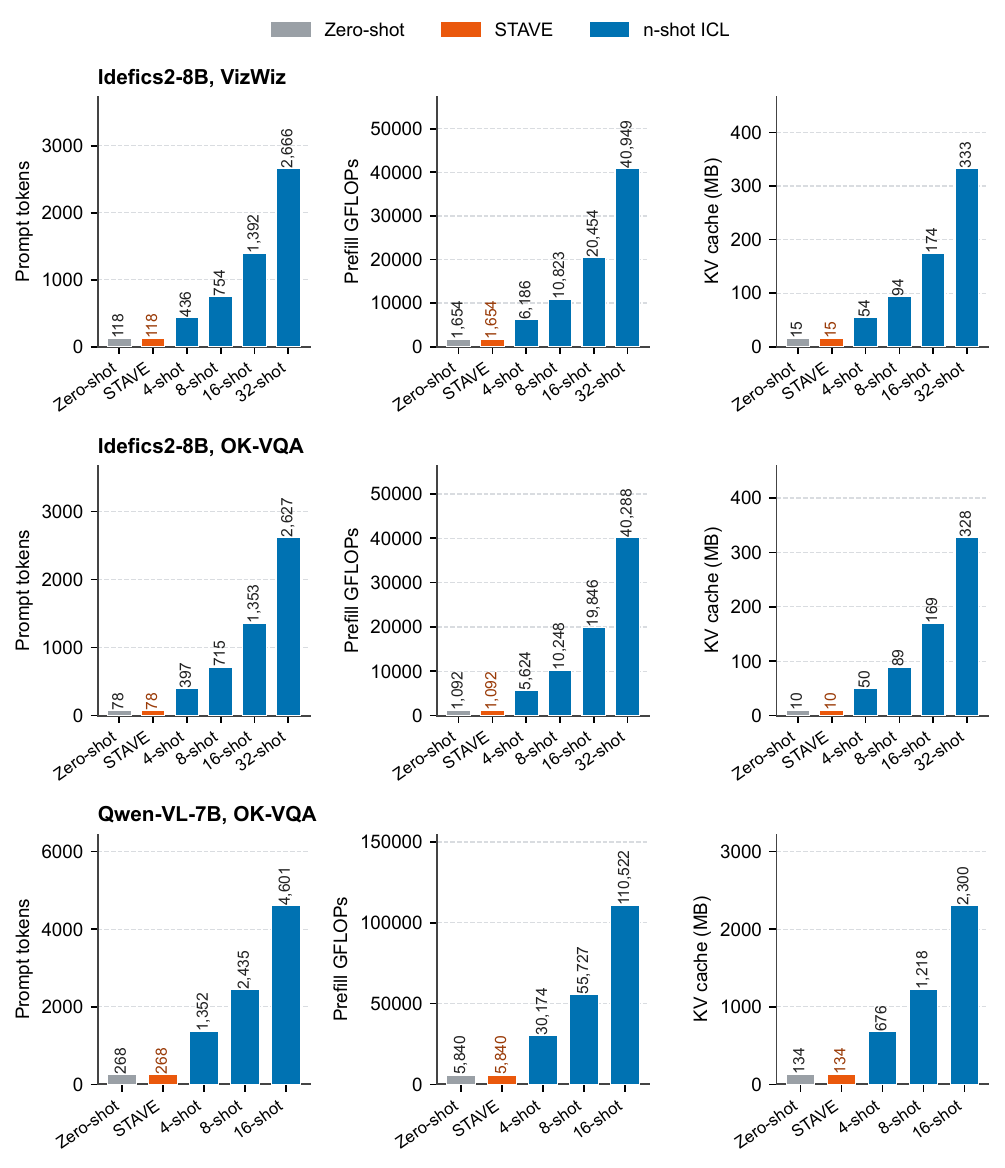}\par
\vspace{-3.2pt}
\caption{Prompt tokens, prefill FLOPs, and KV cache against the demo count.}
\label{fig:cost-shots-a}
\end{figure}

\begin{figure}[!htbp]
\centering
\includegraphics[width=0.85\linewidth]{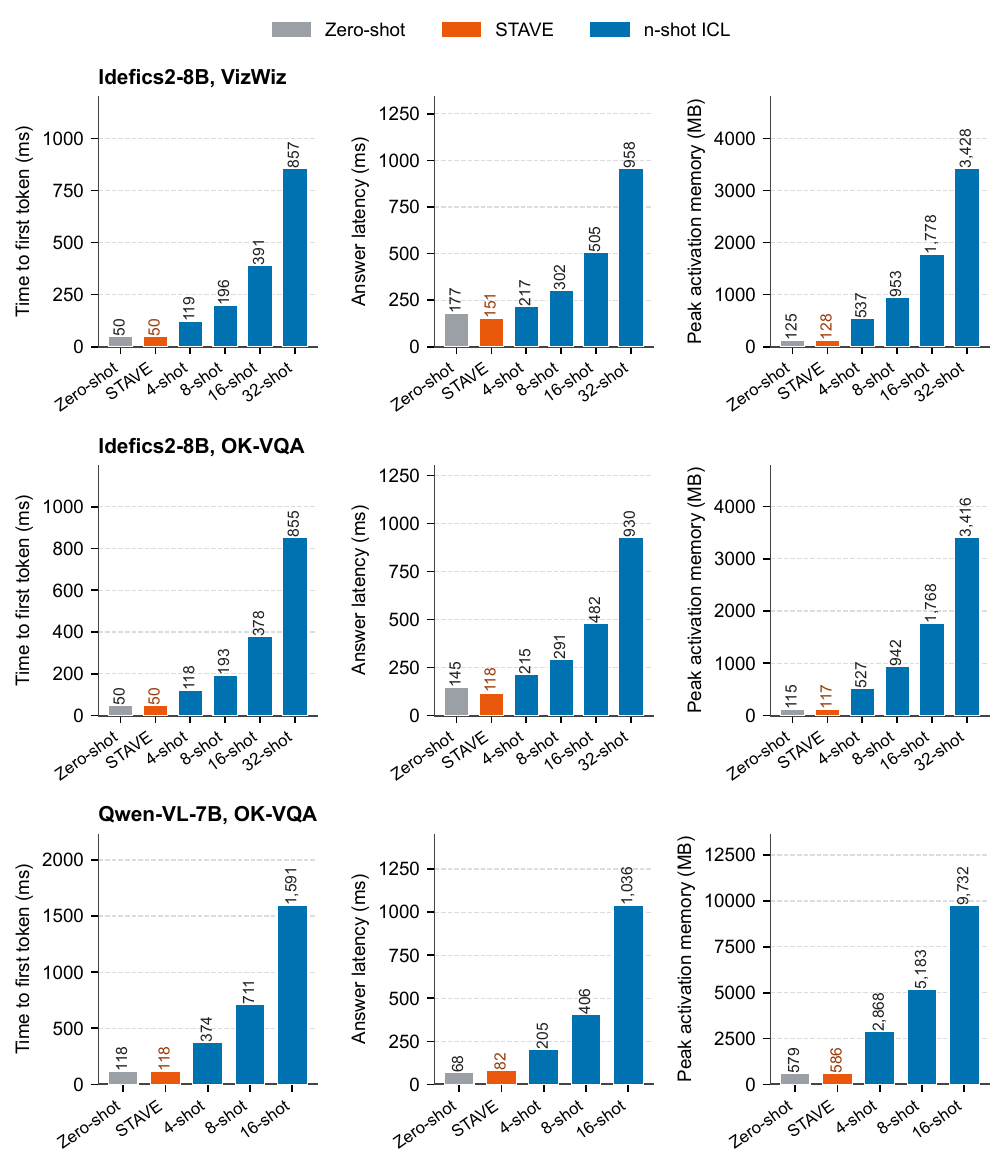}\par
\vspace{-3.2pt}
\caption{TTFT, answer latency, and peak activation memory against the demo count.}
\label{fig:cost-shots-b}
\end{figure}

\begin{table}[!htbp]
\centering
\caption{Measured per-query inference cost on Idefics2-8B VQAv2.}
\label{tab:cost-h200}
\footnotesize
\setlength{\tabcolsep}{4.5pt}
\renewcommand{\arraystretch}{1.15}
\resizebox{\linewidth}{!}{%
\begin{tabular}{l|rr|rrr|r}
\toprule
\rowcolor{tableheader}
\textbf{Method} & State (KiB) & Prompt tok. & TTFT (ms) & Decode (ms/tok) & Fixed-20 (ms) & Peak mem. (GiB) \\
\midrule
Zero-shot & 0 & 128 & \textbf{98} & \textbf{23.9} & \textbf{560} & 16.5 \\
8-shot ICL & 0 & 1{,}026 & 1{,}430 & 24.0 & 2{,}021 & 29.8 \\
16-shot ICL & 0 & 1{,}922 & 2{,}854 & 24.4 & 3{,}605 & 43.0 \\
32-shot ICL & 0 & 3{,}717 & 5{,}877 & 26.1 & 6{,}933 & 68.7 \\
\midrule
LoRA ($r{=}16$) & 68{,}704 & 128 & 116 & 33.6 & 763 & 16.5 \\
LIVE & 512 & 128 & 101 & 27.4 & 629 & 16.5 \\
MimIC & 1{,}028 & 128 & 106 & 31.6 & 713 & 16.5 \\
HiFICL & 8{,}704 & 128 & 108 & 34.3 & 767 & 16.5 \\
\midrule
\rowcolor{stavehighlight}
\stavelabel & \textbf{32} & 128 & \textbf{98} & \textbf{23.9} & \textbf{560} & 16.5 \\
\bottomrule
\end{tabular}%
}
\end{table}

\subsection{Stored task state}
\label{app:cost-storage}

Table~\ref{tab:storage} lists the task-specific state each method stores per task and backbone at hidden dimension $d=4096$. STAVE stores two vectors. The full group set stores one vector per structural token group. ICV stores one vector per layer, added after the MLP of each block. MTV stores one vector per attention head across layers, and LIVE stores one vector and one scale per layer. I2CL stores one attention-output vector and one MLP-output vector per layer, with two scales for each. MimIC and HiFICL store trained attention modules, and LoRA stores low-rank weight updates. Parameter counts are read back from the loaded module of each method and match the main tables. Sizes assume fp32 storage. The Idefics2 MimIC entry is the original encoder, which adds one shift vector per layer rather than learning a projection of the query states. The largest state here, LoRA on LLaVA, is 2,399 times that of STAVE.

\begin{table}[!htbp]
\centering
\caption{Stored task state per task and backbone ($d=4096$, fp32).}
\label{tab:storage}
\footnotesize
\setlength{\tabcolsep}{6pt}
\renewcommand{\arraystretch}{1.15}
\resizebox{\linewidth}{!}{%
\begin{tabular}{l|l|rr}
\toprule
\rowcolor{tableheader}
\textbf{Method} & \textbf{Stored state} & \textbf{Parameters} & \textbf{Size} \\
\midrule
Zero-shot, ICL & none & 0 & 0 \\
Single vector (TV, FV) & one embedding vector & 4{,}096 & 16 KiB \\
ICV & one vector per layer & 131{,}072 & 512 KiB \\
MTV & one vector per head & 131{,}072 & 512 KiB \\
LIVE & one vector and one scale per layer & 131{,}104 & 512 KiB \\
I2CL & two vectors and four scales per layer & 262{,}272 & 1.0 MiB \\
MimIC (Idefics2 / LLaVA) & attention modules & 263{,}168 / 17{,}040{,}384 & 1.0 MiB / 65.0 MiB \\
HiFICL & attention modules & 2{,}228{,}224 & 8.5 MiB \\
LoRA (Idefics2 / LLaVA) & low-rank weight updates & 17{,}588{,}224 / 19{,}652{,}608 & 67.1 MiB / 75.0 MiB \\
\midrule
\rowcolor{stavehighlight}
\stavelabel & two embedding vectors & 8{,}192 & 32 KiB \\
\bottomrule
\end{tabular}%
}
\end{table}

\subsection{Performance overview settings}
\label{app:overview-figure}

Figure~\ref{fig:performance-overview} uses reported results from four tables. From Table~\ref{tab:vqav2-coco}, it shows LLaVA and Idefics2-8B on VQAv2, OK-VQA, and COCO. From Table~\ref{tab:vqa+finegrained}, it shows Idefics2-8B on VizWiz, DTD, and Flowers and Qwen-VL-7B on VizWiz and OK-VQA. From Table~\ref{tab:text}, it shows Pythia 2.8B. TV benchmark accuracy aggregates its 18 tasks, weighted by test queries. Bar labels show the reported accuracy percentages or COCO CIDEr scores at the precision printed in each table. Each performance bar's length is the ratio of its score to STAVE's score for that model and dataset, with the inner edge representing zero.

The params sector uses the Idefics2-8B parameter counts of Table~\ref{tab:vqav2-coco}. Its bar length is $1-\log_{10}(P_m/P_{\mathrm{STAVE}})/4$ of the full length, where $P_m$ is method $m$'s parameter count, so STAVE has the full length and each factor of 10 removes a quarter of it. The latency sector shows the fixed-20 latencies printed in Table~\ref{tab:cost-h200}. Its bar length is $t_{\mathrm{STAVE}}/t_m$ of the full length, where $t_m$ is method $m$'s latency, so a longer bar is faster. These measurements are for Idefics2-8B VQAv2 under the settings of that table. ICL uses 16 demos in the Table~\ref{tab:vqav2-coco} and latency sectors, four in Table~\ref{tab:vqa+finegrained}, and 15 on the TV benchmark. LoRA uses rank 8 on the TV benchmark and rank 16 elsewhere. Every value in the overview is copied from the cited table.

\section{Implementation details}
\label{app:implementation}

\subsection{Backbones}
\label{app:backbones}

Table~\ref{tab:backbones} lists the eleven frozen backbones used in the experiments. The hidden dimension $d$ fixes the size of the learned matrix, which holds $2d$ parameters on every backbone, and $N_{\mathrm{layer}}$ is the number of decoder blocks. Image tokens is the number of decoder-input tokens one query image occupies after the frozen visual encoder and connector expand it, under the image settings used here. The readout vector covers at most a few tokens on every backbone, so these counts are the scale the context vector is normalized against.

{\Urlmuskip=0mu plus 0mu\relax
Idefics2 and Idefics3 run without image splitting, and Qwen2.5-VL images are resized to 448 pixels, so each of the six LMMs has a fixed count. Idefics2 appears under two checkpoints, \path{HuggingFaceM4/idefics2-8b-base} in Table~\ref{tab:vqav2-coco} and \path{HuggingFaceM4/idefics2-8b} in Table~\ref{tab:vqa+finegrained}, written Idefics2-8B-base and Idefics2-8B-instruct where both appear, and the model written as LLaVA in Table~\ref{tab:vqav2-coco} is \path{llava-hf/llava-interleave-qwen-7b-hf}. The other LMM checkpoints are \path{Qwen/Qwen-VL-Chat}, \path{OpenGVLab/InternVL3_5-8B}, \path{Qwen/Qwen2.5-VL-7B-Instruct}, and \path{HuggingFaceM4/Idefics3-8B-Llama3}. The LLM checkpoints are the EleutherAI \path{pythia-2.8b}, \path{pythia-6.9b}, and \path{pythia-12b} releases together with \path{EleutherAI/gpt-j-6B} and \path{huggyllama/llama-7b}. Panel (a) names the language model each backbone is built from, and several released checkpoints are instruction-tuned: the Idefics3 decoder starts from Llama-3.1-8B-Instruct, the LLaVA-Interleave decoder starts from Qwen1.5-7B-Chat, and Qwen-VL and Qwen2.5-VL are used in their chat and instruct releases. Training and evaluation run on NVIDIA A100 (40GB) and H200 GPUs.\par}

\begin{table}[!htbp]
\centering
\caption{Frozen backbones.}
\label{tab:backbones}
\footnotesize
\setlength{\tabcolsep}{4pt}

\textbf{(a) LMMs}\par\vspace{2pt}
\begin{tabular}{lllccr}
\toprule
\rowcolor{tableheader}
\textbf{Backbone} & \textbf{Language model} & \textbf{Vision encoder and connector} & $d$ & $N_{\mathrm{layer}}$ & \textbf{Img.} \\
\midrule
Idefics2-8B         & Mistral-7B   & SigLIP-SO400M, perceiver resampler   & 4096 & 32 & 64 \\
LLaVA-Interleave-7B & Qwen1.5-7B   & SigLIP-SO400M, MLP projector         & 4096 & 32 & 729 \\
Qwen-VL-7B          & Qwen-7B      & OpenCLIP ViT-bigG, cross-attention   & 4096 & 32 & 256 \\
InternVL3.5-8B      & Qwen3-8B     & InternViT-300M, pixel shuffle, MLP   & 4096 & 36 & 256 \\
Qwen2.5-VL-7B       & Qwen2.5-7B   & Native-resolution ViT, patch merger  & 3584 & 28 & 256 \\
Idefics3-8B         & Llama-3.1-8B & SigLIP-SO400M, pixel shuffle, MLP    & 4096 & 32 & 169 \\
\bottomrule
\end{tabular}

\vspace{8pt}
\textbf{(b) LLMs}\par\vspace{2pt}
\begin{tabular}{lccl}
\toprule
\rowcolor{tableheader}
\textbf{Backbone} & $d$ & $N_{\mathrm{layer}}$ & \textbf{Family} \\
\midrule
Pythia 2.8B & 2560 & 32 & GPT-NeoX \\
Pythia 6.9B & 4096 & 32 & GPT-NeoX \\
Pythia 12B  & 5120 & 36 & GPT-NeoX \\
GPT-J 6B    & 4096 & 28 & GPT-J \\
LLaMA 7B    & 4096 & 32 & LLaMA \\
\bottomrule
\end{tabular}
\end{table}

\subsection{VQAv2, OK-VQA, and COCO}
\label{app:impl-table1}

Table~\ref{tab:vqav2-coco} follows the protocol of MimIC and HiFICL, and all methods, STAVE included, share the same data, prompts, and decoding. The training set is the first 1,000 samples of the released shuffle at seed 3407, which is also the demo pool. The trained methods are evaluated without demos on the full evaluation set, with 10,000 VQAv2 questions, the 5,046 validation questions of the official OK-VQA v1.1 release, and 5,000 COCO Karpathy validation images. ICL draws the demos of each query from the same pool without replacement, with seeds 41, 42, and 3407. Zero-shot and 8-shot ICL are scored on the full evaluation set, and 16-shot and 32-shot ICL on the first 1,000 samples of the shuffled evaluation order. Zero-shot has no demos and deterministic decoding, so the three seeds give the same score. The prompts are the released templates, an instruction-style prefix for the VQA tasks and a caption-style prefix for COCO (Appendix~\ref{app:dataset-examples}). Decoding uses three beams on both backbones, at most 20 new tokens, length penalty 0, and batch size 1. Post-processing and scorers are the released ones, and CIDEr uses pycocoevalcap with CoreNLP 3.4.1. Published CIDEr values reported as fractions are multiplied by 100.

MimIC and HiFICL report the best-performing epoch, and we follow this protocol. Tables~\ref{tab:vqa+finegrained} and~\ref{tab:vl-modern} instead select the checkpoint on validation samples (Appendix~\ref{app:impl-table2}). STAVE trains its $2d=8{,}192$ parameters for five epochs with the source and target CE losses weighted 1 to 1, using AdamW with weight decay $10^{-3}$ and a cosine schedule with 10\% warmup. The learning rate is $10^{-3}$ or $3\times10^{-3}$ set per model and dataset, each training step uses $K=4$ paired examples ($K=16$ on LLaVA COCO), and the source prompt holds one demo on OK-VQA and eight on COCO.

All compared methods run with the authors' released code. LIVE~\citep{peng2024live}, MimIC~\citep{jiang2025mimic}, and HiFICL~\citep{li2026hificl} follow the settings of their papers. PT follows the setting of \citet{lester-etal-2021-power} with vocabulary initialization. The three PT modes learn 20 virtual tokens each on the prompt without demos and train for five epochs with AdamW. We tune the PT learning rate between $10^{-3}$ and $3\times10^{-2}$. PT is sensitive to the learning rate and hard to tune, and a learning rate that trains stably at one random seed can collapse at another seed on the same model and dataset.

LoRA appears in Tables~\ref{tab:vqav2-coco} and~\ref{tab:text} with different ranks and settings. In Table~\ref{tab:vqav2-coco} it follows the LoRA configuration released with HiFICL. It applies rank 16 to the attention query, key, value, and output projections of the language model and, on Idefics2, of the perceiver connector, and to the query, key, and value projections of the visual encoder. The MLP connector of LLaVA is not adapted. It trains with AdamW at learning rate $5\times10^{-4}$, weight decay $10^{-3}$, and a cosine schedule with 10\% warmup. It uses batch size 2 with 8 accumulation steps, the setting of the released LoRA script. It follows the five-epoch schedule. The rank 8 setting of Table~\ref{tab:text} is given in Appendix~\ref{app:impl-text}.

\subsection{VizWiz, OK-VQA, and fine-grained classification}
\label{app:impl-table2}

Tables~\ref{tab:vqa+finegrained} and~\ref{tab:vl-modern} follow the protocol of MTV~\citep{huang2024multimodal} and STV~\citep{ma2026sensitivity}, and every method uses the STAVE split. The pool is a subsample of the training split drawn with seed 41, the validation block follows the fitting block, and the test set is the full held-out split. MimIC draws its teacher demos only from its training records. On VizWiz and OK-VQA the validation questions are those of STAVE.

Table~\ref{tab:vqa+finegrained} pairs a weaker and a stronger frozen checkpoint. Qwen-VL-7B, used in its chat release, is the stronger one on the two VQA datasets, with zero-shot accuracy of 45.49\% on VizWiz and 57.51\% on OK-VQA against 37.74\% and 49.44\% for Idefics2-8B. Four demos raise its average from 70.35\% to 73.53\%, whereas they lower the Idefics2-8B average from 70.57\% to 69.72\%. The pair therefore tests whether the learned vectors help both when the frozen model leaves much room for improvement and when it already answers many questions correctly. STAVE has the highest VizWiz and OK-VQA accuracy and the best average with both checkpoints. On Qwen-VL-7B OK-VQA, PT and the three trained per-layer interventions fall below zero-shot, while STAVE raises accuracy from 57.51\% to 62.39\%.

All compared methods start from the authors' released code, and we adapt each one to the datasets and protocol of these tables for a fair comparison. Following their papers, MimIC and HiFICL train for 10 epochs with AdamW~\citep{loshchilov2017decoupled} at learning rate $5\times10^{-3}$ and weight decay $10^{-3}$, batch size 2 with 2 accumulation steps, a cosine schedule with 10\% warmup, and gradient clipping at 1.0. MimIC uses a four-shot teacher with loss weights 0.5 for CE and 1 for alignment. HiFICL uses rank 8 and eight virtual tokens. Decoding is greedy. Each of the 10 epochs is scored on the validation samples with the test metric, ties break by the lower mean teacher-forced answer CE and then by the earlier epoch, and only the selected epoch is evaluated on the test split.

STAVE trains its $2d$ parameters with the source and target CE losses weighted 1 to 1, using AdamW with weight decay 0 and a cosine schedule that decays the learning rate to zero. The source prompt holds one demo on VizWiz and OK-VQA and two on DTD, Flowers, and CUB. Each run takes 250 to 4,000 training steps of $K=4$ or $K=8$ paired examples, with a learning rate between $10^{-3}$ and $10^{-2}$ set per model and dataset. As for MimIC and HiFICL, and unlike Table~\ref{tab:vqav2-coco}, the checkpoint is selected on the validation samples, and only the selected checkpoint is evaluated on the test split.
The three PT modes in Table~\ref{tab:vqa+finegrained} also learn 20 virtual tokens each on the prompt without demos. They train for 500 training steps of $K=4$ samples with AdamW at weight decay $10^{-2}$ and a cosine schedule after a linear warmup, with the CE of the first answer token as the loss. The learning rate is $10^{-2}$ in most entries. When a run collapsed at one seed, that seed was trained again at another learning rate between $3\times10^{-4}$ and $5\times10^{-3}$, so the three seeds of an entry do not always share one learning rate.

\subsection{TV benchmark}
\label{app:impl-text}

Table~\ref{tab:text} uses the 18 tasks of the TV benchmark~\citep{hendel-etal-2023-context} with five frozen LLMs. Each task has one split, drawn with seed 41, of at most 256 training, 64 development, and 200 test queries. Every baseline uses the same prompts, training samples, and test queries as STAVE, and every per-task choice of a baseline is made on the demo-free development split, never on the test split. Each baseline runs with the authors' released code and follows the settings of its paper. TV and FV~\citep{todd2024function} choose the injection layer per task, as in their papers. SITE~\citep{park2026soft} runs all three stages, with the task embedding averaged over 50 prompts of 10 demos. PT again learns 20 virtual tokens. LoRA follows \citet{hu2022lora} and applies rank 8 with scaling $\alpha=16$ to the attention query and value projections of every decoder block, or to the fused query, key, and value projection on Pythia. It trains with AdamW at weight decay 0.01, $K=8$, and at most 500 training steps with early stopping on the development loss, and its learning rate is $10^{-4}$ or $10^{-3}$, chosen per task by the development loss. STAVE picks its configuration per task by the lowest demo-free development CE, with one demo in the source prompt. The chosen configurations use AdamW with weight decay 0.01 on 84 of the 90 backbone and task pairs and 0 on the rest, 500 or 1,000 training steps, $K=8$ paired examples per step on 74 pairs and 16 or 32 on the rest, and a learning rate between 0.01 and 0.5. The context vector uses 0.003 to 0.3 times the learning rate of the readout vector. The average weights each task by its number of test queries.

\subsection{Hyperparameter sensitivity}
\label{app:hparams}

Figures~\ref{fig:hparam-lr} to~\ref{fig:hparam-demos} vary one training setting at a time on four model and dataset pairs of Table~\ref{tab:vqav2-coco}, LLaVA VQAv2 and Idefics2 VQAv2, OK-VQA, and COCO. Every run uses seed 41 and the 1,000 training samples of that table and reports its best epoch, the checkpoint selection used for its entries. Unless a figure varies them, runs use $K=4$ paired examples per training step in Equation~\ref{eq:stave-objective}, five epochs, weight decay 0.001, and one source demo on OK-VQA or eight on COCO. Figures~\ref{fig:hparam-epochs} and~\ref{fig:hparam-batch} use the learning rate with the highest score in Figure~\ref{fig:hparam-lr}. The strongest baseline for a model and dataset is the best non-STAVE entry of its column in Table~\ref{tab:vqav2-coco}. Scores are accuracy (\%) for VQAv2 and OK-VQA and CIDEr for COCO.

\begin{figure}[!t]
\centering
\includegraphics[width=\linewidth]{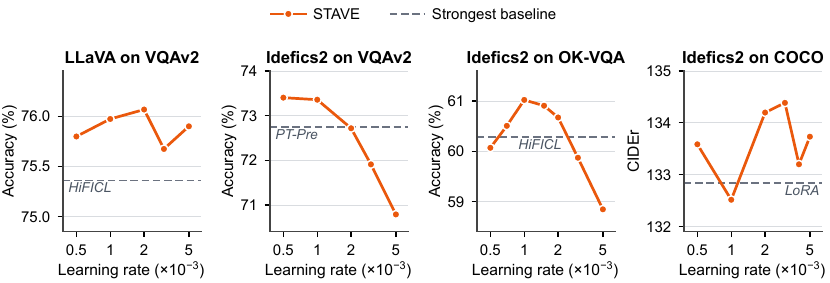}\par
\vspace{-3.2pt}
\caption{Learning-rate sensitivity on four model and dataset pairs of Table~\ref{tab:vqav2-coco}.}
\label{fig:hparam-lr}

\vspace{6pt}
\begin{minipage}[t]{0.323\linewidth}
\vspace{0pt}%
\centering
\includegraphics[width=\linewidth]{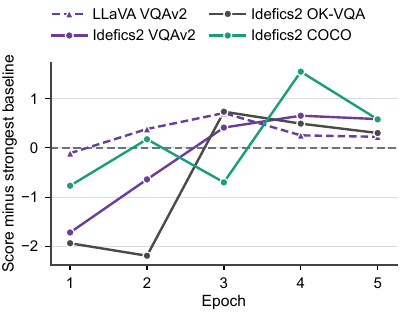}\par
\vspace{-3.2pt}
\caption{Score after each training epoch, relative to the strongest baseline.}
\label{fig:hparam-epochs}
\end{minipage}\hfill
\begin{minipage}[t]{0.323\linewidth}
\vspace{0pt}%
\centering
\includegraphics[width=\linewidth]{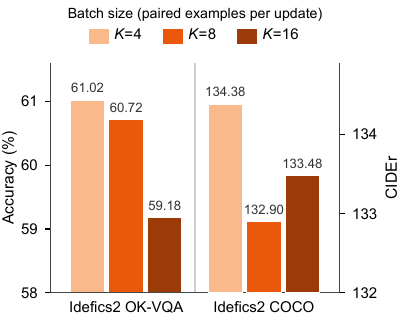}\par
\vspace{-3.2pt}
\caption{Batch size $K$ on two Idefics2 datasets (axes start above zero).}
\label{fig:hparam-batch}
\end{minipage}\hfill
\begin{minipage}[t]{0.323\linewidth}
\vspace{0pt}%
\centering
\includegraphics[width=\linewidth]{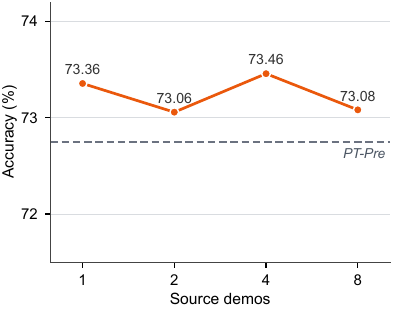}\par
\vspace{-3.2pt}
\caption{Number of source demos on Idefics2 VQAv2.}
\label{fig:hparam-demos}
\end{minipage}
\end{figure}

\textbf{Learning rate.}
STAVE stays above the strongest baseline over most of the learning-rate range in Figure~\ref{fig:hparam-lr}. On LLaVA VQAv2 every learning rate from $5\times10^{-4}$ to $5\times10^{-3}$ is above it, and on Idefics2 COCO every one except $10^{-3}$. On Idefics2 VQAv2 the score decreases as the learning rate grows and remains above the strongest baseline up to $10^{-3}$. Idefics2 OK-VQA peaks at $10^{-3}$ and exceeds the strongest baseline from $7\times10^{-4}$ to $2\times10^{-3}$.

\textbf{Training epochs.}
Figure~\ref{fig:hparam-epochs} follows the score after every epoch. All four model and dataset pairs are above the strongest baseline from the fourth epoch on, including the last one, so the gain over the strongest baseline does not depend on selecting a single epoch.

\textbf{Batch size.}
Figure~\ref{fig:hparam-batch} varies $K$ on the two Idefics2 datasets and compares STAVE with itself. $K=4$ gives the highest score on both datasets. With the learning rate and the number of epochs fixed, a larger $K$ means fewer training steps per epoch, 63 at $K=16$ against 250 at $K=4$, and the OK-VQA score decreases as $K$ grows.

\textbf{Number of source demos.}
Figure~\ref{fig:hparam-demos} varies the number of source demos on Idefics2 VQAv2 from one to eight, with every other setting fixed. The score changes by at most 0.40 points, from 73.06\% to 73.46\%, and does not rise or fall with the number of demos. Every setting remains above the strongest baseline, PT-Pre at 72.75\%. The number of source demos therefore has little effect on STAVE.

\subsection{Dataset and prompt examples}
\label{app:dataset-examples}

This subsection shows the decoded, demo-free query block and its reference answer. The VQA and captioning queries come from the evaluation splits. The source branch used during training places demos before the same query block. Angle-bracketed image markers mark where the image shown beside the prompt is placed. Each visual sample is identified by its dataset ID or original file path.

\par\ifdim\dimexpr\pagegoal-\pagetotal\relax<0.25\textheight\newpage\fi
\subsubsection{VQA and captioning}

\begin{datasetexample}{VQAv2: question ID 575633002}
\begin{minipage}[t]{0.24\linewidth}
\vspace{0pt}
\centering
\includegraphics[width=\linewidth,height=0.14\textheight,keepaspectratio]{dataset_examples/vqav2_575633.jpg}
\end{minipage}\hfill
\begin{minipage}[t]{0.73\linewidth}
\vspace{0pt}
\textbf{Decoded LLaVA-Interleave prompt}\par
{\ttfamily\footnotesize\frenchspacing
Instruction: provide an answer to the question. Use the image to answer.\par
Image:<image>Question:How many sides does the STOP sign have?\par
Answer:\par}
\vspace{4pt}
\textbf{Reference answer:} 8. The ten annotations are \{8, 8, 2, 8, 8, 8, 8, 8, 8, 8\}.
\end{minipage}
\end{datasetexample}

\begin{datasetexample}{OK-VQA: question ID 4516835}
\begin{minipage}[t]{0.24\linewidth}
\vspace{0pt}
\centering
\includegraphics[width=\linewidth,height=0.14\textheight,keepaspectratio]{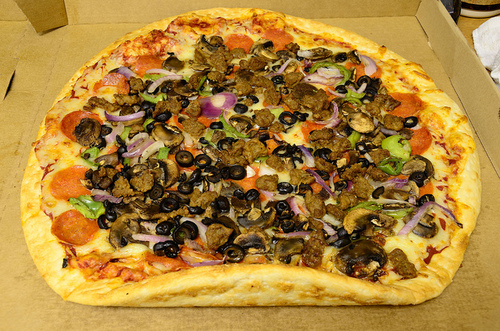}
\end{minipage}\hfill
\begin{minipage}[t]{0.73\linewidth}
\vspace{0pt}
\textbf{Decoded prompt}\par
{\ttfamily\footnotesize\frenchspacing
<image>What black veggie is on this pizza? Answer:\par}
\vspace{4pt}
\textbf{Reference answer:} olives. Eight annotators answered ``olives'' and two answered ``olive.''
\end{minipage}
\end{datasetexample}

\begin{datasetexample}{VizWiz: question ID VizWiz\_val\_00001816}
\begin{minipage}[t]{0.24\linewidth}
\vspace{0pt}
\centering
\includegraphics[width=\linewidth,height=0.15\textheight,keepaspectratio]{dataset_examples/vizwiz_00001816.jpg}
\end{minipage}\hfill
\begin{minipage}[t]{0.73\linewidth}
\vspace{0pt}
\textbf{Decoded prompt}\par
{\ttfamily\footnotesize\frenchspacing
First carefully understand the given examples.\par
Then use the given image and answer the question in the same way as the examples.\par
If the question can not be answered, respond unanswerable. <image>What is this? Answer:\par}
\vspace{4pt}
\textbf{Reference answer:} coffee creamer. The ten annotations include ``coffee creamer,'' ``coffee mate,'' and more specific variants such as ``french vanilla coffee creamer.''
\end{minipage}
\end{datasetexample}

\begin{datasetexample}{COCO Captioning: image ID 306229}
\begin{minipage}[t]{0.24\linewidth}
\vspace{0pt}
\centering
\includegraphics[width=\linewidth,height=0.14\textheight,keepaspectratio]{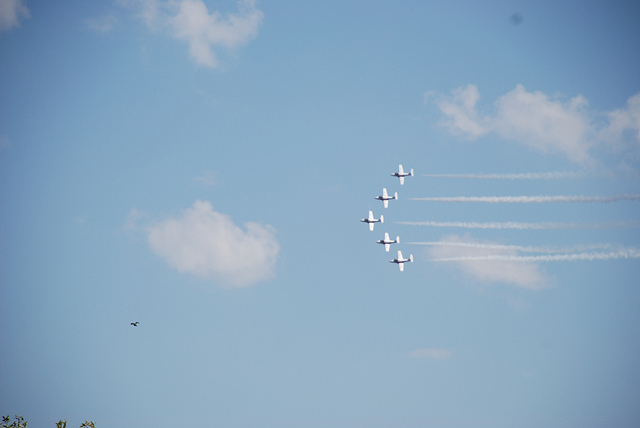}
\end{minipage}\hfill
\begin{minipage}[t]{0.73\linewidth}
\vspace{0pt}
\textbf{Decoded Idefics2 prompt}\par
{\ttfamily\footnotesize\frenchspacing
Image:<image><end\_of\_outterance>\par
Caption:\par}
\vspace{4pt}
\textbf{Reference captions:} ``This is an air show put on by the US Navy Blue Angels.'' Other references describe five airplanes flying in formation, smoke from their tails, and a group of small airplanes in the sky. The separator keeps the spelling of the released MimIC Idefics2 template.
\end{minipage}
\end{datasetexample}

\subsubsection{Fine-grained visual classification}

The three classification datasets use the same two-choice query format, and evaluation uses the same format on the test split. Each example below is the first query of the fitting split, which the target branch sees without demos. The options are shuffled deterministically. The answer is an option letter, while the class name below makes the label explicit.

\begin{datasetexample}{Describable Textures Dataset: query banded/banded\_0005.jpg}
\begin{minipage}[t]{0.24\linewidth}
\vspace{0pt}
\centering
\includegraphics[width=\linewidth,height=0.12\textheight,keepaspectratio]{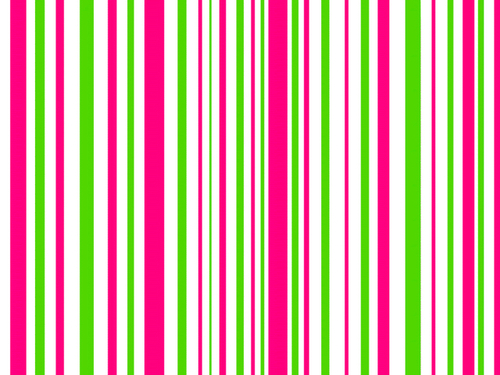}
\end{minipage}\hfill
\begin{minipage}[t]{0.73\linewidth}
\vspace{0pt}
\textbf{Decoded prompt}\par
{\ttfamily\footnotesize\frenchspacing
<image>What is the type of texture in the image? A.crosshatched B.banded\par
Answer with the option\textquotesingle{}s letter from the given choice directly. Answer:\par}
\vspace{4pt}
\textbf{Reference answer:} B, banded.
\end{minipage}
\end{datasetexample}

\begin{datasetexample}{Flowers-102: query image\_06765.jpg}
\begin{minipage}[t]{0.24\linewidth}
\vspace{0pt}
\centering
\includegraphics[width=\linewidth,height=0.12\textheight,keepaspectratio]{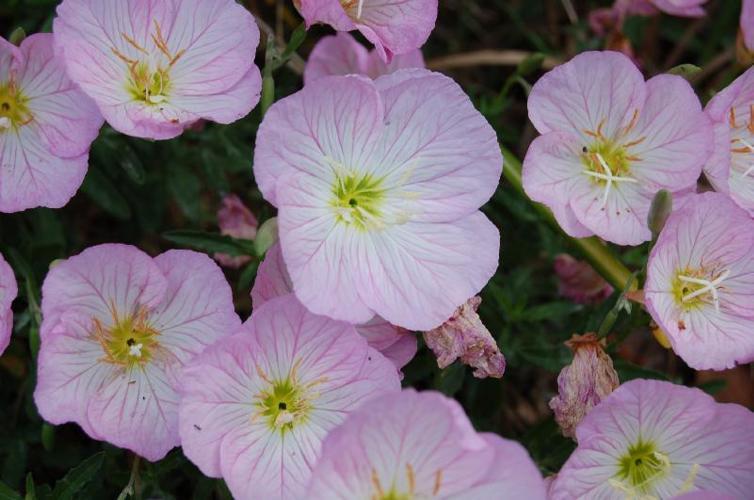}
\end{minipage}\hfill
\begin{minipage}[t]{0.73\linewidth}
\vspace{0pt}
\textbf{Decoded prompt}\par
{\ttfamily\footnotesize\frenchspacing
<image>What is the type of flower in the image? A.purple coneflower B.pink primrose\par
Answer with the option\textquotesingle{}s letter from the given choice directly. Answer:\par}
\vspace{4pt}
\textbf{Reference answer:} B, pink primrose.
\end{minipage}
\end{datasetexample}

\begin{datasetexample}{CUB-200-2011: query Black\_Footed\_Albatross\_0009\_34.jpg}
\begin{minipage}[t]{0.24\linewidth}
\vspace{0pt}
\centering
\includegraphics[width=\linewidth,height=0.12\textheight,keepaspectratio]{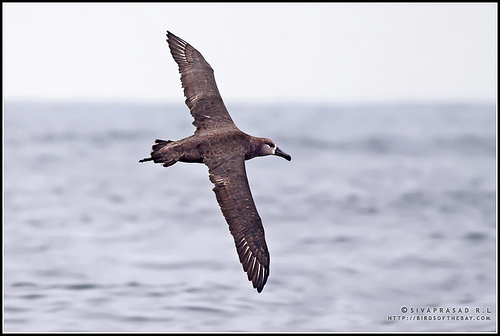}
\end{minipage}\hfill
\begin{minipage}[t]{0.73\linewidth}
\vspace{0pt}
\textbf{Decoded prompt}\par
{\ttfamily\footnotesize\frenchspacing
<image>What is the type of bird in the image? A.Crested Auklet B.Black footed Albatross\par
Answer with the option\textquotesingle{}s letter from the given choice directly. Answer:\par}
\vspace{4pt}
\textbf{Reference answer:} B, Black footed Albatross.
\end{minipage}
\end{datasetexample}

\subsubsection{TV benchmark tasks}

Every TV benchmark task uses the demo-free template \texttt{example:<input>->}. The answer tokens follow the arrow during supervision. In the translation and linguistic tasks, the input and the answer each start with one space, as in \texttt{example:\ chien->}, and a pair is kept only when its answer is a single token for the backbone's tokenizer. The algorithmic and knowledge tasks add no space and apply no single-token filter. The entries below are input-output pairs that remain after filtering on every backbone, or records from the deterministic algorithmic task generators.

Each task draws up to 200 test queries, and a task with a small input space yields
fewer. Next letter and Previous letter have 25 possible inputs and the two case tasks
have 26, so their test splits hold 12 or 13 queries. The single-token filter also
shrinks some translation and linguistic tasks. English to Spanish holds 33 to 110 test
queries depending on the backbone, while French to English and Spanish to English hold
200. An unweighted mean over the 18
tasks would give a 12-query split the same influence as a 200-query split, so every
TV benchmark average in this paper weights each task by its number of test queries. The
reported figure is therefore the accuracy over the pooled test queries of the suite,
2,138 to 2,202 per backbone. Every column uses the same weighting, and selection is
unaffected because it is made per task on the development split.

\begin{datasetexample}{Algorithmic tasks}
\footnotesize
\renewcommand{\arraystretch}{1.12}
\begin{tabularx}{\linewidth}{@{}lXl@{}}
\toprule
\textbf{Task} & \textbf{Prompt} & \textbf{Reference answer} \\
\midrule
Next letter & {\ttfamily example:a->} & b \\
Previous letter & {\ttfamily example:b->} & a \\
List first & {\ttfamily example:a,b,c->} & a \\
List last & {\ttfamily example:a,b,c->} & c \\
To uppercase & {\ttfamily example:a->} & A \\
To lowercase & {\ttfamily example:A->} & a \\
\bottomrule
\end{tabularx}
\end{datasetexample}

\begin{datasetexample}{Translation tasks}
\footnotesize
\renewcommand{\arraystretch}{1.12}
\begin{tabularx}{\linewidth}{@{}lXl@{}}
\toprule
\textbf{Task} & \textbf{Prompt} & \textbf{Reference answer} \\
\midrule
French to English & {\ttfamily example:\ chien->} & dog \\
Spanish to English & {\ttfamily example:\ perro->} & dog \\
English to French & {\ttfamily example:\ good->} & bon \\
English to Spanish & {\ttfamily example:\ two->} & dos \\
\bottomrule
\end{tabularx}
\end{datasetexample}

\begin{datasetexample}{Linguistic tasks}
\footnotesize
\renewcommand{\arraystretch}{1.12}
\begin{tabularx}{\linewidth}{@{}lXl@{}}
\toprule
\textbf{Task} & \textbf{Prompt} & \textbf{Reference answer} \\
\midrule
Present to gerund & {\ttfamily example:\ run->} & running \\
Present to past & {\ttfamily example:\ run->} & ran \\
Plural to singular & {\ttfamily example:\ children->} & child \\
Antonym & {\ttfamily example:\ hot->} & cold \\
\bottomrule
\end{tabularx}
\end{datasetexample}

\begin{datasetexample}{Knowledge tasks}
\footnotesize
\renewcommand{\arraystretch}{1.12}
\begin{tabularx}{\linewidth}{@{}lXl@{}}
\toprule
\textbf{Task} & \textbf{Prompt} & \textbf{Reference answer} \\
\midrule
Country to capital & {\ttfamily example:Germany->} & Berlin \\
Person to language & {\ttfamily example:Gilad Atzmon->} & Hebrew \\
Location to continent & {\ttfamily example:Paris->} & Europe \\
Location to religion & {\ttfamily example:Edwin of Northumbria->} & Christian \\
\bottomrule
\end{tabularx}
\end{datasetexample}

\section{Additional results}
\label{app:per-category}

\subsection{More recent LMMs}
\label{app:recent-lmms}

Table~\ref{tab:vl-modern} extends the VizWiz and OK-VQA comparison of Table~\ref{tab:vqa+finegrained} to InternVL3.5-8B, Qwen2.5-VL-7B, and Idefics3-8B, under the protocol of Appendix~\ref{app:impl-table2}, and each trained method is run three times. STAVE has the best mean on five of the six columns, with 14 to 36 times fewer parameters than LIVE and MimIC and 238 to 306 times fewer than HiFICL, and it exceeds 4-shot ICL by more than 10 points on every column.

\begin{table}[!htbp]
\centering
\caption{Comparison on VizWiz and OK-VQA with more recent LMMs.}
\label{tab:vl-modern}
\footnotesize
\setlength{\tabcolsep}{4pt}
\resizebox{\textwidth}{!}{%
\begin{tabular}{l|c|cc|c|cc|c|cc}
\toprule
& \multicolumn{3}{c|}{\textbf{InternVL3.5-8B}} & \multicolumn{3}{c|}{\textbf{Qwen2.5-VL-7B}} & \multicolumn{3}{c}{\textbf{Idefics3-8B}} \\
\cmidrule(lr){2-4}\cmidrule(lr){5-7}\cmidrule(lr){8-10}
\rowcolor{tableheader}
\textbf{Method} & \# Params (M) & VizWiz & OK-VQA & \# Params (M) & VizWiz & OK-VQA & \# Params (M) & VizWiz & OK-VQA \\
\midrule
Zero-shot & -- & 27.7 & 19.4 & -- & 48.1 & 54.5 & -- & 29.9 & 45.2 \\
4-shot ICL & -- & 56.5 & 47.4 & -- & 54.5 & 43.0 & -- & 37.6 & 45.9 \\
\midrule
LIVE   & \cellcolor{parammedium}{0.15 ($\times$18.00)} & 63.54$_{\pm 0.71}$ & 49.98$_{\pm 0.13}$ & \cellcolor{parammedium}{0.10 ($\times$14.00)} & 52.97$_{\pm 9.08}$ & 55.31$_{\pm 0.33}$ & \cellcolor{parammedium}{0.13 ($\times$16.00)} & 46.93$_{\pm 1.96}$ & 39.25$_{\pm 1.03}$ \\
MimIC  & \cellcolor{parammedium}{0.30 ($\times$36.14)} & 66.79$_{\pm 1.42}$ & 57.36$_{\pm 0.42}$ & \cellcolor{parammedium}{0.20 ($\times$28.11)} & \textbf{73.44}$_{\pm 0.87}$ & \underline{64.55}$_{\pm 0.47}$ & \cellcolor{parammedium}{0.26 ($\times$32.13)} & \underline{66.71}$_{\pm 0.86}$ & \underline{57.22}$_{\pm 0.22}$ \\
HiFICL & \cellcolor{paramlarge}{2.5 ($\times$306.00)} & \underline{68.09}$_{\pm 0.41}$ & \underline{58.74}$_{\pm 0.17}$ & \cellcolor{paramlarge}{1.7 ($\times$238.00)} & \underline{71.74}$_{\pm 0.51}$ & 59.81$_{\pm 0.53}$ & \cellcolor{paramlarge}{2.2 ($\times$272.00)} & 65.60$_{\pm 0.77}$ & 56.53$_{\pm 0.64}$ \\
\midrule
\rowcolor{stavehighlight}
\stavelabel & \cellcolor{paramsmall}\textbf{0.008 ($\times$1.00)} & \textbf{68.37}$_{\pm 0.71}$ & \textbf{58.94}$_{\pm 0.11}$ & \cellcolor{paramsmall}\textbf{0.007 ($\times$1.00)} & 71.34$_{\pm 0.24}$ & \textbf{64.91}$_{\pm 0.29}$ & \cellcolor{paramsmall}\textbf{0.008 ($\times$1.00)} & \textbf{68.66}$_{\pm 0.18}$ & \textbf{58.74}$_{\pm 0.37}$ \\
\bottomrule
\end{tabular}%
}
\end{table}

\subsection{VQAv2 answer types}
\label{app:answer-type}

Figure~\ref{fig:vqav2-answer-type} splits the VQAv2 scores of Table~\ref{tab:vqav2-coco} by the three official answer types on all 10,000 evaluation questions, using one run of each method. Each point is the accuracy gain of STAVE over LIVE, MimIC, or HiFICL, with a paired bootstrap 95\% interval over questions (2,000 resamples). STAVE matches or exceeds every baseline on every answer type on both backbones. The other type holds half of the questions and most open-vocabulary answers, and there the gain is 1.3 to 2.6 points, with an interval that excludes zero against every baseline on both backbones.

\begin{figure}[!htbp]
\centering
\includegraphics[width=\linewidth]{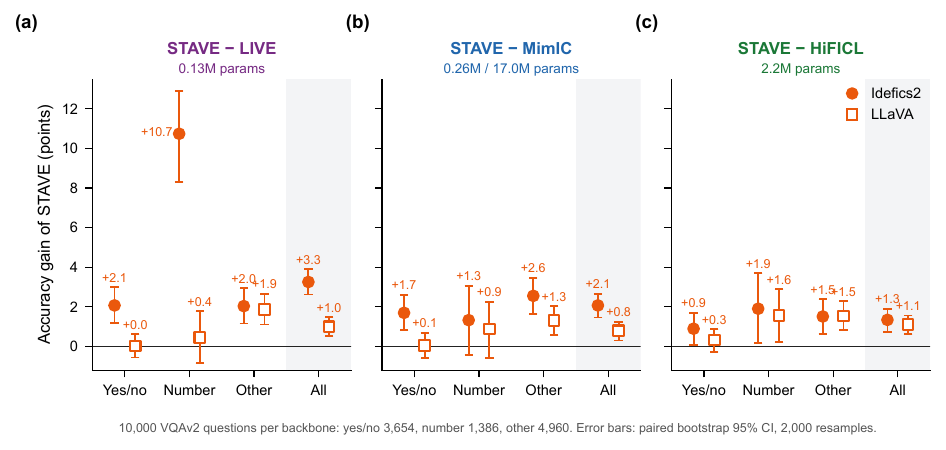}\par
\vspace{-3.2pt}
\caption{VQAv2 accuracy gain of STAVE over each trained baseline, by answer type.}
\label{fig:vqav2-answer-type}
\end{figure}

\section{Ablation details}
\label{app:ablations}

\subsection{Demos in the prompt on COCO Captioning}
\label{app:demo-robustness-coco}

Table~\ref{tab:vqav2-coco} evaluates every method without demos. Table~\ref{tab:demo-robustness-coco} puts demos back into the prompt at inference and changes nothing else. Every method is read from its own published checkpoint on the first 1,000 images of the COCO Karpathy validation split, and the demos of a query are drawn from the same pool in the same order for every method, so a column compares methods on identical prompts. Columns are the number of demos, and the mismatched column gives four demos whose captions belong to other images. Avg. drop averages the change from the demo-free score over the five demo conditions.

MimIC and HiFICL lose almost all captioning ability on LLaVA as soon as one demo is present, falling from about 135 CIDEr to below 5, and their loss does not shrink as demos are added. On Idefics2, MimIC holds its score within about 3 CIDEr of the demo-free one until the demos are mismatched, and HiFICL does not. Paired STAVE stays within 6 CIDEr of its demo-free score in every condition on both backbones, and within 0.7 CIDEr on Idefics2, including under mismatched demos. Target-only training falls to between 11 and 26 CIDEr on LLaVA, and the paired objective repairs that case.

The released evaluator turns a generation into a caption by cutting it at the first of a list of stop words, and the HiFICL Idefics2 entries at four and at mismatched demos hold 130 and 160 generations of 1,000 that open with a template marker, which places a stop word at the first token and leaves an empty caption. Those two entries are scored after one leading marker is removed, which recovers 93 and 59 captions. The rest are empty generations, so both entries stay far below the demo-free score.

\begin{table}[!htbp]
\centering
\caption{COCO CIDEr with demos added to the prompt at inference (first 1,000 Karpathy validation images).}
\label{tab:demo-robustness-coco}
\footnotesize
\setlength{\tabcolsep}{4pt}
\renewcommand{\arraystretch}{1.15}
\begin{tabular}{l|cccccc|c}
\toprule
\rowcolor{tableheader}
\textbf{Method} & \textbf{0} & \textbf{1} & \textbf{2} & \textbf{4} & \textbf{8} & \textbf{mismatched} & \textbf{Avg. drop} \\
\midrule
\multicolumn{8}{@{}l}{\textit{LLaVA-Interleave, COCO}} \\
MimIC & 134.6 & 3.3 & 3.0 & 3.4 & 3.9 & 4.3 & -131.0 \\
HiFICL & 138.0 & 3.4 & 3.0 & 3.4 & 4.5 & 4.0 & -134.4 \\
STAVE (source-only) & 134.0 & \textbf{135.2} & \textbf{134.8} & \textbf{135.8} & \textbf{137.4} & \textbf{135.7} & +1.8 \\
STAVE (target-only) & 135.2 & 11.3 & 21.8 & 21.6 & 25.5 & 15.0 & -116.2 \\
\rowcolor{stavehighlight}
STAVE (paired) & \textbf{138.8} & 133.9 & 134.7 & 135.5 & 134.4 & 133.0 & -4.5 \\
\midrule
\multicolumn{8}{@{}l}{\textit{Idefics2, COCO}} \\
MimIC & 140.4 & 139.3 & 139.1 & 137.3 & 137.6 & 94.5 & -10.8 \\
HiFICL & 137.7 & 114.6 & 129.7 & 66.5 & 136.0 & 21.4 & -44.1 \\
STAVE (source-only) & 138.1 & 140.4 & 141.4 & 140.8 & 141.0 & 139.8 & +2.6 \\
STAVE (target-only) & 141.5 & \textbf{141.8} & \textbf{142.0} & 142.4 & \textbf{143.4} & 138.1 & +0.0 \\
\rowcolor{stavehighlight}
STAVE (paired) & \textbf{141.8} & 141.1 & 141.6 & \textbf{142.5} & 142.1 & \textbf{142.1} & +0.0 \\
\bottomrule
\end{tabular}
\end{table}

\subsection{Object hallucination by training objective}
\label{app:objective-hallucination}

Table~\ref{tab:objective-hallucination} scores the Idefics2 COCO captions of the three objectives of Table~\ref{tab:training-objective-ablation} with the CHAIR metrics of Table~\ref{tab:hallucination}. Paired training has the highest recall, and its captions name absent objects less often than target-only captions on both CHAIR measures. Source-only captions name absent objects least but recall the fewest objects.

\begin{table}[!htbp]
\centering
\caption{Object hallucination by training objective on Idefics2 COCO.}
\label{tab:objective-hallucination}
\footnotesize
\setlength{\tabcolsep}{7pt}
\renewcommand{\arraystretch}{1.15}
\begin{tabular}{l|ccc}
\toprule
\rowcolor{tableheader}
\textbf{Objective} & CHAIRs $\downarrow$ & CHAIRi $\downarrow$ & Recall $\uparrow$ \\
\midrule
\rowcolor{stavehighlight}
Paired      & \underline{2.62} & \underline{1.76} & \textbf{44.81} \\
Target-only & 2.84 & 1.93 & \underline{44.47} \\
Source-only & \textbf{2.36} & \textbf{1.64} & 43.08 \\
\bottomrule
\end{tabular}
\end{table}

\subsection{Training objective on the TV benchmark}
\label{app:text-objective}

The paired objective has two branches. The source branch trains the learned vectors inside a one-demo prompt and the target branch trains them in the demo-free prompt. Table~\ref{tab:text-objective} removes one branch at a time. STAVE-source sets the target weight to zero, STAVE-target sets the source weight to zero, and STAVE-paired is the configuration of Table~\ref{tab:text}. Each entry reuses the configuration that Table~\ref{tab:text} selected for that backbone and task, with no further tuning, and is evaluated twice on the same test queries: without a demo, and with the same learned vectors placed in a one-demo prompt. Accuracy weights each task by its test queries. The TV bench rows of Table~\ref{tab:training-objective-ablation} average the three LLMs of Table~\ref{tab:text-objective}.

Each control holds only in the prompt format it was trained in. STAVE-source loses 2.4 to 25.4 points when the demo is removed, and STAVE-target loses up to 4.3 points when one is added. STAVE-paired stays within half a point of itself in both formats and is above each control in the format that control was not trained for. The one entry where a control is ahead is Pythia 6.9B with a demo, where STAVE-source is 0.2 points above STAVE-paired. The same pattern holds on Pythia 2.8B and Pythia 12B, where STAVE-source loses 4.4 and 24.5 points without a demo. SITE is listed for reference and is read the same way. It loses 0.8 to 1.0 points when a demo is added, a milder form of the STAVE-target pattern, so STAVE-paired is the only entry that holds its accuracy in both formats.

\begin{table}[!htbp]
\centering
\caption{Training objective ablation on the TV benchmark (accuracy, \%). Best objective in bold.}
\label{tab:text-objective}
\small
\setlength{\tabcolsep}{2.5pt}
\renewcommand{\arraystretch}{1.15}
\begin{tabular}{l|ccc|ccc|ccc}
\toprule
 & \multicolumn{3}{c|}{\textbf{Pythia 6.9B}} & \multicolumn{3}{c|}{\textbf{GPT-J 6B}} & \multicolumn{3}{c}{\textbf{LLaMA 7B}} \\
\cmidrule(lr){2-4}\cmidrule(lr){5-7}\cmidrule(lr){8-10}
\rowcolor{tableheader}
\textbf{Method} & \textbf{w/o demo} & \textbf{w/ 1 demo} & $\Delta$ & \textbf{w/o demo} & \textbf{w/ 1 demo} & $\Delta$ & \textbf{w/o demo} & \textbf{w/ 1 demo} & $\Delta$ \\
\midrule
SITE & 90.4 & 89.4 & \ddown{1.0} & 91.6 & 90.6 & \ddown{1.0} & 92.1 & 91.3 & \ddown{0.8} \\
STAVE-source & 66.6 & \textbf{92.0} & \dup{25.4} & 82.8 & 92.3 & \dup{9.5} & 89.6 & 92.0 & \dup{2.4} \\
STAVE-target & 91.3 & 88.1 & \ddown{3.2} & 91.7 & 87.4 & \ddown{4.3} & 91.6 & 91.2 & \ddown{0.4} \\
\rowcolor{stavehighlight}
STAVE-paired & \textbf{91.6} & 91.8 & \dup{0.2} & \textbf{92.3} & \textbf{92.4} & \dup{0.1} & \textbf{92.3} & \textbf{92.5} & \dup{0.2} \\
\bottomrule
\end{tabular}
\end{table}

\subsection{Structural token groups}
\label{app:groups}

Figure~\ref{fig:groups-main}(a) and Table~\ref{tab:groups-main} hold the training configuration fixed and vary the group-to-vector assignment or the updated tokens. Every configuration trains for five epochs at the learning rate and demo count of the Table~\ref{tab:vqav2-coco} entry of its model and dataset, and reports its best epoch, the checkpoint selection used for that entry. Configuration names list the tokens covered by each learned vector, with vectors separated by a slash. The readout tokens are the final prompt token and the answer cue. The context tokens are all other eligible tokens. The image tokens are the query-image embeddings. In ``Text / image'', the text tokens are all eligible tokens outside the query image, and in ``Readout / text / image'', they exclude the readout tokens. ``Readout / context'' is the configuration used in the main experiments. A configuration marked ``only'' learns one vector and leaves all other tokens unchanged.

\begin{table}[!htbp]
\centering
\caption{Group ablation on COCO (CIDEr) and VQAv2 (accuracy, \%). Best: \textbf{bold}, second: \underline{underlined}.}
\label{tab:groups-main}
\footnotesize
\setlength{\tabcolsep}{7pt}
\renewcommand{\arraystretch}{1.15}
\begin{tabular}{l|cc|cc}
\toprule
 & \multicolumn{2}{c|}{\textbf{COCO}} & \multicolumn{2}{c}{\textbf{VQAv2}} \\
\cmidrule(lr){2-3}\cmidrule(lr){4-5}
\rowcolor{tableheader}
\textbf{Configuration} & Idefics2 & LLaVA & Idefics2 & LLaVA \\
\midrule
\multicolumn{5}{@{}l}{\textit{One vector}} \\
Context only            & 132.65 & 129.57 & 72.84 & 75.51 \\
Readout only            & 132.24 & 129.85 & 71.78 & 75.86 \\
All tokens           & 131.95 & 129.18 & 72.75 & 75.74 \\
\midrule
\multicolumn{5}{@{}l}{\textit{Two vectors}} \\
\rowcolor{stavehighlight}
Readout / context       & \textbf{134.38} & \textbf{131.78} & \underline{73.46} & \textbf{76.12} \\
Last / rest            & 133.28 & 128.98 & 73.33 & 75.77 \\
Text / image            & 133.35 & 129.12 & 72.67 & \underline{75.98} \\
\midrule
\multicolumn{5}{@{}l}{\textit{Three vectors}} \\
Readout / text / image  & 133.79 & \underline{130.79} & \textbf{73.62} & 75.96 \\
\midrule
\multicolumn{5}{@{}l}{\textit{One vector per group}} \\
Per group               & \underline{133.91} & 130.19 & 72.72 & 75.84 \\
\bottomrule
\end{tabular}
\end{table}

\textbf{One-vector configurations.}
``All tokens'' ties every eligible group to one vector. ``Readout only'' learns the readout vector and leaves the context tokens unchanged. ``Context only'' learns the context vector and leaves the readout tokens unchanged. Excluded groups are zeroed rather than reassigned. The LLaVA prompt has no image delimiters, so its \grp{Image-gate} group is empty. Demos, BOS tokens, padding, and answer tokens remain excluded in every configuration.

\textbf{Two-vector, three-vector, and per-group configurations.}
``Last / rest'' learns one vector at the final prompt token and one vector shared by every other eligible token, so the answer cue joins the context tokens instead of the final token. ``Text / image'' learns one vector for the text tokens, including \grp{Image-gate}, and one vector for the image tokens. ``Readout / text / image'' keeps the readout vector and splits the context tokens into one vector for the query-image tokens and one vector for the other context tokens. ``Per group'' is the full group set, with a separate vector for every structural token group. Every configuration recomputes the count normalization on its merged groups, so each learned vector is divided by the square root of the number of tokens it updates in the prompt.

\textbf{What the comparisons support.}
Figure~\ref{fig:groups-main}(a) shows the mean of the two backbones in Table~\ref{tab:groups-main}, on axes that start above zero. Overall, ``Readout / context'' is the best configuration. One learned vector is not enough: ``All tokens'', ``Readout only'', and ``Context only'' are below it for every model and dataset, by up to 1.7 points on Idefics2 VQAv2 and up to 2.6 CIDEr on LLaVA COCO. Two-vector partitions that mix readout and context tokens are also below it for every model and dataset, ``Text / image'' by up to 2.7 CIDEr and ``Last / rest'' by up to 2.8 CIDEr. More vectors bring no consistent gain. ``Per group'' is below it for every model and dataset, and ``Readout / text / image'' is above it only on Idefics2 VQAv2, by 0.16 points, inside the paired-bootstrap interval against ``Readout / context'', which spans about half a point on VQAv2 (2,000 resamples over validation samples). Appendix~\ref{app:row-roles} examines what each vector changes in the trained matrix.

\subsection{Allocation of a shared update}
\label{app:allocation}

\textbf{Resolution sweep.}
Figure~\ref{fig:motivation} tests the dilution of Equation~\ref{eq:sharing-bound} by changing only the number of image tokens. From 64 to 4{,}096 image tokens, one vector shared by all eligible tokens loses 81\% of its first-order decrease on OK-VQA and stays within about 3\% of the prediction with zero gradient at the image tokens of Proposition~\ref{prop:sharing-cost}. The readout/context split at the same total norm keeps its effect.

\begin{wrapfigure}{R}{0.47\textwidth}
\vspace{-\intextsep}
\centering
\includegraphics[width=\linewidth]{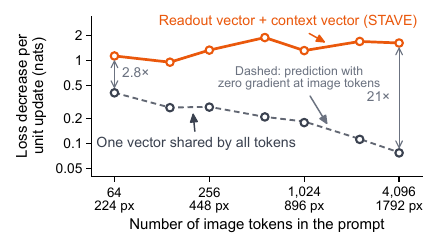}\par
\vspace{-9pt}
\caption{One shared vector against the readout/context split as the image grows (Qwen2.5-VL, OK-VQA).}
\label{fig:motivation}
\vspace{-\intextsep}
\end{wrapfigure}
The figure uses the frozen Qwen2.5-VL-7B backbone of Table~\ref{tab:vl-modern} with the demo-free OK-VQA prompt of that table. Each image is resized to $s\times s$ pixels with $s$ from 224 to 1792, which gives $(s/28)^2$ image tokens and leaves every text token unchanged. Besides the image, the prompt holds one \grp{First} token, one \grp{Image-gate} token, about nine question tokens, and two readout tokens, \grp{Answer-cue} and \grp{Last}. The same 160 training-pool samples are used at every size, 96 to fit the update directions and 64 to measure the loss. For each configuration, the direction is the negative mean zero-update gradient of its vector over the fitting samples, scaled to unit total norm, so that it changes the input embeddings by one in Frobenius norm. The reported value is the first-order decrease of the answer loss of Equation~\ref{eq:stave-ce} along this direction, averaged over the held-out samples. The prediction with zero gradient at the image tokens is the value of one vector on the text tokens, multiplied by $\sqrt{N_{\mathrm{text}}/N_1}$, where $N_{\mathrm{text}}$ counts the eligible text tokens and $N_1$ all eligible tokens. The readout/context configuration ties \grp{Last} and \grp{Answer-cue} to one vector and all other eligible tokens, including the image, to a second vector. Two vectors on random token sets with the readout and context counts, under three partition seeds, never exceed one shared vector by more than 0.3\%, so the gain comes from separating the readout tokens rather than from the second vector. Updates of norm $0.1$ give the same pattern in the actual loss. At 64 image tokens one shared vector lowers it by 0.047 nats and the readout/context split by 0.115 nats, and at 4{,}096 image tokens by 0.0002 and 0.167 nats. Smaller norms are not informative on this backbone, because bf16 rounding moves the loss of a single sample by up to about $0.1$ nats under any perturbation of its embeddings. Table~\ref{tab:resolution} lists the values for OK-VQA and for VizWiz, whose prompt also carries a 39-token instruction, so that the shared vector falls later there. Real photos fall inside this range. The default Qwen2.5-VL processor limits each image to 12{,}845{,}056 pixels and rounds its sides to multiples of 28. On 1{,}500 OK-VQA images drawn at random from the same pool, it gives 234 to 391 image tokens between the 10th and 90th percentiles, with a median of 345. On 1{,}500 VizWiz images the median is 1{,}610 and the 90th percentile 2{,}552.

\begin{table}[!htbp]
\centering
\caption{First-order loss decrease per unit update as the image grows (Qwen2.5-VL).}
\label{tab:resolution}
\footnotesize
\setlength{\tabcolsep}{6pt}
\renewcommand{\arraystretch}{1.12}
\begin{tabular}{rr|ccc|cc}
\toprule
\rowcolor{tableheader}
\textbf{Image tok.} & \textbf{Size (px)} & \textbf{One shared} & \textbf{Predicted} & \textbf{Readout / context} & \textbf{Ratio} & \textbf{Image share (\%)} \\
\midrule
\multicolumn{7}{l}{\textit{OK-VQA}} \\
64 & 224 & 0.408 & 0.409 & 1.137 & 2.8 & 1.95 \\
144 & 336 & 0.270 & 0.271 & 0.956 & 3.5 & 1.80 \\
256 & 448 & 0.274 & 0.276 & 1.337 & 4.9 & 1.41 \\
576 & 672 & 0.209 & 0.209 & 1.897 & 9.1 & 1.27 \\
1{,}024 & 896 & 0.178 & 0.184 & 1.320 & 7.4 & 1.35 \\
2{,}304 & 1344 & 0.112 & 0.113 & 1.702 & 15.2 & 0.68 \\
4{,}096 & 1792 & 0.077 & 0.078 & 1.630 & 21.2 & 0.74 \\
\midrule
\multicolumn{7}{l}{\textit{VizWiz}} \\
64 & 224 & 1.082 & 1.089 & 1.397 & 1.3 & 0.01 \\
144 & 336 & 0.936 & 0.947 & 1.399 & 1.5 & 0.01 \\
256 & 448 & 0.856 & 0.867 & 1.501 & 1.8 & 0.02 \\
576 & 672 & 0.806 & 0.820 & 2.125 & 2.6 & 0.04 \\
1{,}024 & 896 & 0.877 & 0.897 & 3.352 & 3.8 & 0.04 \\
2{,}304 & 1344 & 0.550 & 0.564 & 3.100 & 5.6 & 0.04 \\
4{,}096 & 1792 & 0.443 & 0.453 & 3.213 & 7.3 & 0.06 \\
\bottomrule
\end{tabular}
\tablekey{One shared: one vector on all eligible tokens. Predicted: the same vector with zero gradient at the image tokens. Readout / context: one vector on \grp{Last} and \grp{Answer-cue} and one on all other eligible tokens, with the same total norm. Ratio: readout / context over one shared. Image share: part of the one-shared value contributed by image tokens.}
\end{table}

\textbf{Native resolution.}
The LLaVA measurements below use the frozen LLaVA backbone of Table~\ref{tab:vqav2-coco} on VQAv2, OK-VQA, and COCO Captioning, with the prompt templates and structural token groups of the main experiments. All prompts are demo-free, the loss is the answer CE of Equation~\ref{eq:stave-ce} in nats per answer token, and all images come from the training pool, never from an evaluation split. Each configuration learns one or two vectors, which are the rows of $W$ in Equation~\ref{eq:stave-injection}.

\textbf{Effect of one token.}
The eligible tokens form three sets: the readout tokens (\grp{Last} and \grp{Answer-cue}), the query-image tokens, and the remaining context tokens, which are all text tokens in the LLaVA prompts. For each set, one vector is tied to its tokens with the count normalization of Equation~\ref{eq:stave-injection}, and all other tokens stay unchanged. On 300 prompts per task, the negative mean gradient of this vector at zero update is estimated on one half and scaled to unit norm, and the first-order loss decrease along this direction is averaged over the other half. Dividing it by the square root of the number of tokens in the set gives the average decrease for a unit shift at one token of the set. A readout token exceeds a query-image token by a factor of about 19{,}000 on VQAv2 and OK-VQA and about 14{,}000 on COCO. With the same halves, query-image tokens contribute 0.2\% of the first-order decrease of one shared vector on VQAv2 and 0.04\% on COCO, although they form 95\% and 99\% of the eligible tokens.

\textbf{Equal update size.}
For each configuration, the update direction is the negative mean zero-update gradient of its vectors over 96 images, scaled to unit total norm. Because the columns of $H(p)$ in Equation~\ref{eq:stave-injection} are orthonormal when every row of $W$ is nonempty, vectors of total norm $\rho$ change the input embeddings by exactly $\rho$ in Frobenius norm. The loss is measured on 64 other images at $\rho=0.01$, a value fixed before the measurement, so no optimizer or learning rate is involved. The random configurations assign the eligible tokens of each prompt at random to two vectors with the readout and context counts, under three partition seeds. Table~\ref{tab:allocation} reports the resulting loss decreases. At the same update norm, the readout/context split lowers the held-out loss 3.4 to 7.5 times as much as one shared vector, while random splits with the same counts stay within 5\% of it.

\begin{table}[!htbp]
\centering
\caption{Held-out loss decrease at update norm $0.01$ on LLaVA. Ratios are relative to one shared vector.}
\label{tab:allocation}
\footnotesize
\setlength{\tabcolsep}{7pt}
\renewcommand{\arraystretch}{1.15}
\begin{tabular}{l|c|c|c}
\toprule
\rowcolor{tableheader}
\textbf{Task} & \textbf{One shared vector (nats)} & \textbf{Readout / context} & \textbf{Random splits} \\
\midrule
VQAv2  & 0.113 & \cellcolor{stavehighlight}7.52$\times$ & 0.99--1.00$\times$ \\
OK-VQA & 0.094 & \cellcolor{stavehighlight}7.31$\times$ & 1.01--1.03$\times$ \\
COCO   & 0.028 & \cellcolor{stavehighlight}3.43$\times$ & 0.95--1.01$\times$ \\
\bottomrule
\end{tabular}
\end{table}

\subsection{Count normalization}
\label{app:normalization}

Table~\ref{tab:normalization} varies only the scale applied to each learned vector and keeps every other setting of the Table~\ref{tab:vqav2-coco} entry of the model and dataset. Write $n_k$ for the number of tokens of structural token group $k$ in the prompt and $N_r$ for the merged count of the groups that share vector $r$. ``Vector count'' divides a vector by $\sqrt{N_r}$ and is the normalization used in all other experiments. ``Group count'' divides each token by the square root of its own group count, so a vector that ties several groups is not renormalized after merging. ``Vector mean'' divides by $N_r$. ``None'' applies no scale. Every configuration uses the two-vector assignment, paired supervision, five epochs, and the learning rate and demo count of the model and dataset, and reports its best epoch.

\begin{table}[!htbp]
\centering
\caption{Count normalization ablation on VQAv2 (accuracy, \%) and COCO (CIDEr).}
\label{tab:normalization}
\footnotesize
\setlength{\tabcolsep}{7pt}
\renewcommand{\arraystretch}{1.15}
\begin{tabular}{lc|cc|cc}
\toprule
 & & \multicolumn{2}{c|}{\textbf{VQAv2}} & \multicolumn{2}{c}{\textbf{COCO}} \\
\cmidrule(lr){3-4}\cmidrule(lr){5-6}
\rowcolor{tableheader}
\textbf{Normalization} & \textbf{Scale} & Idefics2 & LLaVA & Idefics2 & LLaVA \\
\midrule
\rowcolor{stavehighlight}
Vector count & $1/\sqrt{N_r}$ & \textbf{73.46} & \textbf{76.12} & \textbf{134.38} & \textbf{131.78} \\
Group count & $1/\sqrt{n_k}$ & \underline{72.36} & 75.73 & 133.14 & 129.10 \\
Vector mean  & $1/N_r$        & 72.10 & \underline{75.87} & \underline{133.80} & 130.09 \\
None        & $1$            & 67.61 & 74.96 & 132.98 & \underline{130.23} \\
\bottomrule
\end{tabular}
\end{table}

\textbf{What the comparisons support.}
Vector count is the best entry for every model and dataset. Group count is below vector count for every model and dataset, by 1.10 points on Idefics2 VQAv2 and by 2.68 CIDEr on LLaVA COCO, which supports computing the scale on merged groups rather than on the original groups. Removing the scale is the most damaging choice on VQAv2 for both models, by 5.85 points on Idefics2 and 1.16 points on LLaVA. The readout vector covers the final token and the answer cue, while the context vector covers all image tokens, 64 on Idefics2 and 729 on LLaVA, so an unscaled update injects far more energy through the context vector than through the readout vector at the same per-coordinate step.

\subsection{Injection depth}
\label{app:depth}

The depth ablation replaces the embedding injection by an update after a decoder block while retaining the readout vector, context vector, group assignments, and prompt-only mask. The Qwen-VL depths use the VizWiz and OK-VQA entries in Table~\ref{tab:vqa+finegrained}. Both backbones have 32 decoder blocks. A displayed depth $k/8$ means injection after block $4k$ for $k=1,\ldots,8$. Depth zero denotes the input embeddings.

\textbf{Raw scores and retained gain.}
Figure~\ref{fig:depth-raw} contains all four raw curves. VQA scores are accuracies in percentages. Captioning uses CIDEr \citep{vedantam2015cider}, on the scale reported in Table~\ref{tab:vqav2-coco}.
The retained gain in Figure~\ref{fig:depth-ablation}(b) is computed as
\begin{equation}
 \operatorname{RetainedGain}(\ell)
 =\frac{s_\ell-s_{\mathrm{zero}}}{s_{\mathrm{emb}}-s_{\mathrm{zero}}}\times100\%.
\label{eq:retained-gain}
\end{equation}
Here $s_\ell$ is the score at injection depth $\ell$, $s_{\mathrm{zero}}$ is the unadapted zero-shot score, and $s_{\mathrm{emb}}$ is the embedding endpoint. A negative retained gain means the depth-specific score is below zero-shot.

Embedding injection is best or tied for best in every curve. Later Idefics2 scores degrade sharply, while Qwen-VL retains more of its embedding gain on VizWiz than on OK-VQA. At depth one, after the last block, Proposition~\ref{prop:final-depth} explains why the context vector receives zero gradient and why later teacher-forced stopping decisions are not directly controllable.

\begin{figure}[!htbp]
\centering
\includegraphics[width=\linewidth]{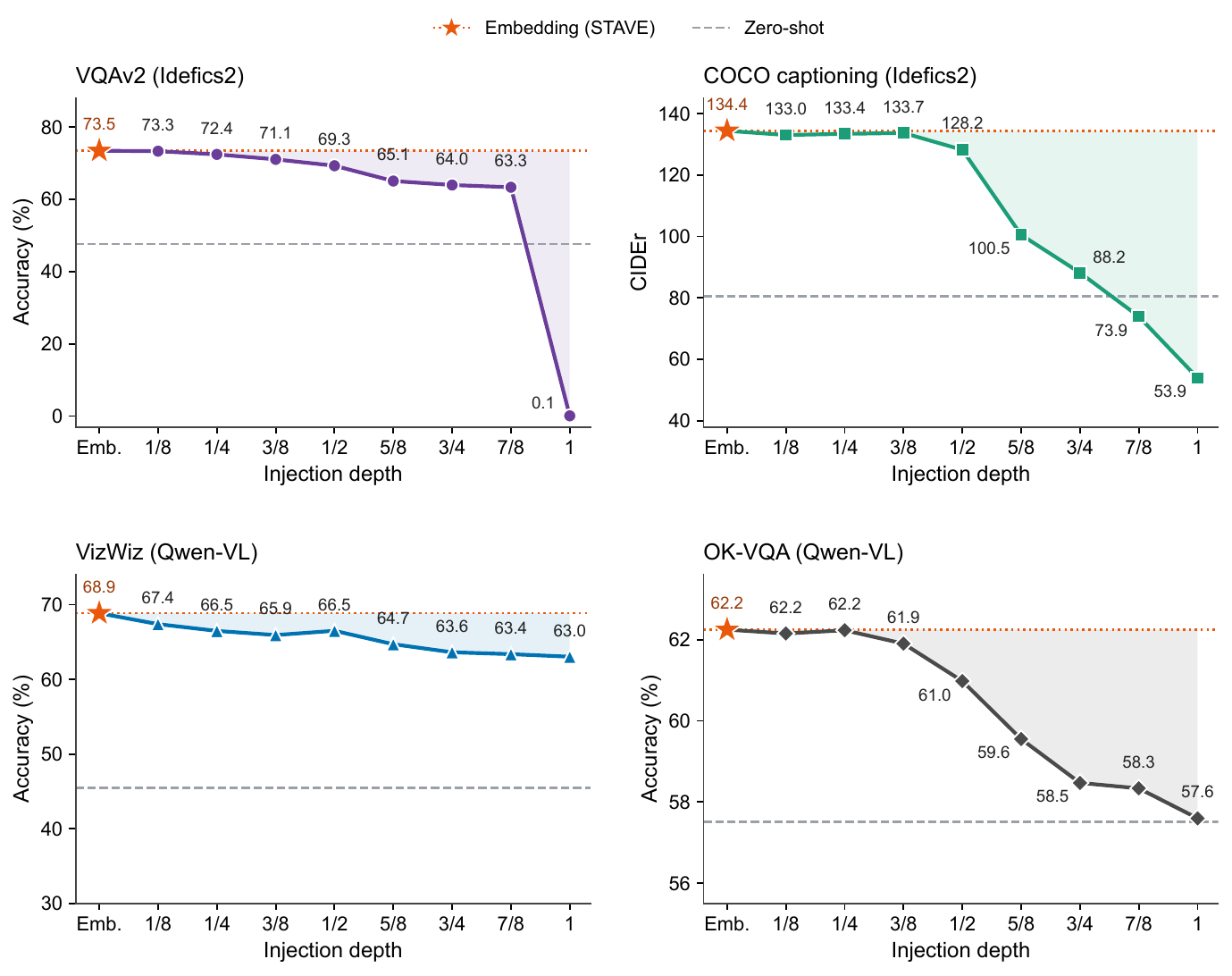}\par
\vspace{-3.2pt}
\caption{Scores at each injection depth on Idefics2 (top) and Qwen-VL (bottom).}
\label{fig:depth-raw}
\end{figure}

\section{Analysis of the learned vectors}
\label{app:analysis}

\subsection{Promoted tokens}
\label{app:promoted-tokens}

The learned vectors are added to input embeddings and act through the frozen decoder, so Table~\ref{tab:promoted-tokens} describes them by their effect on the next-token distribution. For a category with demo-free test prompts $p_1,\ldots,p_n$, the promotion of a vocabulary token $u$ is
\begin{equation}
 \Delta(u)=\frac{1}{n}\sum_{i=1}^{n}\Big[p_\theta\big(u\mid\widetilde E(p_i,W)\big)-p_\theta\big(u\mid E(p_i)\big)\Big],
\label{eq:promotion}
\end{equation}
where both distributions are taken at the final prompt token. Spellings with and without a leading space are merged, and whitespace tokens are omitted. Each cell lists the four tokens with the largest $\Delta(u)$ over 30 to 476 test prompts. A token is black when, lowercased and stripped of punctuation, it begins at least one reference answer of the category, and gray otherwise.

The checkpoints are the seed-41 runs of Tables~\ref{tab:vqav2-coco} and~\ref{tab:text} and the reported entries of Table~\ref{tab:vqa+finegrained}. Prompts are rendered as in the corresponding evaluation. VQAv2 and VizWiz categories are assigned from the question wording, the VizWiz unanswerable column uses the answerability label of the dataset, and the OK-VQA columns use its annotated knowledge categories. No category depends on a model prediction.

The same two vectors are added to every prompt of a task, yet the promoted tokens change with the category and are the answers it asks for. Over the 68 categories of the eight VQA checkpoints, excluding the catch-all other category, 95.9\% of the five most promoted tokens begin a reference answer. The table shows a subset of these categories. Adding demos without the learned vectors, as in the source prompt $p_i^{\mathrm{src}}$ of paired training, gives 91.8\% on the same categories, with four demos on VQAv2 and one on the other tasks. On the 90 TV benchmark checkpoints the two rates are 92.2\% and 82.2\%.

Several entries are the first subword of a longer answer, for example \tok{gir} for giraffe, \tok{sam} and \tok{son} for Samsung and Sony, \tok{vit} for vitamin, and \tok{un} for unanswerable. LLaVA answers without a leading space, and its tokenizer then splits more answers, as in \tok{base}, \tok{ten}, and \tok{soc} for baseball, tennis, and soccer.

\begin{table*}[!htbp]
\centering
\caption{Tokens most promoted by STAVE at the final prompt token. Gray tokens begin no reference answer of the column.}
\label{tab:promoted-tokens}
\providecommand{\tok}[1]{{\fontfamily{zi4}\selectfont #1}}
\providecommand{\miss}[1]{\textcolor{gray}{#1}}
\newcolumntype{P}{>{\raggedright\arraybackslash}p{7.2em}}
\newcolumntype{Q}{>{\raggedright\arraybackslash}p{8.3em}}
\newcolumntype{K}[1]{>{\centering\arraybackslash\hsize=#1\hsize}X}
\scriptsize
\setlength{\tabcolsep}{3pt}
\renewcommand{\arraystretch}{1.18}
{\footnotesize\textbf{(a) VQAv2, Table~\ref{tab:vqav2-coco} checkpoints}}\par\vspace{2pt}
\begin{tabularx}{\linewidth}{Q|K{0.923}K{0.989}K{1.088}}
\toprule
\rowcolor{tableheader}
\textbf{Model} & \textbf{Color} & \textbf{Sport} & \textbf{Animal} \\
\midrule
Idefics2-8B-base & \tok{blue, white, black, brown} & \tok{tennis, baseball, soccer, sk} & \tok{cat, horse, gir, z} \\
LLaVA & \tok{white, black, blue, red} & \tok{base, ten, soc, sk} & \tok{cat, g, dog, horse} \\
\midrule
\rowcolor{tableheader}
\textbf{Model} & \textbf{Food} & \textbf{Vehicle} & \textbf{Time or weather} \\
\midrule
Idefics2-8B-base & \tok{pizza, hot, sandwich, cake} & \tok{bus, motor, truck, train} & \tok{afternoon, winter, morning, sun} \\
LLaVA & \tok{pizza, hot, grass, sand} & \tok{bus, bike, motor, pickup} & \tok{winter, day, summer, s} \\
\bottomrule
\end{tabularx}
\par\vspace{7pt}
{\footnotesize\textbf{(b) OK-VQA, Table~\ref{tab:vqav2-coco} and Table~\ref{tab:vqa+finegrained} checkpoints}}\par\vspace{2pt}
\begin{tabularx}{\linewidth}{Q|K{1.074}K{1.000}K{0.926}}
\toprule
\rowcolor{tableheader}
\textbf{Model} & \textbf{Vehicles and transport} & \textbf{Brands and products} & \textbf{Objects and material} \\
\midrule
Idefics2-8B-base & \tok{stop, double, fighter, dies} & \tok{sam, son, apple, laptop} & \tok{t, clock, wood, plastic} \\
Idefics2-8B-instruct & \tok{stop, double, vol, c} & \tok{apple, son, sam, computer} & \tok{t, bear, wood, metal} \\
Qwen-VL-7B & \tok{bo, schw, east, stop} & \tok{apple, c, schw, monitor} & \tok{sit, cut, suit, leather} \\
\midrule
\rowcolor{tableheader}
\textbf{Model} & \textbf{Sports and recreation} & \textbf{Cooking and food} & \textbf{Weather and climate} \\
\midrule
Idefics2-8B-base & \tok{ski, surf, tennis, sk} & \tok{b, vit, c, cake} & \tok{sun, rain, winter, cold} \\
Idefics2-8B-instruct & \tok{ski, tennis, sk, surf} & \tok{b, c, vit, f} & \tok{rain, sun, cloud, cold} \\
Qwen-VL-7B & \tok{surf, fall, ser, run} & \tok{carrot, lunch, spinach, k} & \tok{dusk, str, rain, star} \\
\bottomrule
\end{tabularx}
\par\vspace{7pt}
{\footnotesize\textbf{(c) VizWiz, Table~\ref{tab:vqa+finegrained} checkpoints}}\par\vspace{2pt}
\begin{tabularx}{\linewidth}{Q|K{1.074}K{0.963}K{0.963}}
\toprule
\rowcolor{tableheader}
\textbf{Model} & \textbf{Unanswerable} & \textbf{Identify the object} & \textbf{Color} \\
\midrule
Idefics2-8B-instruct & \tok{un, no, yes, windows} & \tok{un, laptop, computer, tv} & \tok{white, grey, black, blue} \\
Qwen-VL-7B & \tok{un, \miss{priority}, pink, unknown} & \tok{un, computer, \miss{sweet}, mac} & \tok{un, white, black, blue} \\
\bottomrule
\end{tabularx}
\par\vspace{7pt}
{\footnotesize\textbf{(d) TV benchmark, Table~\ref{tab:text} checkpoints}}\par\vspace{2pt}
\begin{tabularx}{\linewidth}{P|K{0.711}K{1.206}K{1.082}}
\toprule
\rowcolor{tableheader}
\textbf{Model} & \textbf{Location to continent} & \textbf{Location to religion} & \textbf{Person to language} \\
\midrule
Pythia 6.9B & \tok{An, Europe, Asia, Af} & \tok{Muslim, Christian, Jewish, \miss{Chinese}} & \tok{English, French, Italian, Spanish} \\
Pythia 12B & \tok{An, Europe, Asia, Af} & \tok{Muslim, Christian, Jewish, Jew} & \tok{English, French, Italian, Spanish} \\
GPT-J 6B & \tok{Ant, Europe, Asia, Af} & \tok{Muslim, Christian, Jewish, \miss{Chinese}} & \tok{English, French, Spanish, Italian} \\
\midrule
\rowcolor{tableheader}
\textbf{Model} & \textbf{Country to capital} & \textbf{Present to gerund} & \textbf{English to French} \\
\midrule
Pythia 6.9B & \tok{Paris, Val, Pr, Ber} & \tok{staying, blessing, begging, providing} & \tok{par, comme, m, toujours} \\
Pythia 12B & \tok{Paris, Ber, Pr, Val} & \tok{rushing, staying, handing, guiding} & \tok{comme, par, image, toujours} \\
GPT-J 6B & \tok{Paris, Ber, Pr, Val} & \tok{singing, making, banning, baking} & \tok{dire, par, personnel, pour} \\
\bottomrule
\end{tabularx}
\end{table*}

\subsection{States and decoded tokens}
\label{app:representations}

The state of a prompt is the output of the last decoder block at the final prompt token, computed without the update and with the STAVE update.

On LLaVA VQAv2 we leave out the catch-all other category, which keeps 1,585 validation questions. Without the update, the next token is \tok{<|im\_start|>} for 99.1\% of them. In a principal component analysis (PCA) of the states of both conditions, the first component explains 50\% of the variance and has cosine 0.99 with the difference between the two condition means, so Figure~\ref{fig:representation-map}(a) shows the second and third components. Along them the no-update states stay in one cluster, and the STAVE states leave it in a different direction for each answer type. Panel (b) refits the PCA on the STAVE states of the six object categories. The no-update states still encode the question type: in the full space, a nearest-centroid classifier recovers it with 99.2\% accuracy under five-fold cross-validation.

Figure~\ref{fig:representation-appendix}(a) decodes the state after each block with the final normalization layer and the output projection of the frozen model, for three questions of Figure~\ref{fig:representation-map}(b). We call an answer found late and lost when the reference first token, with and without a leading space, is the top decoded token after some block in the last 40\% of the decoder but not at the output. Without the update this holds for 10\% to 58\% of the test prompts of the seven LMM checkpoints in panel (e) and for 15\% of the TV benchmark prompts. STAVE outputs that token for 73\% to 95\% of these prompts. For these prompts the answer is present inside the frozen decoder, and the learned vectors keep it through the last blocks.

Panels (b) to (d) project the states of all 18 TV benchmark tasks onto the two leading principal components of the STAVE states, and the no-update states onto the same axes. Several tasks share test prompts exactly, which we detect as prompts with identical no-update states. A shared prompt has one no-update state and one next token, while STAVE moves it with the vectors of each task. List first and list last share all 200 test prompts. With the vectors of each task, STAVE answers both tasks correctly on all of them on every LLM, against at most 2.5\% for the frozen model. Present to gerund and present to past share 9 to 29 prompts per LLM, with 100\% against 0\%. Panel (f) counts every pair of tasks with a shared test prompt, 265 to 277 per LLM. STAVE answers both tasks correctly for 92\% to 94\% of them and the frozen model for at most 2.6\%. On GPT-J 6B and Pythia 12B the cards show next letter and to uppercase in place of the translation pair.

Figure~\ref{fig:depth-mechanism} applies the same decoding to the first answer token on the Idefics2 and LLaVA VQAv2 and COCO checkpoints of Table~\ref{tab:vqav2-coco}, with 128 held-out demo-free prompts per model and dataset and a prompt with four demos and no update as a reference. Without the update, on three of the four model and dataset pairs the correct token is the top decoded token before the last block for 33\% to 48\% of the prompts and is lost at the output. With both vectors it stays top-1 through the output for 64.8\% to 70.3\% of the prompts, against 55.5\% to 64.1\% with four demos. In panel (b), the change that the vectors cause in the state stays between 0.7 and 1.3 times the norm of the state from block 16 to the output, so the frozen blocks do not suppress the update. In panel (c), the attention from the final token to the query image, averaged over layers, rises with the vectors for every model and dataset.

\begin{figure}[p]
\centering
\includegraphics[width=\textwidth,height=0.9\textheight,keepaspectratio]{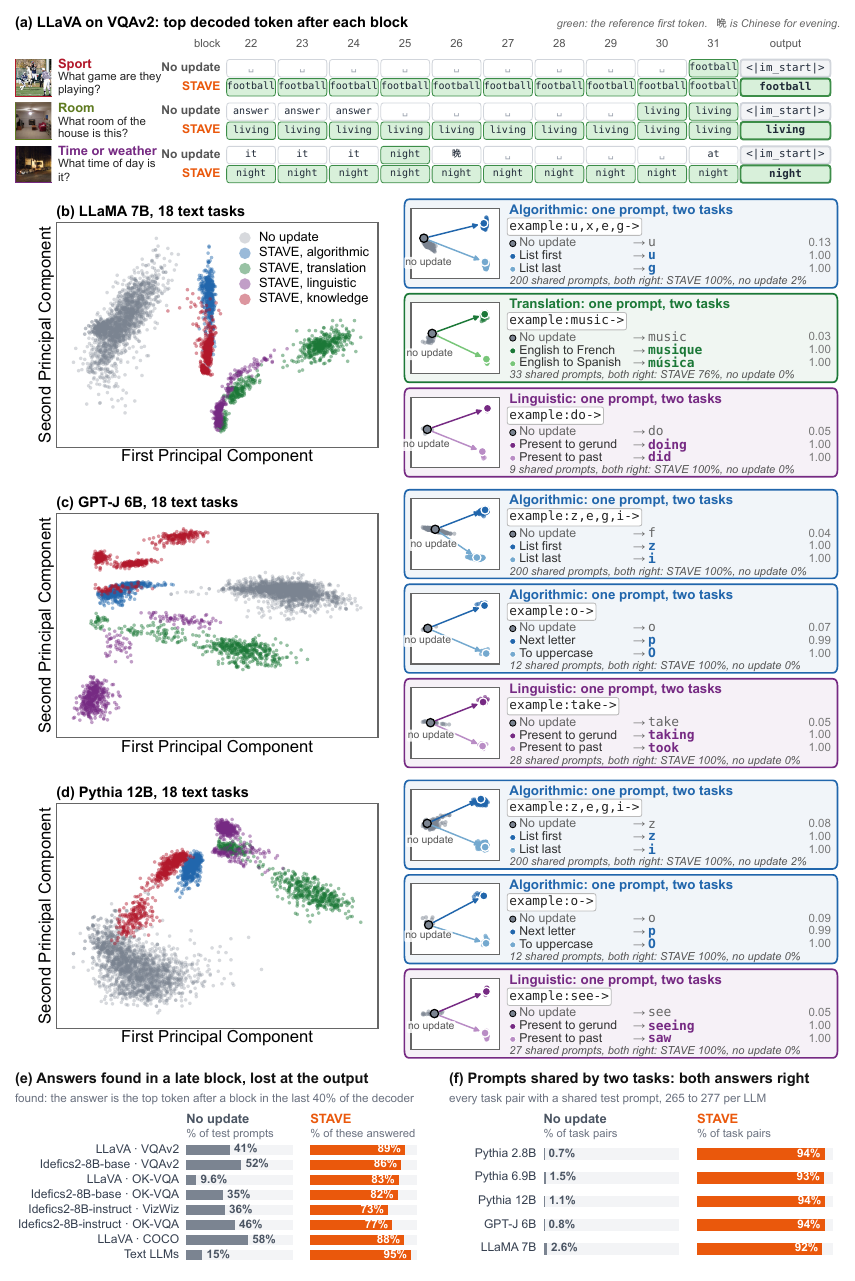}\par
\vspace{-3.2pt}
\caption{Final-token analysis beyond Figure~\ref{fig:representation-map}. (a) Top decoded token after each block. (b) to (d) TV benchmark states. (e) Answers found late and lost. (f) Shared prompts answered correctly for both tasks.}
\label{fig:representation-appendix}
\end{figure}

\begin{figure}[!htbp]
\centering
\includegraphics[width=\linewidth]{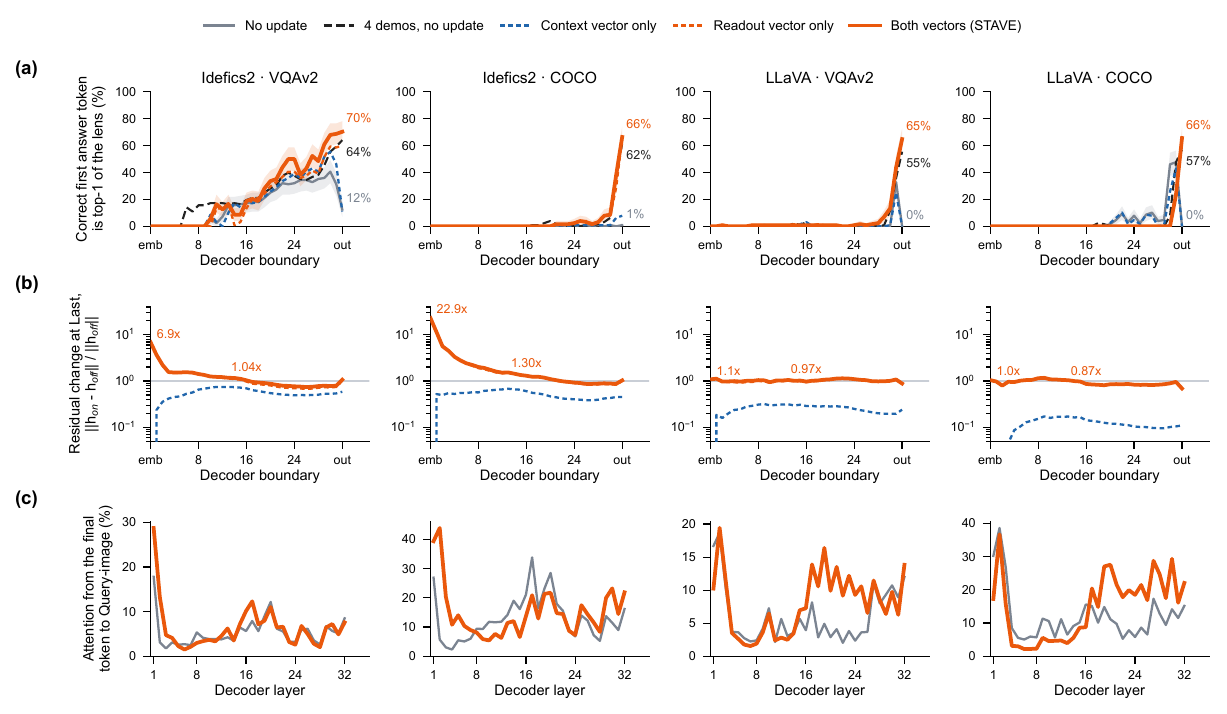}\par
\vspace{-3.2pt}
\caption{Layerwise effect of the learned vectors at the final prompt token. (a) Correct first token decoded as top-1. (b) Relative change of the state. (c) Attention to the query image.}
\label{fig:depth-mechanism}
\end{figure}

\subsection{What each learned vector changes}
\label{app:row-roles}

Figures~\ref{fig:row-roles-main} and~\ref{fig:row-roles} apply the deployed vectors of the Idefics2 entries of Table~\ref{tab:vqav2-coco} to held-out demo-free prompts with no update, with one of the two vectors, and with both vectors. No vector is retrained, so the four conditions differ only in which vectors are added. Figure~\ref{fig:row-roles-main}(a) and Figure~\ref{fig:row-roles}(a) read the first answer token on 1{,}000 prompts per task, for VQAv2 and OK-VQA. A token counts as correct when it matches the answer of one annotator, the reference used for training, so these rates are lower than the VQA scores of Table~\ref{tab:vqav2-coco}, which give full credit to any answer given by at least three annotators. VQAv2 number questions are left out because every number answer begins with the same token on this backbone. The left group of bars asks whether the correct token is the most probable among the first tokens of the task's answers, and the right group whether it is the generated token.

On VQAv2 and OK-VQA, the context vector alone makes the correct answer the most probable task answer for 10 and 14 more percentage points of the prompts, yet it is generated for only 0.3\% and 10.5\% of them, because without the readout vector the most probable next token is usually not an answer. The readout vector alone makes the model answer. Both vectors give the highest rate on both measures, so the improvement that the context vector brings to the choice among answers reaches the generated answer mostly when the readout vector is present. In Figure~\ref{fig:row-roles-main}(b), for each prompt, let $\Delta$ be a change of the state at the final prompt token. The figure reports the share $\|\bar\Delta\|_2^2/\overline{\|\Delta\|_2^2}$ that is common to all prompts, where the bars denote means over prompts.

Figure~\ref{fig:row-roles}(b) to (d) show what the prompt-specific change of the context vector does. On VQAv2 yes/no questions whose reference is no, the readout vector alone still answers yes for 45\% of them, and adding the context vector lowers this rate to 26\% while the rate on questions whose reference is yes stays near 80\%. In panel (c), the readout vector is kept at every layer, while the tokens after the context see the change that the context vector makes to the context states only inside one block of four decoder layers. Layers 12 to 15 alone recover 94\% of the yes/no accuracy gain and layers 16 to 19 alone 65\%, while layers 20 to 31 recover almost none. In layers 12 to 15 the context vector also raises the attention from the final prompt token to the query image the most. Panel (d) replaces the query image by the image of another evaluation sample and keeps the question and the reference answer. On all three tasks the answer loss rises more when the context vector is present, by 0.23, 0.16, and 0.06 nats per token, and every paired 95\% bootstrap interval excludes zero. The context vector therefore makes the answer depend more on the query image, and the readout vector makes this dependence visible in the generated answer.

\begin{figure}[!htbp]
\centering
\includegraphics[width=\linewidth]{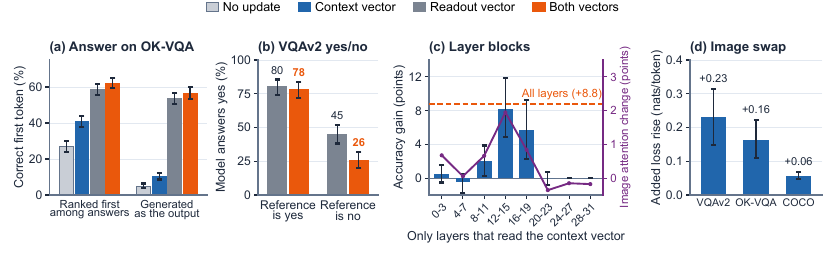}\par
\vspace{-3.2pt}
\caption{More tests of each learned vector on held-out Idefics2 prompts.}
\label{fig:row-roles}
\end{figure}

\subsection{Answer-prefix probabilities}
\label{app:answer-prefix-probabilities}

Figure~\ref{fig:answer-prefix-appendix} compares the LLaVA and Idefics2 entries of Table~\ref{tab:vqav2-coco} and the Idefics2 and Qwen-VL entries of Table~\ref{tab:vqa+finegrained}. In each of its four rows, STAVE has the largest share of examples above the diagonal and the highest mean answer-prefix probability of the four trained methods. Table~\ref{tab:vqav2-coco} rows compare VQAv2 and OK-VQA against PT-Pre, using its 20-token checkpoint. Table~\ref{tab:vqa+finegrained} rows compare VizWiz and OK-VQA against zero-shot. Each row uses the corresponding table's method checkpoints and 512 fixed seed-41 evaluation examples per dataset, paired across methods. The two OK-VQA protocols are kept separate and assigned different colors.

Each point compares the probability that a continuation begins with an annotated answer. We sum disjoint reference-prefix probabilities after removing duplicate token sequences and extensions of shorter reference prefixes. We include tokenizations with and without a leading space, without an end-of-sequence token or length normalization. Each method retains its reported inference prompt. The axes use a logit scale with probabilities displayed as percentages. Marginal histograms show sample counts, and dashed ellipses summarize the empirical 90\% sample spread. Probabilities below 0.0001\% are clipped for display only. Each row is read separately, since its reference method and protocol differ.

\begin{figure*}[!htbp]
\centering
\includegraphics[width=\textwidth,height=0.86\textheight,keepaspectratio]{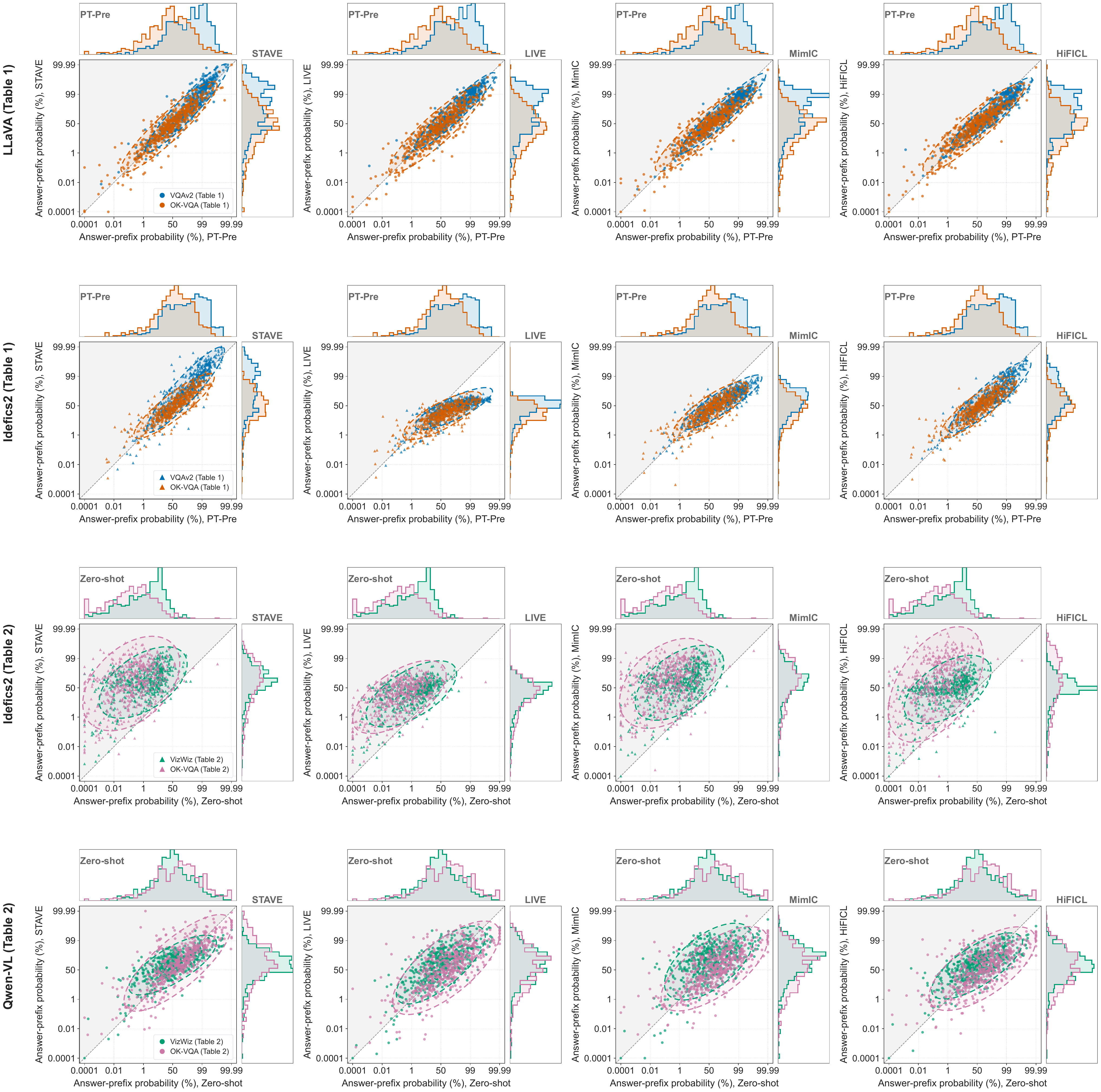}\par
\vspace{-3.2pt}
\caption{Answer-prefix probabilities against PT-Pre (top two rows) and zero-shot (bottom two rows).}
\label{fig:answer-prefix-appendix}
\end{figure*}

\subsection{Qualitative examples}
\label{app:qualitative}

Figures~\ref{fig:qualitative-appendix-1} and~\ref{fig:qualitative-appendix-2} compare STAVE with LIVE, MimIC, and HiFICL on demo-free evaluation examples. The VQA examples come from VQAv2, OK-VQA, and VizWiz with the checkpoints of Tables~\ref{tab:vqav2-coco} and~\ref{tab:vqa+finegrained}, and each is a question that STAVE answers correctly and all three baselines answer incorrectly. The captions come from the Idefics2 and LLaVA COCO checkpoints of Table~\ref{tab:vqav2-coco}. STAVE answers questions that depend on text in the image, as for the bill and the toothpaste, and its captions name the object that the reference captions mention where the baselines name a more common one, for example a girl instead of a doll or a book instead of a laptop.

\clearpage
\begin{figure*}[p]
\centering
\includegraphics[width=\textwidth,height=0.88\textheight,keepaspectratio]{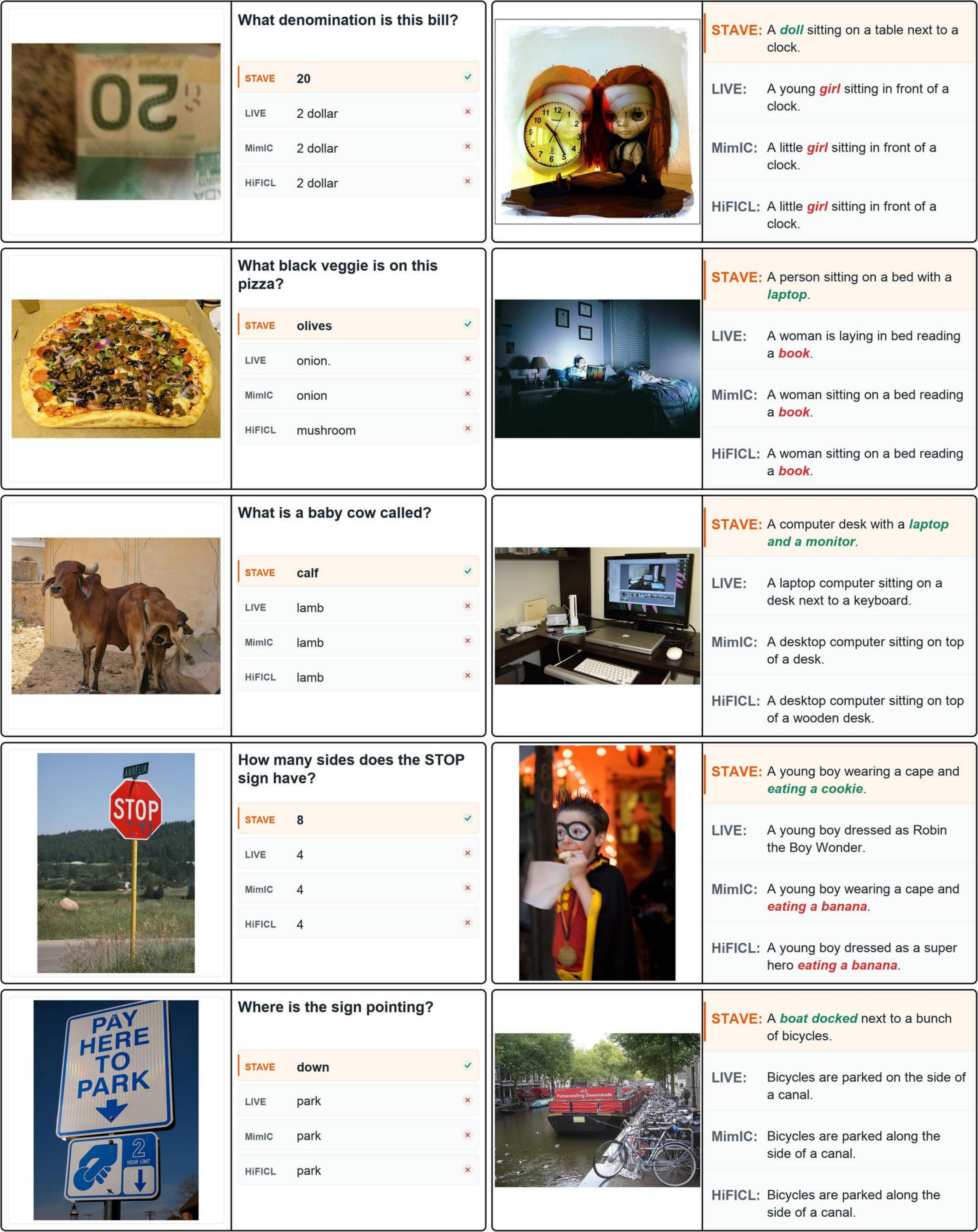}\par
\vspace{-3.2pt}
\caption{More qualitative examples on VQA (left) and COCO (right), set 1.}
\label{fig:qualitative-appendix-1}
\end{figure*}

\clearpage
\begin{figure*}[p]
\centering
\includegraphics[width=\textwidth,height=0.88\textheight,keepaspectratio]{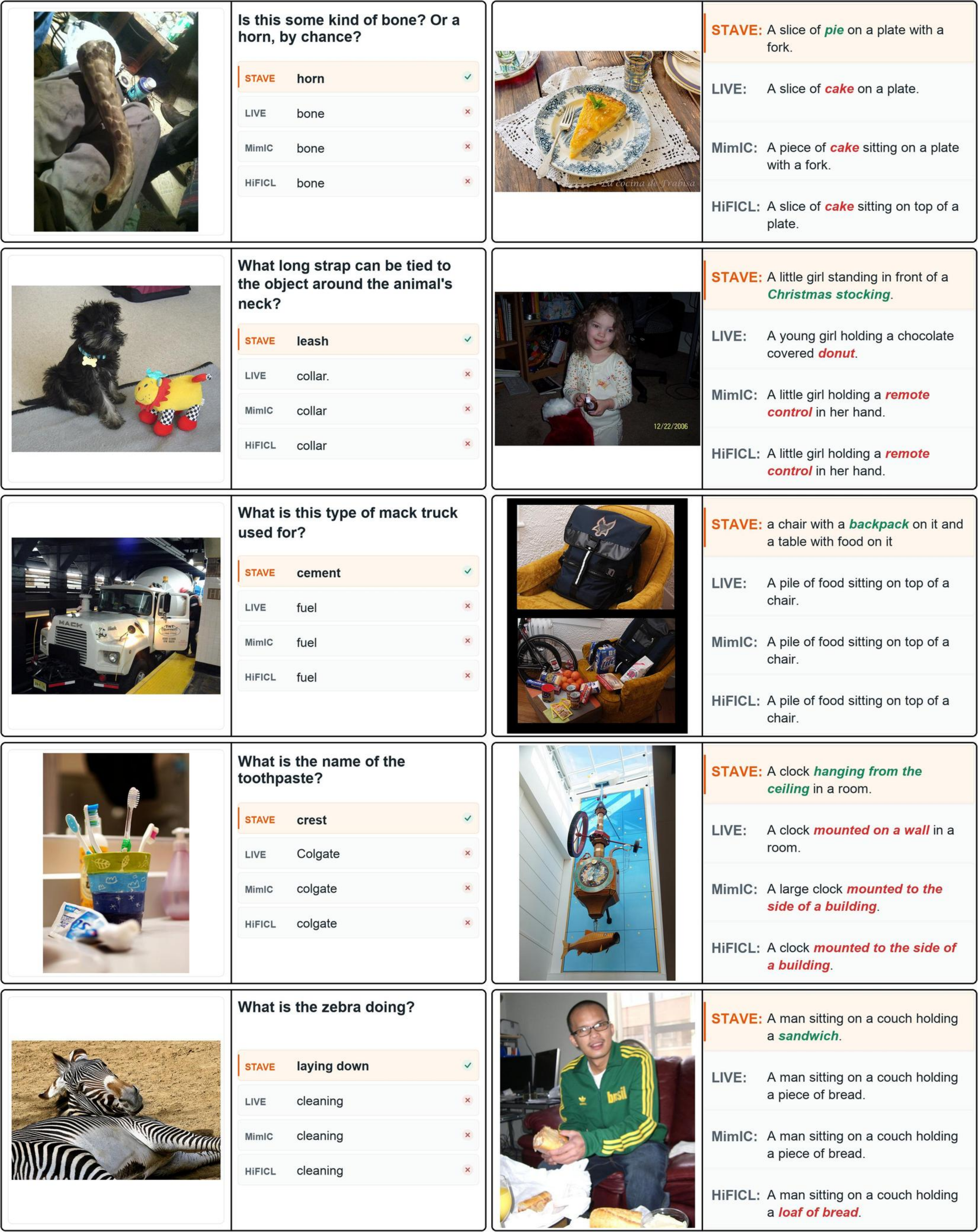}\par
\vspace{-3.2pt}
\caption{More qualitative examples on VQA (left) and COCO (right), set 2.}
\label{fig:qualitative-appendix-2}
\end{figure*}

\clearpage
\section{Limitations and future work}
\label{app:limitations}

STAVE runs at zero-shot inference cost, and its two vectors are trained once per task by backpropagation through the frozen decoder, so training memory and time grow with the backbone. STAVE learns each task from labeled training samples. With a tenth of them it loses a few points, less than MimIC and HiFICL on Idefics2 (Figure~\ref{fig:sample-size}). Training and injection act on the input embeddings, so STAVE requires access to the model weights.

Our LMM experiments cover prompts with one query image and short answers or captions. A natural next step is to apply STAVE to interleaved multi-image and video prompts~\citep{li2024llava,zhang2025learning} and to tasks that require temporal understanding~\citep{ding2025language}. Another application is long-form reasoning, where steering vectors already change the behavior of thinking language models~\citep{venhoff2025understanding}. More broadly, a task stored in two vectors adds no prompt tokens and only $2d$ parameters, so one deployed backbone can serve many downstream tasks, including on devices with limited memory.

\section{Impact statement}
\label{app:impact}

This work aims to make task adaptation of frozen LLMs and LMMs cheaper. STAVE removes demos from the prompt at inference, which lowers the compute, memory, and energy spent on each query, and it stores a task in two vectors, so many tasks can share one deployed backbone. Because the backbone stays frozen, STAVE inherits the biases and failure modes of the pretrained model, including hallucinated content in generated captions, and the same low cost also applies to tasks trained on harmful data. Deployments should therefore keep the safeguards and usage policies of the underlying model.

\end{document}